\documentclass[11pt,reqno,twoside]{article}

\usepackage[hang]{footmisc}
\usepackage{lipsum}

\usepackage{cmap} 

\usepackage[T1]{fontenc}
\usepackage[utf8]{inputenc}
\usepackage{graphicx}
\usepackage{placeins}
\usepackage{enumerate}
\usepackage{algcompatible}
\usepackage{subcaption}
\usepackage{booktabs}
\usepackage{dblfloatfix}  
\usepackage{verbatim}
\newcommand{\comments}[1]{}

\usepackage{soul}

\usepackage{setspace}
\usepackage{dblfloatfix}
\usepackage{subcaption}
\usepackage{graphicx}
\usepackage{caption}   
\usepackage{cuted}    
\usepackage{placeins}

\let\counterwithin\relax  
\usepackage{lmodern} 
\usepackage[scale=0.88]{tgheros} 

\usepackage{bm} 

\usepackage{bbold}

\usepackage{pgfplots}

\usepackage{amsmath,amsbsy,amsgen,amscd,amsthm,amsfonts,amssymb} 
\allowdisplaybreaks[4]

\usepackage[centering,top=0.85in,bottom=0.85in,left=0.9in,right=0.9in]{geometry}

\usepackage{titling}
\usepackage{musicography}
\graphicspath{ {./images/} }

\usepackage[sf,bf,compact]{titlesec}

\usepackage[font=small,margin=20pt,labelfont={sf,bf},labelsep={space}]{caption}

\definecolor{dark-gray}{gray}{0.3}
\definecolor{dkgray}{rgb}{.4,.4,.4}
\definecolor{dkblue}{rgb}{0,0,.5}
\definecolor{medblue}{rgb}{0,0,.75}
\definecolor{rust}{rgb}{0.5,0.1,0.1}

\usepackage{url}
\usepackage[colorlinks=true]{hyperref}
\hypersetup{linkcolor=dkblue}    
\hypersetup{citecolor=rust}      
\hypersetup{urlcolor=rust}     

\usepackage[final]{microtype} 

\newtheoremstyle{myThm} 
    {\topsep}                    
    {\topsep}                    
    {\itshape}                   
    {}                           
    {\sffamily\bfseries}                   
    {.}                          
    {.5em}                       
    {}  

\newtheoremstyle{myRem} 
    {\topsep}                    
    {\topsep}                    
    {}                   
    {}                           
    {\sffamily}                   
    {.}                          
    {.5em}                       
    {}  

\newtheoremstyle{myDef} 
    {\topsep}                    
    {\topsep}                    
    {}                   
    {}                           
    {\sffamily\bfseries}                   
    {.}                          
    {.5em}                       
    {}  

\theoremstyle{myThm}
\newtheorem{theorem}{Theorem}[section]
\newtheorem{lemma}[theorem]{Lemma}

\theoremstyle{myRem}
\newtheorem{remarkx}[theorem]{Remark}
 \newenvironment{remark}
  {\pushQED{\qed}\remarkx}
  {\popQED\endremarkx}

\theoremstyle{myDef}

\usepackage{fancyhdr}
\let\originalleft\left
\let\originalright\right
\renewcommand{\left}{\mathopen{}\mathclose\bgroup\originalleft}
\renewcommand{\right}{\aftergroup\egroup\originalright}

\usepackage{mathtools}
\mathtoolsset{centercolon}  

\definecolor{mygreen}{rgb}{0.1,0.75,0.2}

\newcommand{\nc}{\normalcolor}

\newcommand{\red}{\color{black}}

\providecommand{\mathbbm}{\mathbb} 

\newcommand{\R}{\mathbbm{R}}

\newcommand{\N}{\mathbbm{N}}

\newcommand{\F}{\mathcal{F}}

\renewcommand{\L}{\mathcal{L}}

\usepackage[font = small, margin=30pt]{caption}

\usepackage[]{algorithm}
\usepackage{algpseudocode}
\usepackage{enumerate}

\usepackage{authblk}
\usepackage[square,numbers]{natbib}
\usepackage{chngcntr}
\usepackage{mathrsfs}
\counterwithin{table}{section}
\counterwithin{algorithm}{section}

\newcommand{\E}{\mathbb{E}}

\renewcommand{\P}{\mathbb{P}}

\newcommand{\mcD}{\mathcal{D}}

\usepackage{multirow}
\usepackage{enumitem}

\usepackage[scr = esstix, cal = cm, frak=euler]{mathalfa}

\definecolor{mygreen}{rgb}{0.1,0.75,0.2}
\newcommand{\blue}{\color{black}}

\newcommand{\prob}{\mathbb{P}}

\DeclareMathOperator*{\argmax}{arg\,max}

\newcommand{\noisestd}{\sigma_\varepsilon}
\newcommand{\noisestdip}{\sigma_\eta}

\newcommand{\W}{\mathcal{W}}
\newcommand{\spectralsobolev}{\dot{H}}

\newcommand{\Hexact}{\mathcal{H}_k}
\newcommand{\UCB}{\operatorname{\scriptscriptstyle UCB}}
\newcommand{\TS}{\operatorname{\scriptscriptstyle TS}}
\newcommand{\truth}{f^\dagger}
\newcommand{\vertexset}{\mathcal{V}}
\newcommand{\edgeset}{\mathcal{E}}
\usepackage{cleveref}
\usepackage[scr = esstix, cal = cm, frak=euler]{mathalfa}
\newcommand{\femnodes}{Z_h}
\newcommand{\Ninit}{N_{\mathrm{init}}}

\usepackage[normalem]{ulem}

\usepackage{authblk}
\title{\huge Bayesian Optimization on Networks} 
\author[1,]{\large W. Li}
\author[2]{\large D. Sanz-Alonso}
\author[3]{\large R. Yang}
\affil[1]{\normalsize Committee on Computational and Applied Mathematics, University of Chicago, USA}
\affil[2]{\normalsize Department of Statistics, University of Chicago, USA}
\affil[3]{\normalsize Institute for Mathematical and Statistical Innovation and University of Chicago, USA}
\date{ \vspace{-1.25cm}}

\makeatletter\@addtoreset{section}{part}\makeatother%

\newcommand{\upperRomannumeral}[1]{\uppercase\expandafter{\romannumeral#1}}

\renewcommand{\hat}{\widehat}

\providecommand{\keywords}[1]{\textbf{{Keywords:}} #1}

\begin{document}
\maketitle 


\abstract{ This paper studies optimization on networks modeled as metric graphs. Motivated by applications where the objective function is expensive to evaluate or only available as a black box, we develop Bayesian optimization algorithms
that sequentially update a Gaussian process surrogate model of the objective to guide the acquisition of query points. To ensure that the surrogates are tailored to the network's geometry, we adopt Whittle-Matérn Gaussian process prior models defined via stochastic partial differential equations on metric graphs. In addition to establishing regret bounds for optimizing sufficiently smooth objective functions, we analyze the practical case in which the smoothness of the objective is unknown and the Whittle-Matérn prior is represented using finite elements. Numerical results demonstrate the effectiveness of our algorithms for optimizing benchmark objective functions on a synthetic metric graph and for Bayesian inversion via \emph{maximum a posteriori} estimation on a telecommunication network.}

\bigskip


\section{Introduction}
This paper studies optimization on networks in which nodes are linked by one-dimensional curves. Illustrative applications include finding the most congested site on a road or street network, determining the most likely location for outage in the power grid, and identifying the most active region in a biological neural network, among many others.  We investigate Bayesian optimization algorithms that are particularly effective for global optimization of objective functions that are expensive to evaluate or available only as a black box \cite{shahriari2015taking,frazier2018tutorial,wang2023recent}.  
In Bayesian optimization, a surrogate model of the objective is used to determine where to observe its value. Gaussian process (GP) surrogate models are often employed, but their performance can be sensitive to the choice of kernel. This paper develops and analyzes
Bayesian optimization strategies for objective functions defined on networks,
where standard Euclidean kernels are inadequate and it is essential to use kernels adapted to the network's geometry.

We model networks using compact metric graphs comprising a finite number of vertices and a finite number of edges, where each edge is a curve with finite length \cite{berkolaiko2013introduction}. \blue Throughout this paper, the term ``network'' refers to a metric graph, as formalized in Subsection~\ref{ssec:problem statement}. Metric graphs differ from \emph{discrete graphs} and \emph{embedded networks}. Discrete graphs are purely combinatorial: edges encode only adjacency, functions are defined only on finite vertex/edge sets, and there is no notion of an in-edge location. Embedded networks emphasize the \emph{extrinsic geometry} of a graph embedded in Euclidean space, whereas metric graphs are naturally equipped with a local \emph{intrinsic geometry} via shortest-path distance along the edges.
\nc
To ensure that the surrogate models are tailored to \blue this intrinsic geometry, \nc  we adopt Whittle-Matérn GP prior models specified via stochastic partial differential equations (SPDEs) on metric graphs \cite{bolin2024gaussian}. Whittle-Matérn models offer two important advantages. First, they provide a convenient framework for probabilistic modeling of functions on metric graphs in terms of interpretable parameters controlling the global smoothness, the correlation lengthscale, and the marginal variance  \cite{stein2012interpolation}. Second, they can be represented using finite elements to obtain a sparse approximation of the inverse covariance for efficient sequential update of the posterior surrogate models \cite{lindgren2011explicit}.

We leverage Whittle-Matérn GP priors within two popular Bayesian optimization strategies: \emph{improved GP upper confidence bound} (IGP-UCB)  and \emph{GP Thompson sampling} (GP-TS); see \cite{chowdhury2017kernelized} and also \cite{srinivas2010gaussian,agrawal2012analysis}. We establish convergence rates (simple regret bounds) for both algorithms under natural Sobolev smoothness assumptions on the objective. In addition, we analyze the practical case in which
the smoothness of the objective is unknown and finite element representations of the Whittle-Matérn kernel are employed, which results in epistemic and computational kernel misspecification that affects the regret bounds.  Numerical results illustrate the effectiveness of the proposed algorithms for optimizing benchmark objective functions on a synthetic metric graph and for computing the \emph{maximum a posteriori} estimator in a source-identification Bayesian inverse problem on a telecommunication network.

\subsection{Related Work}
Metric graphs are natural models for networks in which nodes are linked by curves. When equipped with a differential operator, metric graphs are known as \emph{quantum graphs}.
\blue The term ``quantum graph'' originates from mathematical physics \cite{berkolaiko2013introduction,kuchment2008quantum}, where Schr\"odinger-type differential operators act as effective Hamiltonians for wave or quantum-particle propagation on thin network-like structures, such as wire or waveguide networks. \nc
Quantum graphs have been extensively studied in physics \cite{kuchment2008quantum,kuchment2002graph}, and more recently in statistics \cite{bolin2025log}, numerical analysis \cite{bolin2024regularity}, and Bayesian inversion \cite{bolin2025bayesian}. This paper investigates Bayesian optimization on compact metric graphs using kernels defined via fractional elliptic operators. 

Bayesian optimization algorithms are  widely used in many applications, including  hyperparameter tuning for machine learning tasks \cite{snoek2012practical,klein2017fast}, material design \cite{frazier2015bayesian,lei2021bayesian}, drug discovery \cite{colliandre2023bayesian,bellamy2022batched}, parameter estimation for dynamical systems \cite{kim2024optimization,JoshEnhancing}, and experimental particle physics \cite{ilten2017event,cisbani2020ai}. Since Bayesian optimization algorithms replace the objective with a surrogate model and sequentially update this surrogate model as new observations become available, they provide a natural framework for control and sensor placement in digital twins. 
\blue A digital twin workflow is often organized into offline model construction and calibration, online synchronization with streaming data, and online decision making that uses the calibrated and synchronized simulator. Our contribution targets this online decision-making component. We develop a principled strategy for selecting query locations on a network when each query corresponds to a costly and possibly noisy evaluation of the digital twin model and only a limited evaluation budget is available. This is complementary to calibration and synchronization modules, which use data to update the digital twin model over time. \nc 
Recent works that explore the use of Bayesian optimization for digital twins include \cite{chakrabarty2021attentive,nobar2024guided,lin2025digital}. Digital twin systems that are naturally formulated in our metric graph setting include the simulation and control of signal propagation on biological neural networks, and the modeling and optimization of dynamic flow across electrical transmission networks. Our work provides a principled optimization approach for these and related problems.

\red Since compact metric graphs are locally one-dimensional, our formulation is intrinsic to the graph and does not depend on how the network is embedded in any ambient Euclidean space. Scalability in this setting is therefore driven  by the size of the graph and the discretization resolution. At each Bayesian optimization iteration, we perform a single forward simulation corresponding to the queried location and then update the GP surrogate posterior. With finite-element Whittle--Mat\'ern models on metric graphs, the dominant computational cost reduces to solving sparse linear systems. This cost grows with the mesh size and can be handled efficiently using standard sparse direct solvers or iterative methods. \nc

\blue Compared with gradient-based calibration or adjoint-based optimal design, Bayesian optimization is particularly suited to settings where objective evaluations are expensive, noisy, and effectively black-box, so that reliable derivatives or adjoints are unavailable or too costly to compute. In such cases, Bayesian optimization uses each evaluation efficiently by maintaining a probabilistic surrogate and selecting new queries through uncertainty-aware acquisition rules. Bayesian optimization also differs from reduced-order-model optimization, which typically relies on a surrogate trained offline and then optimized. In contrast, Bayesian optimization updates its surrogate online and balances exploration and exploitation under a limited evaluation budget. These advantages are especially relevant in digital twins when decisions correspond to locations on a network, such as sensing or actuation sites, and each evaluation requires running a graph-based simulation model. On the other hand, when accurate gradients or adjoints are readily available and objective evaluations are relatively cheap, gradient-based or adjoint-based methods can be more computationally efficient. Similarly, when a reliable reduced-order model can be constructed offline and used for fast repeated queries, reduced-order-model optimization may be preferable. 

More broadly, our work extends Bayesian optimization beyond its standard Euclidean setting to network-structured domains. 
Classical Bayesian optimization is most commonly developed on Euclidean spaces with stationary kernels induced by Euclidean distance. By considering geometry-aware GP surrogates on compact metric graphs, we encode continuity along edges and coupling at junctions, which yields more faithful models for network-location design with computational tractability and theoretical support. \nc 
Bayesian optimization in non-Euclidean settings has been studied for instance in \cite{baptista2018bayesian}, which considers optimization over discrete combinatorial structures, and in \cite{kim2024optimization}, which considers optimization of an objective function on a manifold that can only be accessed through point cloud data. The authors in \cite{kim2024optimization} model the point cloud as a combinatorial graph, and define surrogate models for functions on the \emph{vertices} of this graph using  graphical Matérn GPs \cite{sanz2022spde}. 
To our knowledge, this is the first paper to investigate Bayesian optimization on networks modeled by compact metric graphs. In contrast to \cite{kim2024optimization}, we leverage recently developed Whittle-Matérn GPs defined along the \emph{vertices and edges} of a metric graph \cite{bolin2024gaussian}. 

To understand the effect of introducing finite element representations of Whittle-Matérn kernels within IGP-UCB and GP-TS, we analyze Bayesian optimization under kernel misspecification. For GP-UCB, regret bounds under kernel misspecification were established in \cite{bogunovic2021misspecified}. Here, we generalize the theory to also cover GP-TS, and quantify the size of the misspecification by building on recent work on numerical approximation of fractional elliptic differential equations on metric graphs \cite{bolin2024regularity}. Related works that investigate  GP regression under epistemic and computational kernel misspecification  include \cite{sanz2022finite} and \cite{sanz2025gaussian}.


\subsection{Outline and Main Contributions}
\begin{itemize}
    \item Section \ref{sec:BayesOptMetricGraph} introduces the problem statement and the necessary background on metric graphs, Bayesian optimization, and Whittle-Matérn processes on metric graphs. Theorem \ref{thm:exact kernel setting} establishes regret bounds for IGP-UCB and GP-TS in the idealized case in which the kernel is chosen to match the smoothness of the objective and the exact Whittle-Matérn kernel is used, without accounting for discretization error. 
    \item Section \ref{sec:FEM} considers the practical implementation of IGP-UCB and GP-TS using finite element representations of Whittle-Matérn processes. Theorem \ref{thm:regret fem} establishes regret bounds under epistemic and computational kernel misspecification, where the smoothness of the objective is unknown and finite element representations are employed.  
    \item Section \ref{sec:Numerics} illustrates the performance of IGP-UCB and GP-TS for benchmark functions on a synthetic metric graph and for Bayesian inversion on a telecommunication network. The results clearly demonstrate the advantage of using  Whittle-Matérn kernels intrinsically defined on the metric graph over standard kernels defined using Euclidean distance.
    \item Section \ref{sec:conclusions} closes with conclusions. 
    \item  Appendix \ref{sec:misspecified BO} presents a new general theory for misspecified TS that may be of independent interest and Appendix \ref{sec:tech proof} contains proofs of all technical lemmas. Appendix \ref{appendix:online} provides supplementary materials for the implementation of the algorithms. 
\end{itemize}

\subsection{Notation}
For real numbers $a, b$, we denote $a \wedge b = \operatorname{min}(a, b)$ and $a \vee b = \operatorname{max}(a, b)$. The symbol $\lesssim$ will denote less than or equal to up to a universal constant and similarly for $\gtrsim$. For real sequences
$\{a_n\},\{b_n\}$, we write $a_n \asymp b_n$ if $a_n \lesssim  b_n$ and $b_n \lesssim  a_n$ for all $n$.

\section{Bayesian Optimization on Metric Graphs}\label{sec:BayesOptMetricGraph}

\subsection{Problem Statement}
\label{ssec:problem statement}
Let $\Gamma$ be a graph with vertices $\vertexset = \{v_i\}$ and edges $\edgeset = \{e_j\}$. We are concerned with graphs in which edges represent physical one-dimensional curves connecting vertices. 
To model this setting, we assign to each edge $e\in \edgeset$ a positive length $L_e>0$, and then we orient each edge arbitrarily and identify it with the interval $[0,L_e]$ via a coordinate $z_e$. A graph $\Gamma$ supplemented with this structure is called a \emph{metric graph}, where the metric is naturally given by the shortest path distance \cite{berkolaiko2013introduction}, denoted as $d$ hereafter.
A generic point $x$ on a metric graph $\Gamma$ can be represented as $x = (e,z_e)$ for some $e \in \edgeset$ and $z_e \in [0, L_e].$  We will focus on \emph{compact metric graphs} comprising finitely many vertices and edges, with every edge of finite length. For illustration, Figure~\ref{fig:metric graph example}(a) depicts a compact metric graph that represents a telecommunication network in New York, capturing traffic behavior representative of operational networks. This graph was taken from the open data-set in \cite{SNDlib10,OrlowskiPioroTomaszewskiWessaely2010}.

Our goal is to find the global maximizer of a black-box objective function $\truth: \Gamma \to \R$ defined on a compact metric graph $\Gamma.$ We assume that $\truth$ can only be noisily observed by measurements of the form 
\begin{align}\label{eq:obsevation}
    y = \truth(x) + \varepsilon ,
\end{align}
where $x\in \Gamma$ is a \emph{query point} and $\varepsilon$ is a centered $R$-sub-Gaussian noise, i.e., for all $\xi \in \R$ it holds that
$\E e^{\xi\varepsilon } \leq \exp\left(\frac{\xi^2R^2}{2}\right).$

\begin{minipage}{0.95\linewidth}
    \begin{figure}[H]
\centering
\begin{subfigure}{0.49\textwidth}
  \centering
  \includegraphics[width=\linewidth]{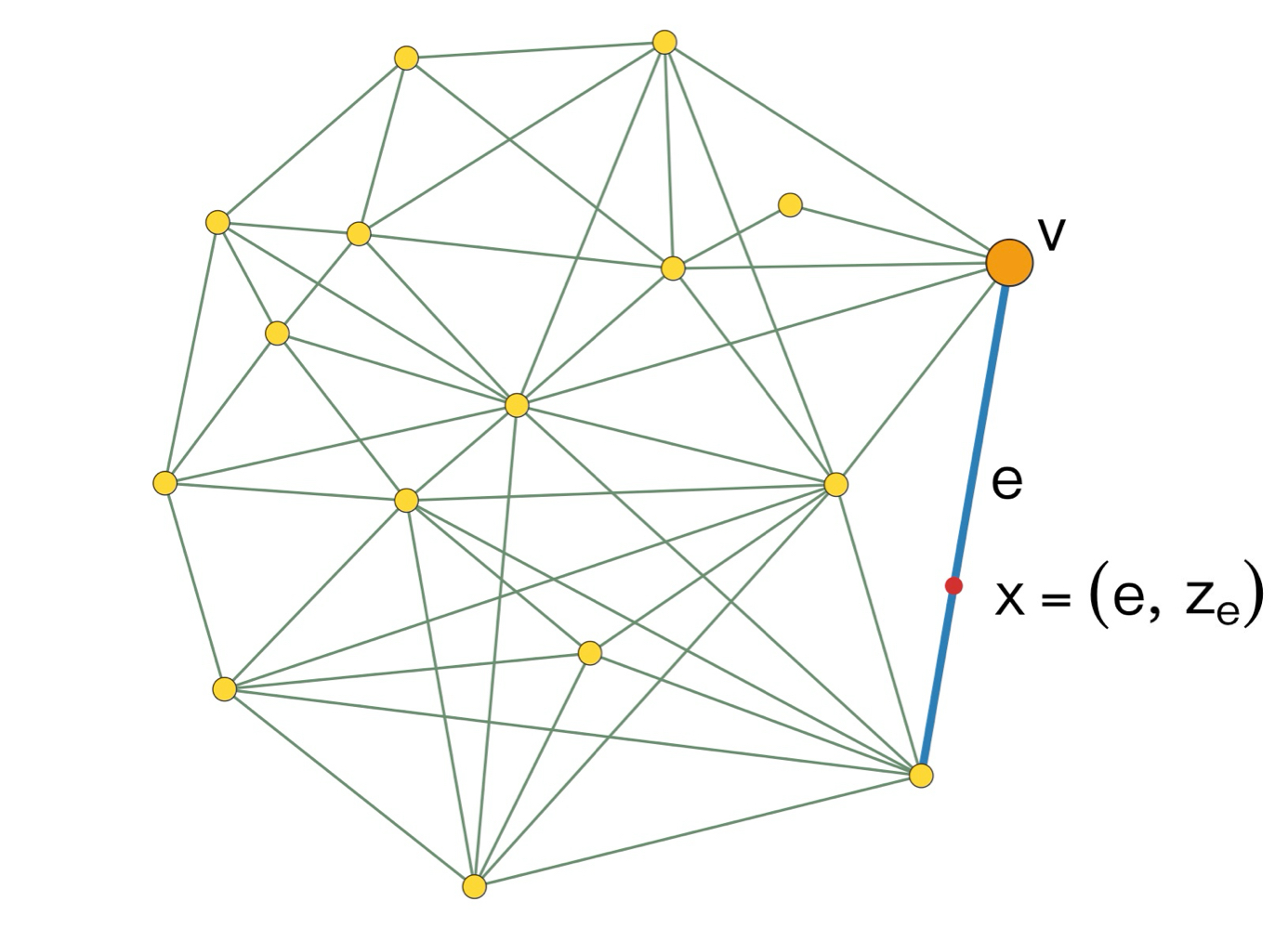}
  \caption{Telecommunication Network in New York}
\end{subfigure}\hfill
\begin{subfigure}{0.50\textwidth}
  \centering
  \includegraphics[width=\linewidth]{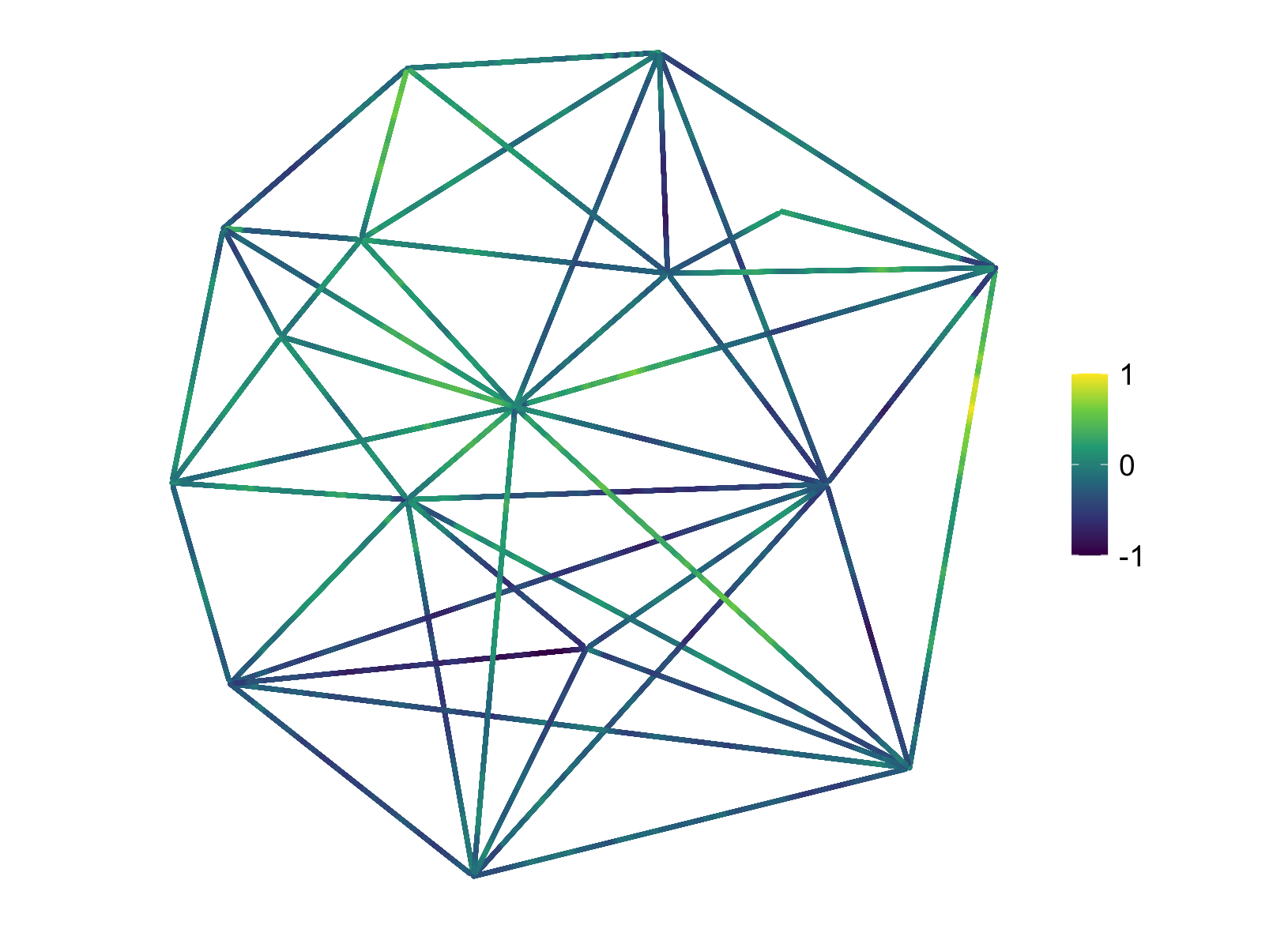}
  \caption{Sample from Whittle-Matérn GP}
\end{subfigure}

\caption{Telecommunication network modeled as a compact metric graph (left) and a sample from the Whittle-Matérn GP defined in Subsection \ref{sec:exact kernel} (right). }
\label{fig:metric graph example}
\end{figure}
\end{minipage}


\subsection{Gaussian Processes and Bayesian Optimization}
In this work, we investigate two Bayesian optimization algorithms:
IGP-UCB and GP-TS. Both methods share a common structure: 
\begin{itemize}
    \item {\bf Initial design:} 
    \begin{enumerate}
        \item Choose initial query points $\{x_i^{(0)} \}_{i=1}^{N_{\mathrm{init}}}.$ 
        \item Observe $y_i^{(0)} = \truth(x_i^{(0)}) + \varepsilon_i^{(0)}$  as in \eqref{eq:obsevation},  $ \, 1 \le i \le N_{\mathrm{init}}.$ 
        \item Set $\mcD_0 := \{(x_i^{(0)}, y_i^{(0)})\}_{i=1}^{N_{\mathrm{init}}}.$
    \end{enumerate}
    \item {\bf Sequential design:} For $t=1, \ldots, T$ do: 
    \begin{enumerate}
        \item Choose $x_t$ by maximizing an acquisition function $\mathrm{acq}_t(x)$ defined using $\mcD_{t-1}.$ 
        \item Observe $y_t = \truth(x_t) + \varepsilon_t,$  as in \eqref{eq:obsevation}. 
        \item Set $\mcD_t: = \mcD_{t-1} \cup \{(x_t,y_t) \}.$
    \end{enumerate}
\end{itemize}

All noise variables $\{\varepsilon_i^{(0)} \}_{i=1}^{N_\mathrm{init}}$ and $\{\varepsilon_t\}_{t=1}^T$ in the observations are assumed throughout to be independent copies of the noise variable $\varepsilon$ in \eqref{eq:obsevation}. 
IGP-UCB and GP-TS use different acquisition functions, both of which involve a GP surrogate model for the objective function. Let $k: \Gamma \times \Gamma \to \R$ be a symmetric, positive definite kernel, and define the following functions
\begin{align}
    \mu_{t-1}(x)& := k_{t-1}(x)^\top (K_{t-1}+\lambda I)^{-1}Y_{t-1}, \label{eq:postmean}\\ 
    k_{t-1}(x, x')& := k(x,x')- k_{t-1}(x)^\top(K_{t-1}+\lambda I)^{-1}k_{t-1}(x'),\label{eq:postcov} \\
    \sigma^2_{t-1}(x)& := k_{t-1}(x, x),\label{eq:poststd} 
\end{align}
where $\lambda>0$ is a regularization parameter to be chosen and 
\begin{align*}
    k_{t-1}(x) &:= \Bigl[k(x,x_1^{(0)}),\ldots,k(x,x_{N_\textrm{init}}^{(0)}), k(x,x_1), \ldots, k(x,x_{t-1}) \Bigr]^\top \, , \\
   K_{t-1} &:= \begin{bmatrix}
       k(x_1^{(0)},x_1^{(0)}) & \cdots & k(x_1^{(0)},x_{N_\textrm{init}}^{(0)}) & k(x_1^{(0)},x_1) & \cdots & k(x_1^{(0)}, x_{t-1}) \\
       \vdots & \ddots & \vdots & \vdots & \ddots & \vdots \\ 
       k(x_{N_\textrm{init}}^{(0)},x_1^{(0)}) & \cdots &k(x_{N_\textrm{init}}^{(0)},x_{N_\textrm{init}}^{(0)}) &k(x_{N_\textrm{init}}^{(0)},x_1) & \cdots &k(x_{N_\textrm{init}}^{(0)},x_{t-1}) \\
       k(x_1,x_1^{(0)}) & \cdots & k(x_1,x_{N_\textrm{init}}^{(0)}) &k(x_1,x_1) & \cdots &k(x_1,x_{t-1})\\
       \vdots & \ddots & \vdots & \vdots & \ddots & \vdots\\
       k(x_{t-1},x_1^{(0)}) & \cdots & k(x_{t-1},x_{N_\textrm{init}}^{(0)}) &k(x_{t-1},x_1) & \cdots &k(x_{t-1},x_{t-1})
   \end{bmatrix}\, ,  \\
   Y_{t-1} &:= \Bigl[ y_1^{(0)}, \cdots, y_{N_{\textrm{init}}}^{(0)}, y_1, \cdots, y_{t-1}   \Bigr]^\top \, .
\end{align*}
When the noise $\varepsilon$ in \eqref{eq:obsevation} is Gaussian with variance $\lambda,$ $\varepsilon \sim \mathcal{N}(0,\lambda)$, the functions $\mu_{t-1},$ $k_{t-1},$ and $\sigma_{t-1}$ defined in \eqref{eq:postmean}, \eqref{eq:postcov}, and \eqref{eq:poststd} have a natural Bayesian interpretation: Placing a GP prior on the objective, 
$\truth \sim \mathcal{GP}(0,k),$ they represent the posterior mean function, posterior covariance function, and posterior standard deviation function given data $\mathcal{D}_{t-1}$. 
 If $\Ninit=0$, we follow the convention that $\mu_0=0$ and $k_0(\cdot,\cdot)=k(\cdot,\cdot)$.  

The $t$-th query point $x_t$ of IGP-UCB and GP-TS is chosen by maximizing the acquisition functions
  \[
  \mathrm{acq}_t(x) =
  \begin{cases}
    \mu_{t-1}(x) + \beta_t\, \sigma_{t-1}(x), & \text{(IGP-UCB)}, \\[6pt]
    f_t(x), \quad f_t \sim \mathcal{GP}\!\bigl(\mu_{t-1},\, v_t^2\, k_{t-1}\bigr), & \text{(GP-TS)}.
  \end{cases}
  \] 
Both acquisitions utilize the posterior distribution of $\truth$ given $\mathcal{D}_{t-1}$ to balance \emph{exploitation} and \emph{exploration}. Exploitation is achieved by favoring query points where the posterior mean $\mu_{t-1}$ is large, whereas exploration is achieved by favoring points where the uncertainty in the surrogate, as captured by $\sigma_{t-1}$ and $k_{t-1},$ is large.  The parameters $\beta_t$ and $v_t$ in IGP-UCB and GP-TS serve to balance the exploitation/exploration trade-off. Following \cite{chowdhury2017kernelized}, in our analysis we specify these parameters via input hyperparameters $B,R,$ $\lambda,$ and $\delta,$ setting
\begin{equation}\label{eq:parametersbetav}
\begin{aligned}
     \beta_t \;&=\; B + R\sqrt{2\bigl(\gamma_{t-1}(k)+(\Ninit+t)(\lambda-1)/2+\log(1/\delta)\bigr)} \, ,\\
  v_t \;&=\; B + R\sqrt{2\bigl(\gamma_{t-1}(k)+(\Ninit+t)(\lambda-1)/2+\log(2/\delta)\bigr)} \, ,
\end{aligned}
\end{equation}
where $\lambda$ is the same  regularization parameter as in \eqref{eq:postmean}, $B$ represents a bound on the reproducing kernel Hilbert space (RKHS) norm of the objective $\truth$, $R$ is the sub-Gaussian constant of the observation noise, and $\delta \in (0,1)$ determines the probabilistic confidence level in our regret bounds. 
For $t \ge 1,$ the quantity $\gamma_t(k)$ in \eqref{eq:parametersbetav} is defined as 
\begin{align}\label{eq:maximum info gain}
    \gamma_t(k) = \underset{A\subset \Gamma: |A|=N_{\mathrm{init}}+t}{\operatorname{max}}\,\, \frac12 \log |I+\lambda^{-1}K_A|,
\end{align}
where $K_A=[k(x,x')]_{x,x'\in A}$. We refer to $\gamma_t(k)$ as the \emph{maximum information gain}, noting that the quantity $\frac12 \log |I+\lambda^{-1}K_A|$ is the mutual information $I(y_A;\truth_A)$ when $\truth\sim \mathcal{GP}(0,k)$ and $\varepsilon\sim \mathcal{N}(0,\lambda)$, which quantifies the reduction of uncertainty about $\truth$ after making the observations.  
Algorithm \ref{alg:GP-bandits} outlines the IGP-UCB and GP-TS procedures.

\begin{minipage}{0.95\linewidth}
\begin{algorithm}[H]
\caption{IGP-UCB and GP-TS}
\label{alg:GP-bandits}
\begin{algorithmic}[1]
\Require Metric graph $\Gamma$; kernel $k$; parameters $B,R,\lambda,\delta$; horizon $T$; initial design size $N_{\mathrm{init}}$.
\State Choose $X_{\mathrm{init}}=\{x_i^{(0)}\}_{i=1}^{N_{\mathrm{init}}}\subset\Gamma$.
\State Observe $y_i^{(0)} = \truth(x_i^{(0)}) + \varepsilon_i^{(0)}$, 
 with $\varepsilon_i^{(0)}$ being i.i.d.\ $R$-sub-Gaussian random variables for $i=1,\dots,N_{\mathrm{init}}$ (as in eq.~\eqref{eq:obsevation}). 

\State Initialize $\mcD_0 \gets \{(x_i^{(0)}, y_i^{(0)})\}_{i=1}^{N_{\mathrm{init}}}$.
\For{$t=1,2,\dots,T$}
  \State Form a decision set $\Gamma_t \subseteq \Gamma$ (either $\Gamma$ or an adaptive subset).
  \State Compute $\mu_{t-1}(\cdot)$, $\sigma_{t-1}(\cdot),$ $k_{t-1}(\cdot,\cdot)$ given $\mcD_{t-1}$ using \cref{eq:postmean,eq:postcov,eq:poststd}.
  \State Define $\beta_t$ and $v_t$ using \eqref{eq:parametersbetav} and set
  \[
  \mathrm{acq}_t(x) =
  \begin{cases}
    \mu_{t-1}(x) + \beta_t\, \sigma_{t-1}(x), & \text{(IGP-UCB)}, \\[6pt]
    f_t(x), \quad f_t \sim \mathcal{GP}\!\bigl(\mu_{t-1},\, v_t^2\, k_{t-1}\bigr), & \text{(GP-TS)}.
  \end{cases}
  \]
  \State Select $x_t \in \arg\max_{x\in \Gamma_t} \mathrm{acq}_t(x)$.
  \State Observe $y_t = \truth(x_t)+\varepsilon_t$, 
   with $\{\varepsilon_t\}$ as defined in \eqref{eq:obsevation}. 
  \State Update $\mcD_t \gets \mcD_{t-1} \cup \{(x_t, y_t)\}$.
\EndFor
\end{algorithmic}
\end{algorithm}
\end{minipage}

\begin{remark}
\label{rmk:maximin design}
A practical choice for the initialization policy is the maximin (farthest--first) design, which yields a space-filling set of $N_{\mathrm{init}}$ points on $\Gamma$. Specifically, we start with a randomly sampled point on the metric graph and then use a deterministic maximin selection rule until reaching the desired number $N_{\rm init}$ of initial design points. Since computing the maximum information gain $\gamma_{t-1}$ is expensive, in our numerical experiments we approximate it by the mutual information on the algorithm’s realized history set $\mathcal{D}_{t-1}$, i.e. $\hat{\gamma}_{t-1}\coloneqq \frac12 \log |I+\lambda^{-1}K_{X_{t-1}}|$, where $X_{t-1}={\{x_i^{(0)} \}_{i=1}^{N_{\mathrm{init}}}\cup \{x_s\}_{s=1}^{t-1}}$. 
Furthermore, the original proof in \cite{chowdhury2017kernelized} that leads to the choices \eqref{eq:parametersbetav} in fact only requires $\hat{\gamma}_{t-1}$ in the expression, where $\gamma_{t-1}$ serves as a convenient upper bound that unifies the theory. 
\end{remark}

\begin{remark}
As is standard in the Bayesian optimization literature, our analysis will focus on the setting in which the acquisition functions are exactly optimized. Bayesian optimization  with inexactly optimized acquisition functions has been recently studied in \cite{kim2025bayesian}.
\end{remark}

\subsection{Choice of Kernel: Whittle-Matérn Gaussian Processes}\label{sec:exact kernel}

The prior that we shall employ in modeling the objective function is the Whittle–Matérn GP on the compact metric graph \(\Gamma\) introduced by \cite{bolin2024gaussian}. As in the Euclidean case, Mat\'ern type GPs can be defined via a fractional SPDE. The main difference on compact metric graphs is the need to impose vertex conditions to obtain a graph-native kernel consistent with the graph geometry (see \cite{bolin2024regularity,bolin2025bayesian} for details).
Here we overview the main ideas of the construction while keeping technical details minimal.

To set up the SPDE on a compact metric graph, we begin with a brief introduction to differential operators on metric graphs $\Gamma$. 
Let $\widetilde{H}^2(\Gamma)=\bigoplus_{e\in \mathcal{E}} H^2(e)$, where $H^2(e)$ is the standard Sobolev space defined over $e$ by identifying it with the interval $[0,L_e]$.
We introduce the second-order elliptic operator $\mathcal{L}$ whose action on functions $u \in \widetilde{H}^2(\Gamma)$ is defined on each edge by
\begin{equation}\label{eq:edge-operator-const}
(\mathcal{L}u)_e(z)\;=\;-\frac{d^2}{dz^2}\;u_e(z)+\;\kappa^{\,2}\!\,u_e(z),
\qquad z\in(0,L_e),
\end{equation}
where \(\kappa>0\) is a fixed constant. To couple the edgewise operators into a global operator \(\mathcal{L}\) acting on the entire compact metric graph \(\Gamma\), we impose Kirchhoff vertex conditions:
\begin{equation}\label{eq:kirchhoff}
\text{$u$ is continuous on \(\Gamma\),}
\qquad
\forall v\in \vertexset:\ \sum_{e\in E_v} \,\partial_e u(v)\;=\;\theta\,u(v),
\end{equation}
where \(E_v\) is the set of edges incident to vertex \(v\), \(\partial_e u(v)\)  denotes the outward-directed derivative of \(u\) at vertex \(v\) along edge \(e\), and \(\theta \in \mathbb{R}\) is a given parameter. Throughout the rest of the paper we assume standard Kirchhoff conditions with \(\theta = 0\), which ensures that flux is conserved at every vertex. 

Following \cite{bolin2024regularity,bolin2025bayesian}, the operator \(\mathcal{L}\) with standard Kirchhoff vertex conditions is positive definite and has a discrete spectrum. Let \(\{(\lambda_i,\psi_i)\}_{i=1}^\infty\) denote its eigenpairs, with the eigenvalues in nondecreasing order. The fractional operator \(\mathcal{L}^\alpha\) is defined spectrally by
\[
\mathcal{L}^\alpha  u \;=\; \sum_{i=1}^\infty \lambda_i^\alpha \,\langle u,\psi_i\rangle_{L^2(\Gamma)} \,\psi_i .
\]
Here $L^2(\Gamma)
:= \bigoplus_{e\in E} L^2(e)$ and its inner product are defined edgewise, similarly to $\widetilde{H}^2(\Gamma)$ above. 
The Whittle-Matérn GP \(u\) on $\Gamma$ is then specified as the solution of 
\begin{equation}
\label{equa:prior gaussian}
\mathcal{L}^{\alpha} (\tau u) \;=\; \mathcal{W},
\end{equation}
where \(\alpha>0\) controls the regularity, \(\tau>0\) sets the scale, the parameter $\kappa$ in $\mathcal{L}$ determines the correlation lengthscale, and \(\mathcal{W}\) denotes Gaussian white noise on \(\Gamma\). For brevity, we keep the same symbol \(u\) in the SPDE setting, where \(u:\Gamma\times\Omega\to\mathbb{R}\) is a random field, whereas above \(u\) was a deterministic function used to define \(\mathcal{L}\) in \eqref{eq:edge-operator-const}. Concretely, if \(\{e_i\}_{i\ge 1}\) is any orthonormal basis of \(L^2(\Gamma)\), then
\(
\mathcal{W}=\sum_{i\ge 1}\xi_i e_i
\)
with \(\xi_i\stackrel{\mathrm{i.i.d.}}{\sim}\mathcal{N}(0,1)\) on \(\Omega\).
Figure~\ref{fig:metric graph example}(b) shows a realization of the GP defined in \eqref{equa:prior gaussian} with $\alpha=\kappa=\tau=1$ and normalized to take values in $[-1,1].$

Proposition 3.2 of \cite{bolin2024regularity} guarantees that, for \(\alpha>\tfrac{1}{4}\),  \eqref{equa:prior gaussian} has a unique solution \(u\in L^2(\Omega;L^2(\Gamma))\) admitting the series representation
\begin{align}\label{eq:expansion exact}
    u={\tau}^{-1}\sum_{i=1}^\infty \lambda_i^{-\alpha}\,\xi_i\,\psi_i,\qquad \xi_i \stackrel{\mathrm{i.i.d.}}{\sim} \mathcal{N}(0,1),
\end{align}
which induces the kernel 
\begin{align}\label{eq:kernel exact}
k(x,x') =  {\tau}^{-2}  \sum_{i=1}^\infty \lambda_i^{-2\alpha}\,\psi_i(x)\,\psi_i(x').
\end{align}
Let \(\mathcal{H}_k\) denote the RKHS associated with \(k\):
\begin{align}\label{eq:RKHS exact}
    \mathcal{H}_k = \left\{\,g=\sum_{i=1}^\infty a_i \psi_i \;:\; \|g\|^2_{\mathcal{H}_k} :=  \tau^2 \sum_{i=1}^\infty a_i^2 \lambda_i^{2\alpha} < \infty \,\right\}=:\dot{H}^{2\alpha}(\Gamma).
\end{align}
The space $\dot{H}^{2\alpha}(\Gamma)$ (and hence $\Hexact$)  is a spectrally defined Sobolev space over $\Gamma,$ closely related to classical Sobolev spaces defined using weak derivatives and interpolation \cite[Theorem 4.1]{bolin2024regularity}.

The following lemma, proved in Appendix \ref{sec:tech proof}, states some properties of the kernel \eqref{eq:kernel exact} and its eigenfunctions that will be used to derive regret bounds for IGP-UCB and GP-TS in Subsection \ref{sec:regret bounds} below. 
\begin{lemma}\label{lemma:uniform boundedness of fem eigenfunction}
Suppose the $\{\psi_i\}_{i=1}^{\infty}$ are $L^2(\Gamma)$ normalized, then $\operatorname{sup}_i \|\psi_i\|_{L^\infty(\Gamma)} \leq \Psi$ for a constant $\Psi$ independent of $i$. 
As a consequence, for $\alpha>\frac14$, $|k(x,x')|\leq \overline{k}$ for some constant $\overline{k}<\infty$. 
Moreover, for $\alpha>\frac12$ we have 
\begin{align*}
    |k(x'',x)-k(x'',x')|\leq S  \tau^{-2} \Psi^2 d(x,x'),\qquad \forall x,x',x''\in \Gamma,
\end{align*}
where $S= \sum_{i=1}^\infty \lambda_i^{-2\alpha+1/2}<\infty$ and we recall $d$ is the shortest path distance on $\Gamma$. 
\end{lemma}

\subsection{Regret Bounds}\label{sec:regret bounds}
Here we establish regret bounds for Algorithm \ref{alg:GP-bandits}.
Following the standard practice in kernelized Bayesian optimization (see \cite{chowdhury2017kernelized,vakili2021information}), we assume the unknown objective \(\truth\) lies in \(\mathcal H_k\) with bounded RKHS norm, i.e., \(\|\truth\|_{\mathcal H_k}\le B\).
We will analyze the simple regret 
\begin{equation}
\label{equa:simple regret}
    \begin{aligned}
    r_t^{\operatorname{alg}}:=\truth(x^*)-\truth(x_t^*),\qquad x^*=\underset{x\in \Gamma}{\operatorname{arg\,max}}\,\, \truth(x),\quad x_t^* = \underset{x\in \{x_i\}_{i=1}^t}{\operatorname{arg\,max}}\,\, \truth(x),
\end{aligned}
\end{equation}
where $\operatorname{alg}\in\{\text{UCB,TS}\}$ denotes the algorithm used. Since the initial design does not affect the convergence rate of the algorithms, in our theory we assume without loss of generality that the algorithms are implemented without an initial design. Specifically, we assume that $N_{\mathrm{init}}=0$ and take by convention $\mu_0 = 0$ and $k_0(x,x') = k(x,x').$ 

Before stating the theorem, we recall that TS as described in Algorithm \ref{alg:GP-bandits} requires a choice of a finite subset $\Gamma_t$ at each iteration, which following \cite{chowdhury2017kernelized} will be chosen so that 
\begin{align*}
    |\truth(x)-\truth([x]_t)| \leq 1/t^2, \qquad \forall x\in \Gamma,
\end{align*}
where $[x]_t:=\operatorname{arg\, min}_{z\in \Gamma_t} d(x,z)$
is the point in $\Gamma_t$ closest to $x$. 
This can be achieved by imposing that $d(x,[x]_t)\leq (2S \tau^{-2} \Psi^2B^2t^4)^{-1}$ for all $x\in \Gamma$, where we recall that $S,\Psi$ are as in Lemma \ref{lemma:uniform boundedness of fem eigenfunction} and $B$ is an upper bound for $\|\truth\|_{\Hexact}.$ 
Indeed, under this condition we have
\begin{align*}
    |\truth(x)-\truth([x]_t)| &= |\langle \truth, k(\cdot,x)-k(\cdot,[x]_t)\rangle|\\
&\leq \| \truth \|_{\Hexact} \|k(\cdot,x)-k(\cdot,[x]_t)\|_{\Hexact}\\
&\leq B \sqrt{k(x,x)-2k(x,[x]_t)+k([x]_t,[x]_t)}\leq B\Psi  \tau^{-1} \sqrt{2S d(x,[x]_t)} \leq 1/t^2
\end{align*}
by Lemma \ref{lemma:uniform boundedness of fem eigenfunction}. 
Consequently, we define $\Gamma_t = \bigcup_{e\in\edgeset} P_{e,t}$, where 
\begin{align}\label{eq:gamma_t}
    P_{e,t} = \text{ uniform partitioning of } e \text{ with mesh size } (2S \tau^{-2} \Psi^2B^2t^4)^{-1}.
\end{align}
In particular, $\Gamma_t$ has size $(2S \tau^{-2} \Psi^2B^2t^4)\sum_{e\in \edgeset} L_e$. 

\begin{theorem}\label{thm:exact kernel setting}
Suppose $\truth\in \spectralsobolev^{2\alpha}(\Gamma)$ with $\alpha>\frac12$. 
Let $\delta\in(0,1)$, kernel $k$ be chosen as in \eqref{eq:kernel exact} with the same $\alpha$, $B=\|\truth\|_{\Hexact}$,  $\lambda=1+2/t$ in \cref{eq:postmean,eq:postcov,eq:poststd}, $R$ as the sub-Gaussian constant of the noise. Let $\Gamma_t\equiv \Gamma$ for ICP-UCB and $\Gamma_t$ be chosen as in \eqref{eq:gamma_t} for GP-TS. 
Then, with probability at least $1-\delta$,
\begin{align*}
    r_T^{\UCB} & = O\left(T^{\frac{1-2\alpha}{4\alpha}}\log T+T^{\frac{1-4\alpha}{4\alpha}}\sqrt{\log T}\big(\|\truth\|_{\spectralsobolev^{2\alpha}(\Gamma)}+\sqrt{\log(1/\delta)}\big)\right) \, , \\
    r_T^{\TS} & =O\left(\left[T^{\frac{1-2\alpha}{4\alpha}}\log T+T^{\frac{1-4\alpha}{4\alpha}}\sqrt{\log T}\big(\|\truth\|_{\spectralsobolev^{2\alpha}(\Gamma)}+\sqrt{\log(1/\delta)}\big)\right]\log(\|\truth\|_{\spectralsobolev^{2\alpha}(\Gamma)}^2T^6)\right) \, .
\end{align*}
\end{theorem}
\begin{proof}
The bound for $r_T^{\UCB}$ follows from \cite[Theorem 3]{chowdhury2017kernelized} with $B=\|\truth\|_{\Hexact}=\|\truth\|_{\spectralsobolev^{2\alpha}(\Gamma)}$ and $\gamma_T(k)=O(T^{1/{(4\alpha)}}\log T)$ as in Lemma \ref{lemma:maximum info gain bound exact}.
The bound for $r_t^{\TS}$ follows from Theorem \ref{thm:regret mis TS} setting additionally $b=0$, $|\Gamma_T| = (2S \tau^{-2} \Psi^2B^2T^4)\sum_e L_e$, and $\|\truth\|_\infty \leq \|k(\cdot,x)\|_{\Hexact}\|\truth\|_{\Hexact}\leq \overline{k}\|\truth\|_{\Hexact}$. 
\end{proof}

\section{Bayesian Optimization with Finite Element Kernel Representation}\label{sec:FEM}
The eigenpairs of the operator $\L$ in \eqref{eq:edge-operator-const} are typically unavailable for generic metric graphs.
Therefore, working with the kernel \eqref{eq:kernel exact} is often infeasible and numerical approximation is necessary. 
In this section, we consider the finite element approximation proposed in \cite{bolin2024regularity,bolin2020rational}, which, in addition to being readily computable, leads to efficient implementations for GP regression.

A major issue that arises in such numerical approximation is that the truth $\truth$ is no longer guaranteed to lie in the RKHS of the finite element kernel, so that Theorem \ref{thm:exact kernel setting} does not immediately apply.  
In this section, we address this issue through a careful design and analysis of the IGP-UCB and GP-TS algorithms with finite element kernels. 
We consider first the case where $2\alpha \in \N$ in Algorithm \ref{alg:GP-bandits-FEM} and then the fractional case $2\alpha \notin \N$ using rational approximations in Algorithm \ref{alg:GP-bandits-rational}, establishing regret bounds in Theorem \ref{thm:regret fem}. 

A central theme in this section is the need to understand how well the RKHS of the finite element and rational kernels approximate $\truth$ in order to correct for this approximation error in the Bayesian optimization algorithms. In addition to considering the computational misspecification introduced by finite element and rational approximations, our theory also covers epistemic misspecification arising when the smoothness of the objective $\truth$ is unknown and the smoothness parameter $\alpha$ of the kernel does not match the smoothness of the objective. 

\subsection{Finite Element Kernels}
\subsubsection{FEM space on $\Gamma$}
\label{ssec: fem space on Gamma}
To start with, we review the finite element construction on metric graphs \cite{arioli2018finite}.
At a high level, the construction proceeds by identifying each edge $e$ with the interval $[0,L_e]$ (see Sec.~\ref{ssec:problem statement}), over which one can build the standard 1D finite element spaces, with some additional care at the vertices. 

Consider the uniform partition of each edge $e$ into $n_e$ intervals of length $h_e$, leading to the nodes $\{z_j^e\}_{j=0}^{n_e}$ with 
$z_0^e=0$ and $z_{n_e}^e= L_e$. 
For each internal $z_j^e$ with $1\leq j\leq n_e-1$, denote by $\varphi_j^e$ the hat function 
\begin{align*}
    \varphi_j^e(z_e) = 
    \begin{cases}
        1-\frac{|z_j^e- z_e|}{h_e} & \text{ if } z_{j-1}^e \leq z_e \leq z_{j+1}^e\\
        0 & \text{ otherwise }
    \end{cases}, 
\end{align*}
which forms a basis for the space
\begin{align*}
    V_{h_e}^e = \left\{w\in H^1_0(e),w|_{[z_j^e,z_{j+1}^e]}\text{ is linear, } j= 0,\ldots,n_e-1\right\}. 
\end{align*}
Now for each vertex $v$, consider its neighboring set 
\begin{align*}
    \mathcal{W}_v = \left\{\bigcup_{e\in E_v,z^e_0=v}[v,z_1^e]\right\}\cup \left\{\bigcup_{e\in E_v,z^e_{n_e}=v}[z_{n_e-1}^e,v]\right\},
\end{align*}
i.e., the union of all edges that contain $v$. 
Define for each $v$ a function $\phi_v$ supported on $\mathcal{W}_v$ as 
\begin{align*}
    \phi_v(z_e)|_{\mathcal{W}_v\cap e} = 
    \begin{cases}
        1-\frac{|z_v^e-z_e|}{h_e} & \text{ if } z_e \in \mathcal{W}_v\cap e; e\in E_v\\
        0 &\text{ otherwise }
    \end{cases}, 
\end{align*}
where $z_v^e$ is either $0$ or $L_e$ depending on the direction of the edge and its parametrization.
The finite element space over $\Gamma$ is then defined as 
\begin{align}\label{eq:fem space}
    V_h = \left(\bigoplus_{e\in \edgeset} V_{h_e}^e\right) \oplus \operatorname{span}\{\phi_v\}_{v\in \vertexset},\qquad h:=\underset{e\in \edgeset}{\operatorname{max}}  \,\,h_e. 
\end{align}
For ease of exposition, we denote $V_h=\operatorname{span}\{e_{h,i}\}_{i=1}^{N_h}$, where $N_h:=\operatorname{dim}(V_h)$ is given by \(N_h = |\vertexset| + \sum_{e\in \edgeset}(n_e - 1)\asymp h^{-1}\), and $\femnodes=\{x_i\}_{i=1}^{N_h}$ is a quasi-uniform mesh.

\subsubsection{FEM Approximation}
With the FEM space constructed above, we are ready to define the finite element kernel. 
Recall that the operator $\mathcal L$ in \eqref{equa:prior gaussian} induces the bilinear form
\[
B(u,v) := \kappa^2\,\langle u,v\rangle_{L^2(\Gamma)} + \langle \nabla u, \nabla v \rangle_{L^2(\Gamma)} .
\]
Consider the  operator $\mathcal L_h:V_h\to V_h$ defined by
\[
  \langle \mathcal L_h u_h, v_h\rangle_{L^2(\Gamma)} \;=\; B(u_h,v_h)
  \qquad \forall\, u_h,v_h\in V_h .
\]
It can be shown (see e.g. \cite[Section 7]{boffi2010finite}) that the operator $\mathcal L_h$ admits eigenpairs
$\{(\lambda_{h,i},\psi_{h,i})\}_{i=1}^{N_h}$, which motivates the definition of a finite element approximation kernel to \eqref{eq:kernel exact} by 
\begin{equation}\label{eq:kh-nodewise}
  k_h(x,x') \;=\; \tau^{-2}\sum_{i=1}^{N_h} \lambda_{h,i}^{-2\alpha}\, \psi_{h,i}(x)\,\psi_{h,i}(x
'),
  \qquad x,x'\in\Gamma.
\end{equation}
Intuitively, \eqref{eq:kh-nodewise} should approximate \eqref{eq:kernel exact} as long as the FEM eigenpairs converge towards those of $\L$, which is indeed the case as established in Lemma \ref{lemma:L infnity spectral error}. 
The associated RKHS takes the form 
\begin{align}\label{eq:RKHS fem}
    \mathcal{H}_{k_h} = \left\{g=\sum_{i=1}^{N_h}a_i \psi_{h,i}, \quad \|g\|^2_{\mathcal{H}_{k_h}}:= \tau^2\sum_{i=1}^{N_h} a_i^2 \lambda_{h,i}^{2\alpha} <\infty  \right\} =  \operatorname{span} \{\psi_{h,1},\ldots,\psi_{h,N_h}\}=V_h,
\end{align}
where the latter space is precisely the FEM space. 
As mentioned above, a key ingredient in the design of Bayesian optimization algorithms is understanding the expressivity of the RKHS of the kernel used for computation. Since in this case the RKHS agrees with the FEM space, we can leverage the extensive literature on finite elements to address this question.

The definition of $k_h$ in \eqref{eq:kh-nodewise} relies on the eigenpairs of $\mathcal{L}_h,$ which may be expensive to compute. 
The next lemma, proved in Appendix \ref{sec:tech proof}, shows that when $\alpha$ is a half integer, $k_h$ can be computed by only working with the finite element basis $e_{h,j}$'s. 

\begin{lemma}\label{lemma:kernel in terms of fem basis}
Let $C,G\in \R^{N_h\times N_h}$ be the mass and stiffness matrices defined by 
\begin{align}\label{eq:mass and stifness}
    C_{ij}= \langle e_{h,i},e_{h,j}\rangle_{L^2(\Gamma)} ,\quad G_{ij} = \langle \nabla e_{h,i}, \nabla e_{h,j} \rangle_{L^2(\Gamma)}. 
\end{align}
If $2\alpha\in \mathbb{N}$, then 
\begin{align}\label{eq:fem kernel in mass form}
    k_h(x,x') = \tau^{-2} e(x)^\top Q^{-1} e(x'), \qquad Q=C [\kappa^2I+C^{-1}G]^{2\alpha},
\end{align}
where $e(x) = (e_{h,1}(x),\ldots,e_{h,N_h}(x))^\top$.
Furthermore, \cref{eq:postmean,eq:postcov,eq:poststd} can be written as 
\begin{align}
    \mu_{t-1}^h(x)&=e(x)^\top (\tau^2\lambda Q+E_{t-1}E_{t-1}^\top)^{-1}E_{t-1}Y_{t-1},\label{eq:fem postmean}\\
    k_{t-1}^h(x,x')&=\lambda e(x)^\top (\tau^2\lambda Q +E_{t-1}E_{t-1}^\top)^{-1} e(x'),\label{eq:fem postcov}\\
    \sigma_{t-1}^h(x)^2& =e(x)^\top (\tau^2\lambda Q +E_{t-1}E_{t-1}^\top)^{-1} e(x),\label{eq:fem poststd}
\end{align}
where $E_t=[e(x_1^{(0)}),\ldots, e(x^{(0)}_{N_\textrm{init}}), e(x_1),\ldots,e(x_t)]\in \R^{N_h\times (N_\textrm{init}+t)}$.
\end{lemma}

In particular, the computation is straightforward once the matrix $Q$ is available, which can be precomputed from the mass and stiffness matrices. 
 We recall that both $C$ and $G$ are sparse since the $e_{h,i}$'s have non-overlapping supports except for neighboring pairs. 
A lumped mass approximation can be employed by replacing the inverse of $C$ with that of a diagonal matrix $\widetilde{C}$ with entries $\widetilde{C}_{ii}=\sum_{j=1}^{N_h}C_{ij}$ to make $Q$ sparse, leading to efficient computation of all the above quantities as $E_t$ is sparse as well. 
If we further restrict attention to search for candidates in $\femnodes$ only, then $e(x)$ reduces to a standard basis vector in $\R^{N_h}$ and $E_tE_t^\top$ becomes a diagonal matrix of 0's and 1's. 

Our FEM approximated Bayesian optimization algorithm is summarized in Algorithm \ref{alg:GP-bandits-FEM}. We focus on optimizing $\truth$ over the FEM nodes $\femnodes$ due to its simple computation and demonstrate that this is often sufficient for applications in Section \ref{sec:Numerics}.
Notice that the main difference with Algorithm \ref{alg:GP-bandits} is the incorporation of a correction term in the parameters $\beta_t^h$ and $v_t^h$ that depend on the best approximation error of $\truth$ from $\mathcal{H}_{k_h}$,  which is rooted in the concentration-type results as in Lemma \ref{lemma:concentration}.  
In Theorem \ref{thm:regret fem}, we will establish regret bounds with an appropriate choice of parameter $b$ quantifying such approximation error.  
A similar correction for misspecified Bayesian optimization has been studied in \cite{bogunovic2021misspecified} for the IGP-UCB algorithm. Here, we extend their framework to also cover GP-TS: see Theorem \ref{thm:regret mis TS} in Appendix \ref{sec:misspecified BO}.

\begin{minipage}{0.95\linewidth}
\begin{algorithm}[H]
\caption{IGP-UCB and GP-TS with FEM Approximation ($2\alpha\in\mathbb{N}$)}
\label{alg:GP-bandits-FEM}
\begin{algorithmic}[1]
\Require FEM mesh nodes $\femnodes$, prior $\mathcal{GP}(0,k_h)$, parameters $\alpha,B,R,b,\lambda,\delta$, horizon $T$, initial design size $N_{\mathrm{init}}$.
\State Choose $X_{\mathrm{init}}=\{x_i^{(0)}\}_{i=1}^{N_{\mathrm{init}}}\subset\Gamma.$
\State Observe $y_i^{(0)} = \truth(x_i^{(0)}) + \varepsilon_i^{(0)}$, 
 with $\varepsilon_i^{(0)}$ being i.i.d.\ $R$-sub-Gaussian random variables for $i=1,\dots,N_{\mathrm{init}}$ (as in eq.~\eqref{eq:obsevation}). 
\State Initialize $\mcD_0 \gets \{(x_i^{(0)}, y_i^{(0)})\}_{i=1}^{N_{\mathrm{init}}}$.
\For{$t=1,2,\dots,T$}
\State Compute $\mu_{t-1}^h$, $k_{t-1}^h$ and $\sigma_{t-1}^h$ over $\femnodes$ given the first $t-1$ acquisitions and observations using \cref{eq:fem postmean,eq:fem postcov,eq:fem poststd}.  
\State Define the acquisition function $\mathrm{acq}_t^h(x)$ for $x\in \femnodes$ as
  \[
  \mathrm{acq}_t^h(x) =
  \begin{cases}
    \mu_{t-1}^h(x) + \beta_t^h \sigma_{t-1}^h(x), & \text{(IGP-UCB)}, \\[6pt]
    f_t^h(x), \; f_t^h \sim \mathcal{GP}(\mu_{t-1}^h,\, (v_t^h)^2 k_{t-1}^h), & \text{(GP-TS)},
  \end{cases}
  \]
  where, for $b\geq \operatorname{inf}_{f\in \mathcal{H}_{k_h}}\|\truth-f\|_{L^\infty(\Gamma)},$ 
  \begin{align*}
  \beta_t^h &=  B+ R\sqrt{2\big(\gamma_{t-1}(k_h)+(\Ninit+t)(\lambda-1)/2+\log(1/\delta)\big)}+\frac{b\sqrt{\Ninit+t-1}}{\sqrt{1+2/(\Ninit+t)}},  
  \\
  v_t^h &=  B+ R\sqrt{2\big(\gamma_{t-1}(k_h)+(\Ninit+t)(\lambda-1)/2+\log(2/\delta)\big)}+\frac{b\sqrt{\Ninit+t-1}}{\sqrt{1+2/(\Ninit+t)}}. 
  \end{align*}
\State Select $x_{t} \in \argmax_{x\in \femnodes} \mathrm{acq}_t^h(x)$.
\State Observe $y_t = \truth(x_{t})+\varepsilon_t$, 
 with $\varepsilon_t$ defined as in eq.~\eqref{eq:obsevation}.
\State Update $\mcD_t \gets \mcD_{t-1} \cup \{(x_t, y_t)\}$.
\EndFor
\end{algorithmic}
\end{algorithm}
\end{minipage}

\begin{remark}\label{remark:match with package}
Similarly as in Section \ref{sec:exact kernel}, the above kernel can be derived from the approximate SPDE
\begin{equation}\label{eq:fem-spde}
  \mathcal L_h^{\alpha}\,(\tau u_h) \;=\; \mathcal W_h,
  \qquad
  \mathcal W_h \;=\; \sum_{i=1}^{N_h} \xi_i\, \psi_{h,i}. 
\end{equation}
We remark that this construction coincides with the one proposed in \cite[Section 6.5]{bolin2024regularity} in that their adopted white noise takes the form $\widetilde{\mathcal W}_h=\sum_{i=1}^{\infty} \xi_i\, P_h \psi_i$, where $P_h:L^2(\Gamma)\to V_h$ is the Galerkin projection, and is equal in distribution to $\mathcal{W}_h$ above (see Lemma \ref{lemma:equal white noise}). 
\end{remark}

\subsubsection{Rational Approximation}
The approximation described in the last section leads to a viable algorithm when $2\alpha\in \mathbb{N}$ that can be computed efficiently. 
For general $\alpha$'s, a rational approximation \cite{bolin2020rational} can be applied to retain sparsity and avoid computing the spectral decomposition.
Let $m_\alpha = 1\vee \lfloor\alpha \rfloor$ and for a closed interval $I$ consider 
\begin{align}\label{eq:best rational approx}
    \widehat{s}_I := 
    \underset{\substack{\operatorname{deg}(q_1)=m\\ \operatorname{deg}(q_2)=m+1}}{\operatorname{arg\,min}} \left\| z^{|\alpha-m_\alpha|}-\frac{q_1(z)}{q_2 (z) }\right\|_{L^\infty(I)} \, ,
\end{align}
where $m\ge 1$\ is an integer approximation order.

Define 
\begin{equation}\label{eq:s_h}
\begin{aligned}
s_h(z)= z^{m_\alpha}\cdot
\begin{cases}
    \widehat{s}_{J_h}(z) & \text{ if } \alpha-m_\alpha>0,\\
    \lambda_{h,N_h}^{|\alpha-m_\alpha|} \widehat{s}_{[0,1]}(\lambda_{h,N_h}^{-1}z^{-1})& \text{ if } \alpha-m_\alpha<0,
\end{cases}
\end{aligned}
\end{equation}
where $J_h=[\lambda_{h,N_h}^{-1},\lambda_{h,1}^{-1}]$. 
As shown in \cite{bolin2020rational}, $s_h$ is a good approximation of $x^\alpha$ over $J_h$ so that $s_h(\lambda_{h,i}^{-1})$ approximates well $\lambda_{h,i}^{-\alpha}$ for $i=1,\ldots,N_h$ and $\widehat{s}_I$ can be computed efficiently using for instance the algorithm proposed by \cite{harizanov2018optimal}. 
The rational approximated kernel is then defined as 
\begin{align}\label{eq:rational kernel}
    k_h^{\texttt{r}}(x,x') = \tau^{-2}\sum_{i=1}^{N_h}s_h(\lambda_{h,i}^{-1})^2 \psi_{h,i}(x)\psi_{h,i}(x'),
\end{align}
whose RKHS takes the form
\begin{align}\label{eq:RKHS rational}
    \mathcal{H}_{k_h^{\texttt{r}}} = \left\{g=\sum_{i=1}^{N_h}a_i \psi_{h,i}, \quad \|g\|^2_{\mathcal{H}_{k_h^{\texttt{r}}}}:= \tau^2\sum_{i=1}^{N_h} a_i^2 s_h(\lambda_{h,i}^{-1})^{-2} <\infty  \right\} =  \operatorname{span} \{\psi_{h,1},\ldots,\psi_{h,N_h}\}=V_h.
\end{align}
Notice that $\mathcal{H}_{k_h^{\texttt{r}}}$ differs from $\mathcal{H}_{k_h}$ in \eqref{eq:RKHS fem} only by the norms.

Similarly as in Lemma \ref{lemma:kernel in terms of fem basis}, it is possible to compute $k_h^{\texttt{r}}$ without performing a spectral decomposition of $\L_h$. 
To see this, write $s_h(z) =z^{m_\alpha} q_1(z)/q_2(z)$ for some $\operatorname{deg}(q_1)=m$ and $\operatorname{deg}(q_2)=m+1$. 
We have then 
\begin{align*}
    s_h(z^{-1}) = \frac{q_1(z^{-1})}{z^{m_\alpha}q_2(z^{-1})} = \frac{q_1(z^{-1})z^m}{z^{m_\alpha}q_2(z^{-1})z^m}=: \frac{p_r(z)}{p_\ell(z)},
\end{align*}
where $p_\ell,p_r$ are polynomials of degree at most $m+m_\alpha$ and $m$ respectively. 
Let 
\begin{align}\label{eq:Pl Pr}
    P_\ell:=p_\ell(\kappa^2+C^{-1}G),\quad P_r :=p_r(\kappa^2+C^{-1}G),
\end{align}
where $C,G$ are the mass and stiffness matrices defined in \eqref{eq:mass and stifness}.  
The next lemma, proved in Appendix \ref{sec:tech proof}, presents formulae for $k_h^{\texttt{r}}$ as well as the posterior mean and covariance in terms of the matrices $P_\ell,P_r,C.$ 

\begin{lemma}\label{lemma:rational kernel in terms of fem basis}
Let $e(x)$ and $E_{t}$ be as in Lemma \ref{lemma:kernel in terms of fem basis}. 
We have 
\begin{align*}
    k_h^{\texttt{r}}(x,x') = \tau^{-2}e(x)^\top P_r (P_\ell^\top C P_\ell)^{-1} P_r^\top e(x'), 
\end{align*}
and 
\begin{align}
    \mu_{t-1}^{h,\texttt{r}}& = e(x)^\top P_r\left(\tau^2\lambda P_\ell^\top CP_\ell +E_{t-1}P_rP_r^\top E_{t-1}^\top\right)^{-1}E_{t-1}P_rY_{t-1},\label{eq:rational postmean}\\
    k_{t-1}^{h,\texttt{r}}(x,x')&= \lambda e(x)^\top P_r \left(\tau^2\lambda P_\ell^\top CP_\ell+E_{t-1}P_rP_r^\top E_{t-1}^\top \right)^{-1}P_r^\top e(x'),\label{eq:rational postcov}\\
    \sigma_{t-1}^{h,\texttt{r}}(x)^2&= \lambda e(x)^\top P_r \left(\tau^2\lambda P_\ell^\top CP_\ell+E_{t-1}P_rP_r^\top E_{t-1}^\top \right)^{-1}P_r^\top e(x).\label{eq:rational poststd}
\end{align}
\end{lemma}

 As before, if a lumped mass approximation is applied to  $C^{-1}$, then both $P_\ell$ and $P_r$ are sparse so that the calculations in \cref{eq:rational postmean,eq:rational postcov,eq:rational poststd} only involve sparse matrices. 
Our rational FEM approximated Bayesian optimization algorithm for addressing $2\alpha\notin \mathbb{N}$ is presented in Algorithm \ref{alg:GP-bandits-rational}.

\begin{minipage}{0.95\linewidth}
\begin{algorithm}[H]
\caption{IGP-UCB and GP-TS with Rational FEM Approximation ($2\alpha\notin\mathbb{N}$)}
\label{alg:GP-bandits-rational}
\begin{algorithmic}[1]
\Require FEM mesh nodes $\femnodes$, prior $\mathcal{GP}(0,k_h^{\texttt{r}})$, parameters $\alpha,m,B,R,b,\lambda,\delta$, horizon $T$, initial design size $N_{\mathrm{init}}$.
\State Compute $s_h$ defined in \eqref{eq:s_h} and the matrices $P_\ell,P_r$ in \eqref{eq:Pl Pr}.
\State Choose $X_{\mathrm{init}}=\{x_i^{(0)}\}_{i=1}^{N_{\mathrm{init}}}\subset\Gamma.$
\State Observe $y_i^{(0)} = \truth(x_i^{(0)}) + \varepsilon_i^{(0)}$, 
with $\varepsilon_i^{(0)}$ being i.i.d.\ $R$-sub-Gaussian random variables for $i=1,\dots,N_{\mathrm{init}}$ (as in eq.~\eqref{eq:obsevation}).
\State Initialize $\mcD_0 \gets \{(x_i^{(0)}, y_i^{(0)})\}_{i=1}^{N_{\mathrm{init}}}$.
\For{$t=1,2,\dots,T$}
\State Compute $\mu_{t-1}^{h,\texttt{r}}$, $k_{t-1}^{h,\texttt{r}}$ and $\sigma_{t-1}^{h,\texttt{r}}$ over $\femnodes$ given the first $t-1$ acquisitions and observations using \cref{eq:rational postmean,eq:rational postcov,eq:rational poststd}.  
\State Define the acquisition function $\mathrm{acq}_t^{h,\texttt{r}}(x)$ for $x\in \femnodes$ as
  \[
  \mathrm{acq}_t^{h,\texttt{r}}(x) =
  \begin{cases}
    \mu_{t-1}^{h,\texttt{r}}(x) + \beta_t^{h,\texttt{r}} \sigma_{t-1}^{h,\texttt{r}}(x), & \text{(IGP-UCB)}, \\[6pt]
    f_t^{h,\texttt{r}}(x), \; f_t^{h,\texttt{r}} \sim \mathcal{GP}(\mu_{t-1}^{h,\texttt{r}},\, (v_t^{h,\texttt{r}})^2 k_{t-1}^{h,\texttt{r}}), & \text{(GP-TS)},
  \end{cases}
  \]
  where, for $b\geq \operatorname{inf}_{f\in \mathcal{H}_{k_h^{\texttt{r}}}}\|\truth-f\|_{L^\infty(\Gamma)},$
  \begin{align*}
  \beta_t^{h,\texttt{r}} & = B+ R\sqrt{2\big(\gamma_{t-1}(k_h^{\texttt{r}})+(\Ninit+t)(\lambda-1)/2+\log(1/\delta)\big)}+\frac{b\sqrt{\Ninit+t-1}}{\sqrt{1+2/(\Ninit+t)}},  
  \\
  v_t^{h,\texttt{r}} &=  B+ R\sqrt{2\big(\gamma_{t-1}(k_h^{\texttt{r}})+(\Ninit+t)(\lambda-1)/2+\log(2/\delta)\big)}+\frac{b\sqrt{\Ninit+t-1}}{\sqrt{1+2/(\Ninit+t)}}. 
  \end{align*}
\State Select $x_{t} \in \argmax_{x\in \femnodes} \mathrm{acq}_t^{h,\texttt{r}}(x)$.
\State Observe $y_t = \truth(x_{t})+\varepsilon_t$, 
with $\varepsilon_t$ defined as in eq.~\eqref{eq:obsevation}.
\State Update $\mcD_t \gets \mcD_{t-1} \cup \{(x_t, y_t)\}$.
\EndFor
\end{algorithmic}
\end{algorithm}
\end{minipage}

\subsection{Regret Bounds}
Now we are ready to present the regret bounds for both FEM approximated IGP-UCB and GP-TS.
We remark that the analysis for Algorithm \ref{alg:GP-bandits-FEM} and Algorithm \ref{alg:GP-bandits-rational} is similar so we shall only present that for Algorithm \ref{alg:GP-bandits-rational} as it is more general. 
Recall the simple regret defined as
\begin{align*}
    r_t^{\operatorname{alg}}:=\truth(x^*)-\truth(x_t^*),\qquad x^*=\underset{x\in \femnodes}{\operatorname{arg\,max}}\,\, \truth(x),\quad x_t^* = \underset{x\in \{x_i\}_{i=1}^t}{\operatorname{arg\,max}}\,\, \truth(x)
\end{align*}
with $\operatorname{alg}\in\{\operatorname{UCB},\operatorname{TS}\}$,  where now the optimization is over $\femnodes,$ as in Algorithm \ref{alg:GP-bandits-rational}. Here again we assume without loss of generality that $N_{\mathrm{init}}=0,$  taking by convention $\mu_0 = 0$ and $k_{0}^{h,\texttt{r}}(x,x') = k_h^{\texttt{r}}(x,x').$

\begin{theorem}\label{thm:regret fem}
Suppose $\truth\in \dot{H}^{2\beta}(\Gamma)$ for $2\beta>1$. 
For $2\alpha>1$, setting in Algorithm \ref{alg:GP-bandits-rational}
\begin{align*}
    B\asymp \|\truth\|_{\dot{H}^{2\beta}(\Gamma)} h^{(2\beta-2\alpha) \vee 0 },\,\,  b\asymp h^{(2\beta-1)\wedge \frac32}, \,\, \pi\sqrt{|\alpha-m_\alpha|m}\gtrsim -(1\vee\alpha)\log h,\,\, \lambda=1+2/t,
\end{align*}
and $R$ the sub-Gaussian constant of the noise,
we have, with probability at least $1- \delta,$ 
\begin{align*}
    r_T^{\UCB}&=O\left(T^{\frac{-(4\alpha-3)}{8\alpha-2}}\log T+T^{\frac{-(2\alpha-1)}{4\alpha-1}}h^{(2\beta-2\alpha)\wedge 0 }\sqrt{\log T}+ h^{(2\beta-1)\wedge \frac32}T^{\frac{1}{8\alpha-2}}\right) \, ,\\
    r_T^{\TS}&=O\left(\left[T^{\frac{-(4\alpha-3)}{8\alpha-2}}\log T+T^{\frac{-(2\alpha-1)}{4\alpha-1}}h^{(2\beta-2\alpha)\wedge 0 }\sqrt{\log T}+ h^{(2\beta-1)\wedge \frac32}T^{\frac{1}{8\alpha-2}}\right]\log(T^2/h)\right) \, .
\end{align*}
If $\alpha=\beta$, i.e., the smoothness of the truth matches with that of the kernel used for computation, then the regret bound reduces to 
\begin{align*}
    r_T^{\UCB}&= O\left(T^{\frac{-(4\alpha-3)}{8\alpha-2}}\log T+ h^{(2\alpha-1)\wedge \frac32}T^{\frac{1}{8\alpha-2}}\right) \, ,\\
    r_T^{\TS}&= O\left(\left[T^{\frac{-(4\alpha-3)}{8\alpha-2}}\log T+ h^{(2\alpha-1)\wedge \frac32}T^{\frac{1}{8\alpha-2}}\right]\log (T^2/h)\right) \, . 
\end{align*}
\end{theorem}

Notice that the error component $h^{(2\beta-1)\wedge \frac32}T^{\frac{1}{8\alpha-2}}$ stays small as long as $T\ll N_h^{(8\alpha-2)[(2\beta-1)\vee \frac32]}$. For $\alpha,\beta>\frac34$, it suffices to require $T\ll N_h^2$, which is enough for our algorithm since the search domain only has size $N_h$. 
Therefore, a vanishing regret is achieved under mild assumptions, although a large proportional constant may be present when $\alpha$ is chosen much larger than $\beta$ due to the term $h^{(2\beta-2\alpha)\wedge 0}$.  

\begin{proof}[Proof of Theorem \ref{thm:regret fem}]
Let $\truth=\sum_{i=1}^\infty \langle \truth, \psi_i\rangle\psi_i$, where $\truth\in \dot{H}^{2\beta}(\Gamma)$ implies that 
$\sum_{i=1}^\infty \langle \truth, \psi_i\rangle^2 \lambda_i^{2\beta} <\infty.$
Consider the approximation 
\begin{align*}
    f_h = \sum_{i=1}^{N_h} \langle \truth, \psi_i\rangle \psi_{h,i} \in \mathcal{H}_{k_h^{\texttt{r}}}.
\end{align*}
By Lemma \ref{lemma:L infnity spectral error}, we have 
\begin{align*}
    \|\truth-f_h\|_{L^{\infty}(\Gamma)} &\leq \left\|\sum_{i=1}^{N_h}\langle \truth, \psi_i\rangle (\psi_i-\psi_{h,i})\right\|_{L^{\infty}(\Gamma)}+ \left\|\sum_{i=N_h+1}^\infty \langle \truth, \psi_i\rangle \psi_i\right\|_{L^{\infty}(\Gamma)}\\
    &\lesssim \sum_{i=1}^{N_h} |\langle \truth, \psi_i\rangle| \lambda_i h^{3/2} + \sum_{i=N_h+1}^\infty |\langle \truth, \psi_i\rangle|\\
    & = h^{3/2}\sum_{i=1}^{N_h} \left(|\langle \truth, \psi_i\rangle|\lambda_i^{\beta}\right) \lambda_i^{1-\beta} + \sum_{i=N_h+1}^\infty \left(|\langle \truth, \psi_i\rangle|\lambda_i^{\beta} \right)\lambda_i^{-\beta}\\
    &\leq h^{3/2}\sqrt{\sum_{i=1}^{N_h} |\langle \truth, \psi_i\rangle|^2\lambda_i^{2\beta}}\sqrt{\sum_{i=1}^{N_h}\lambda_i^{2-2\beta}} + \sqrt{\sum_{i=N_h+1}^\infty |\langle \truth, \psi_i\rangle|^2\lambda_i^{2\beta} }\sqrt{\sum_{i=N_h+1}^\infty \lambda_i^{-2\beta}}\\ 
    & \lesssim h^{3/2}\sqrt{\int_1^{N_h} w^{4-4\beta}dw} + \sqrt{\int_{N_h}^\infty w^{-4\beta}dw}\\
    &\lesssim h^{3/2}N_h^{(5/2-2\beta)\vee 0} + N_h^{1/2-2\beta}\lesssim h^{(2\beta-1)\wedge \frac32},
\end{align*}
which further implies that $\|f_h\|_{L^{\infty}(\Gamma)} \lesssim \|\truth\|_{L^{\infty}(\Gamma)} \lesssim \|\truth\|_{\dot{H}^{2\beta}(\Gamma)}$  since $\dot{H}^{2\beta}(\Gamma)$ continuously embeds into $C(\Gamma)$ for $2\beta>1$ by \cite[Theorem 4.1]{bolin2024regularity}. 

Let $s_h$ be defined in \eqref{eq:s_h}.
By \cite[Appendix B]{bolin2020rational}, we have 
\begin{align*}
    \underset{i=1,\ldots,N_h}{\operatorname{max}}\,\, |\lambda_{h,i}^{-\alpha} - s_h(\lambda_{h,i}^{-1})| \lesssim \lambda_{h,N_h}^{(1-\alpha)\vee 0} e^{-2\pi \sqrt{|\alpha-m_\alpha|m}}\leq \frac12 \lambda_{h,i}^{-\alpha}
\end{align*}
when $m$ is chosen as in the statement of the theorem for a sufficiently large proportion constant.
As a result, we obtain $s_h(\lambda_{h,i}^{-1})\geq \frac12 \lambda_{h,i}^{-\alpha}$ and $s_h(\lambda_{h,i}^{-1})^{-1}\leq 2\lambda_{h,i}^\alpha\lesssim \lambda_i^\alpha$.
This gives 
\begin{align*}
    \|f_h\|_{\mathcal{H}_{k_h^{\texttt{r}}}}^2 = \sum_{i=1}^{N_h} \langle \truth, \psi_i\rangle^2 s_h(\lambda_{h,i}^{-1})^{-2} &\lesssim \sum_{i=1}^{N_h} \langle \truth, \psi_i\rangle^2 \lambda_i^{2\alpha}  \\
    &=  \sum_{i=1}^{N_h} \langle \truth, \psi_i\rangle^2 \lambda_i^{2\beta} \lambda_i^{2\alpha-2\beta}\lesssim \|\truth\|^2_{\dot{H}^{2\beta}(\Gamma)} h^{(4\beta-4\alpha) \vee 0 }. 
\end{align*}
Applying Theorem \ref{thm:regret mis TS} to $f_h$ with 
\begin{align*}
    B\asymp \|\truth\|_{\dot{H}^{2\beta}(\Gamma)} h^{(2\beta-2\alpha) \vee 0 },\quad  b\asymp h^{(2\beta-1)\wedge \frac32},\quad \|f_h\|_{L^{\infty}(\Gamma)} \asymp 1, \quad |D_T| = N_h
\end{align*}
and $\gamma_T(k_h^{\texttt{r}})$ as in Lemma \ref{lemma:maximum info gain bound}, we obtain 
\begin{align*}
    r_T^{\UCB}&=O\left(\frac{\gamma_T(k_h^{\texttt{r}})}{\sqrt{T}}+\sqrt{\frac{\gamma_T(k_h^{\texttt{r}})}{T}}\left(B+\sqrt{\log(1/\delta)}\right)+b\sqrt{\gamma_T(k_h^{\texttt{r}})}\right)\\
    &=O\left(T^{\frac{-(4\alpha-3)}{8\alpha-2}}\log T+T^{\frac{-(4\alpha-2)}{8\alpha-2}}h^{(2\beta-2\alpha)\vee 0 }\sqrt{\log T}+ h^{(2\beta-1)\wedge \frac32}T^{\frac{1}{8\alpha-2}}\right),
\end{align*}
and 
\begin{align*}
    r_T^{\TS}&=O\left(\sqrt{\log (|D_T|T^2)}\left[\frac{\gamma_T(k_h^{\texttt{r}})}{\sqrt{T}}+\sqrt{\frac{\gamma_T(k_h^{\texttt{r}})}{T}}\left(B+\|f_h\|_\infty\sqrt{\log(1/\delta)}\right)+b\sqrt{\gamma_T(k_h^{\texttt{r}})}\right]\right)\\
    &=O\left(\left[T^{\frac{-(4\alpha-3)}{8\alpha-2}}\log T+T^{\frac{-(4\alpha-2)}{8\alpha-2}}h^{(2\beta-2\alpha)\vee 0 }\sqrt{\log T}+ h^{(2\beta-1)\wedge \frac32}T^{\frac{1}{8\alpha-2}}\right]\log(T^2/h)\right).
\end{align*}
\end{proof}

\section{Numerical Experiments}\label{sec:Numerics}
This section investigates the effectiveness of IGP-UCB and GP-TS on benchmark objective functions defined over a synthetic metric graph (Subsection \ref{ssec:benchmarks}) and for \emph{maximum a posteriori} estimation in a source-identification Bayesian inverse problem on a telecommunication network (Subsection \ref{ssec:MAP}). We compare two choices of kernel:
\begin{enumerate}
    \item {\bf SPDE kernel:} the FEM kernel $k_h$ \eqref{eq:kh-nodewise} defined via the SPDE \nc \eqref{eq:fem-spde} with \(\alpha=1\); and  
    \item {\bf Euclidean kernel:} A standard Matérn kernel with smoothness parameter \(\nu=\tfrac{1}{2}\) computed using the  Euclidean distance between graph points, given by 
\begin{equation}
\label{equa: eucl baseline}
    k_{\mathrm{Eucl}}(x,x')=\sigma^{2}\exp\!\Big(-\tfrac{|x - x'|}{\ell}\Big),
\qquad \ell>0,\ \sigma>0 \, .
\end{equation}
We remark that in all our examples, the networks we consider are naturally embedded in $\R^2,$ and points in the metric graphs are naturally identified with points in $\R^2.$ 
\end{enumerate}
We adopt Algorithm~\ref{alg:GP-bandits-FEM} supplemented with an additional layer for maximum likelihood estimation of kernel parameters, as summarized in Algorithm \ref{alg: benchmarks}. Throughout our experiments, we fix the amplitude parameters at $\sigma = \tau = 1,$ and estimate only the parameters $\kappa$ and $\ell$ controlling the correlation lengthscale for the SPDE and Euclidean kernels, respectively. This choice was supported by sensitivity tests with random perturbations of \(\tau\) and \(\sigma\), which showed negligible impact on the results.
We initialize the parameters for the maximum likelihood procedure at \(\ell_0=0.25 \cdot \mathrm{diam}(\Gamma)\)  and \(\kappa_0=1/\ell_0\) using the diameter of \(\Gamma\) in \emph{shortest-path distance}, which yields comparable early-stage exploration for SPDE and Euclidean kernels.
To construct the metric graphs, the FEM mesh, and the discretized SPDE kernel \(k_h\) with $\alpha=1,$ we use the \textsf{MetricGraph}  R package~\citep{MetricGraph} whose implementation for half integer $\alpha$'s is based on \cite{bolin2024regularity}. This aligns with our theoretical analysis above as mentioned in Remark \ref{remark:match with package}.  

We evaluate performance using three complementary metrics: average simple regret,  reach rate, and iterations to $\mathsf{Tol}$.
\emph{Average simple regret} represents the simple regret averaged across $N_{\mathrm{rep}}$ repetitions of the experiment with different initial designs.  Each of these designs is obtained via the maximin selection rule  discussed in Remark~\ref{rmk:maximin design}.
The \emph{reach rate} represents the fraction of runs that achieve simple regret smaller than a given tolerance $\mathsf{Tol}$ within the horizon, and the \emph{iterations to $\mathsf{Tol}$}, represents the number of iterations (excluding initialization) required for the $j$-th repetition of the experiment, if successful, to first cross the tolerance threshold $\mathsf{Tol}$. 
Let 
\begin{align*}
    x^\ast := \underset{x\in \Gamma}{\operatorname{arg\,max}}\, \truth(x),\qquad  x_{t,\ast}^{(j)} := \underset{x\in \{x_i^{(j)}\}_{i=1}^t}
    {\operatorname{arg\,max}}\, \truth(x) 
\end{align*}
denote respectively the global maximizer of the true objective and the best point found after \(t\) acquisitions in the run starting from the \(j\)-th initialization. The three performance metrics can then be expressed as:
\begin{align*}
&\text{(Simple regret):} \quad \overline{r}^{\rm alg}_t \;:=\; \frac{\sum_{j=1}^{N_{\mathrm{rep}}} r_t^{{\rm alg},\,(j)}}{N_{\mathrm{rep}}}= \; \frac{\sum_{j=1}^{N_{\mathrm{rep}}}\Big(\truth(x^\ast)- \truth\!\bigl(x_{t,\ast}^{(j)}\bigr)\Big)}{N_{\mathrm{rep}}},
\quad (t=1,\dots,T),\\
&\text{(Reach rate):} \quad
\rho_{\mathsf{Tol}} \;:=\; \frac{1}{N_{\mathrm{rep}}}\sum_{j=1}^{N_{\mathrm{rep}} }
\mathbf{1}\!\left\{\min_{1\le t\le T} r_t^{{\rm alg},\,(j)} \le \mathsf{Tol} \right\},\\
&\text{(Iterations to $\mathsf{Tol}$):} \quad
N_{\mathsf{Tol}}^{(j)} \;=\;
\min\bigl\{\, t \in \{1,\dots,T\} \;:\; r_t^{{\rm alg},\,(j)} \le \mathsf{Tol} \bigr\},\qquad j \in \mathcal{J}_{\mathrm{succ}},
\end{align*}
where \(\mathcal J_{\mathrm{succ}} := \bigl\{\, j \in \{1,\dots,N_{\mathrm{rep}}\} : r_T^{\mathrm{alg},(j)} \le \mathsf{Tol} \, \bigr\}.\)
Throughout the numerical tests, we set the tolerance threshold to be $\mathsf{Tol}=10^{-6}$.

\subsection{Benchmark Functions}\label{ssec:benchmarks}
\subsubsection{Problem Setting}
In Euclidean domains, benchmark objective functions such as Ackley, Rastrigin, and Lévy are commonly used to assess the performance of optimization algorithms.
These classical benchmark objectives have well-understood landscapes with known global minimizers, oscillatory structure, and multiple local optima, which makes them suitable for evaluating an algorithm’s convergence speed, ability to escape local minima, and overall robustness. However, extending these benchmarks to our compact metric graph setting is not straightforward, as one must ensure that the resulting objectives remain globally continuous, particularly across vertices where multiple edges meet. 

To achieve this, we construct benchmark functions on metric graphs by composing two functions: an outer one-dimensional benchmark \(g:\mathbb{R}\to\mathbb{R}\) and an inner interpolation map \(a=\bigoplus_{e\in E} a_e:\Gamma\to\mathbb{R}\). Specifically, we assign each vertex \(v\) an anchor value \(a_v \in [a_{\min},a_{\max}]\), which can be chosen arbitrarily. For an edge \(e\) of length \(L_e\) connecting vertices \(v_i\) and \(v_j\), parameterized by \(z_e\in[0,L_e]\), we interpolate
\[
a_e(x) = \Bigl(1-\tfrac{z_e}{L_e}\Bigr)a_{v_i} + \tfrac{z_e}{L_e}a_{v_j}, \qquad x=(e,z_e)\in  \Gamma. 
\]
We then evaluate the benchmark at  \(x=(e,z_e)\in \Gamma\) via composition,
\[
 \truth  (x) = g\bigl(a_e(x)\bigr),
\]
where \(g\) is a classical one-dimensional  benchmark function mentioned above. Proceeding similarly for all \(e\in E\), we obtain a benchmark on \(\Gamma\), i.e.,
\( \truth  (x)=g\bigl(a(x)\bigr), \, \forall x\in \Gamma.
\) In this formulation, the vertex anchors define the inputs of the inner interpolation map, 
while the outer benchmark $g$ determines the objective landscape. 
Because the interpolation is continuous across edges, 
the resulting function  $\truth$  is globally continuous on $\Gamma$, 
ensuring that values approaching a vertex from different incident edges always coincide.
\begin{figure}[htbp]
    \centering
    \includegraphics[width=0.875\linewidth]{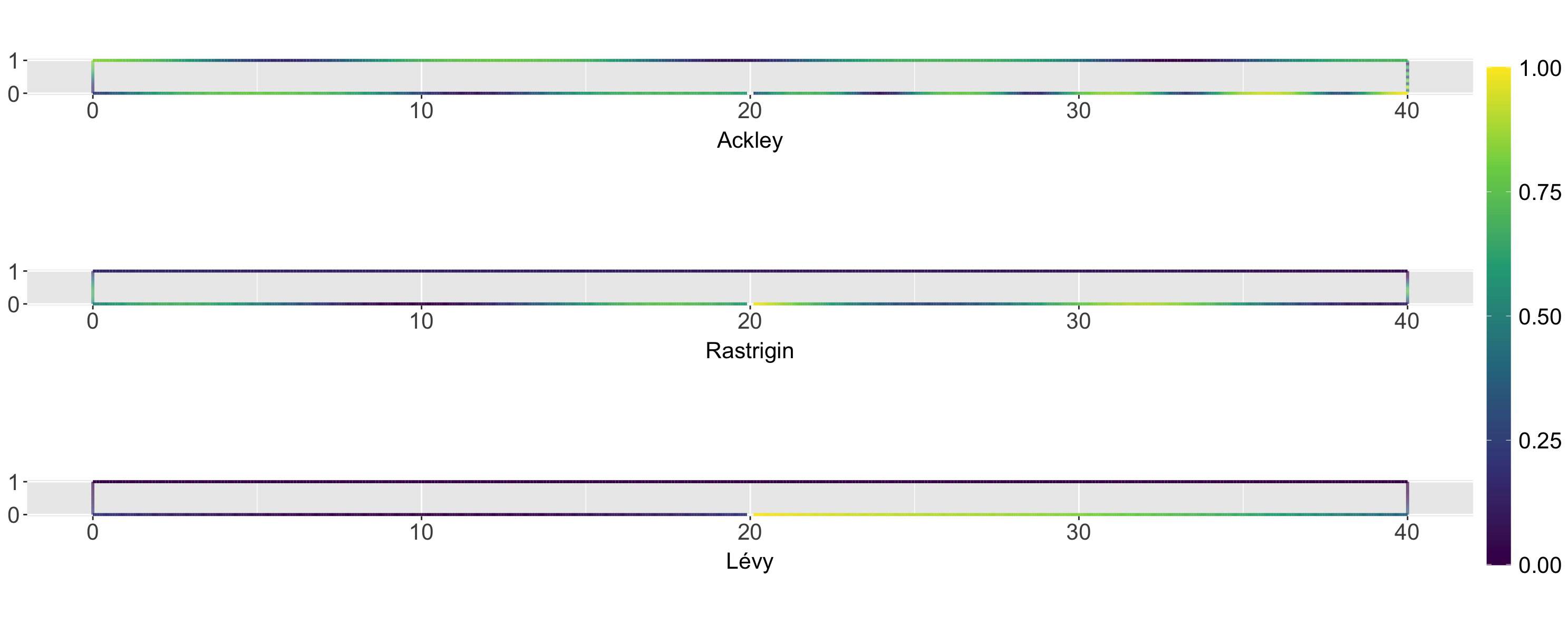}
  \caption{Ackley, Rastrigin, and Lévy benchmark functions on a compact metric graph with the shape of an open rectangle. Notice the small opening in the middle of the bottom side.}
  \label{fig:true-benchmarks-openrectangle}
\end{figure}

\subsubsection{Numerical Results}
\label{ssec:numerical results}
We consider benchmark functions on a compact metric graph shaped as an open rectangle, as shown in Figure~\ref{fig:true-benchmarks-openrectangle}. Despite its simplicity, the open rectangle graph induces pronounced distance distortion: many points in the graph are close in Euclidean distance yet far apart in shortest-path distance, because of the small height of the rectangle and its tiny opening. As we shall see, this metric distortion makes Euclidean kernels inadequate for Bayesian optimization. We consider three benchmarks, constructed using the classical Ackley, Rastrigrin, and L\'evy objectives, which we recall are given by 
\[
\begin{aligned}
\text{Ackley (1D):}\quad 
&g_{\mathrm{Ack}}(x)
= -20\,\exp\!\big(-0.2\,|x|\big)\;-\;\exp\!\big(\cos(2\pi x)\big)\;+\;20+e,\\[0.35em]
\text{Rastrigin (1D):}\quad 
&g_{\mathrm{Ras}}(x)
= x^2 - 10\cos(2\pi x) + 10,\\[0.35em]
\text{L\'evy (1D):}\quad 
&
g_{\mathrm{Lev}}(x)
= \sin^2(\pi w) + (w-1)^2\bigl[\,1+\sin^2(2\pi w)\,\bigr],\quad  w := 1+\frac{x-1}{4}.
\end{aligned}
\]
On the open rectangle graph, the benchmarks can be viewed as stretched-and-shifted versions of their classical Euclidean counterparts. They pose distinct challenges: Ackley is highly multimodal, Rastrigin exhibits dense small-scale oscillations, and Lévy shows sharp variation across the opening. Euclidean kernels “short-cut’’ across this opening, spuriously correlating nodes that are far apart in shortest-path distance; as our experiments show, this leads to poor optimization results.

Following Theorem \ref{thm:regret fem}, the hyperparameters \(B\) and \(R\) denote an a priori bound on \(\|\truth\|_{\mathcal H_k}\) and the sub-Gaussian noise parameter, respectively. 
Because the kernel hyperparameters are learned online via MLE, making the value of \(B\) to be not identifiable, we fix \(B=1\) after normalizing the objective so that prior marginal variances are \(O(1)\).
 To correct for misspecification, we need to choose the parameter $b$, which requires the unknown smoothness parameter of $\truth$.  
\red In practice, one can absorb the additional exploration term
\(\frac{b\sqrt{\Ninit+t-1}}{\sqrt{1+2/(\Ninit+t)}}\)
 in  Algorithm \ref{alg: benchmarks} 
into the overall exploration strength by choosing a slightly more conservative effective value of $B$.
Under our normalization, this provides a practical and conservative choice.
It does not mean that the $b$-term disappears.
Rather, it folds the extra uncertainty induced by misspecification into the constant controlling exploration.
This is supported by the sensitivity study in Appendix~\ref{app:b-sensitivity}, where we vary $b$ and observe that the reach rate, iterations-to-$\mathsf{Tol}$, and simple-regret curves change only mildly across the tested range.
Accordingly, we conclude that $B=1$ is adequate for the normalized objective in our experimental regime, and we adopt this choice throughout.
However, when misspecification is more severe, explicitly tuning $b$ may be beneficial for improving robustness. \nc
 We take the observation noise \(\varepsilon_t\sim\mathcal{N}(0,\sigma_\varepsilon^2)\) to be Gaussian,  setting \(R=\sigma_\varepsilon\), and choose \(\delta=0.05\) to match a nominal \(95\%\) confidence level. 
In our experiments, we set \(\sigma_\varepsilon=0.05\), commensurate with the unit-scaled objective, so that the IGP-UCB and GP-TS variance floor does not dominate the acquisition.

We consider a horizon \(T=40\) for IGP-UCB and a slightly larger horizon \(T=60\) for GP-TS to ensure sufficient iterations for acquisition via posterior sample paths. The initial design uses \(N_{\mathrm{init}}=8\) maximin evaluations over mesh coordinates, which ensures adequate spatial coverage and well-conditioned hyperparameter updates. By contrast, \(N_{\mathrm{init}}=1\) can lead to poor conditioning. Eight points provide basic space-filling coverage and sufficient information to identify the correlation length at modest cost, negligible for graphs with \(N_h\approx 300\).

\red 
To provide practical guidance beyond this baseline setting, we further report a sensitivity study in Appendix~\ref{app:init-sensitivity}, which varies $N_{\mathrm{init}}$ across coarse, baseline, and fine discretizations of the same graph ($N_h\approx 150,300,500$).
The results show that $N_{\mathrm{init}}=8$ lies in a stable regime for the baseline discretization ($N_h\approx 300$).
More broadly, a reasonable choice of $N_{\mathrm{init}}$ should balance \emph{coverage}, so that the maximin initialization provides adequate spatial spread and avoids noticeably worse early-stage behavior, with \emph{hyperparameter-update stability}, to prevent ill-conditioned MLE updates and unstable hyperparameter estimates under online learning.
Based on the numerical evidence in Appendix~\ref{app:init-sensitivity}, we obtain the following simple rule for adjusting $N_{\mathrm{init}}$ with the discretization size.
For coarse meshes ($N_h\approx 150$), $N_{\mathrm{init}}\approx 4$ is typically sufficient;
for baseline resolutions ($N_h\approx 200$--$400$), $N_{\mathrm{init}}=8$ is a robust default;
and for finer discretizations ($N_h\gtrsim 500$), a modest increase to $N_{\mathrm{init}}\approx 8$--$10$ can improve reliability.
\nc

Averaged over $N_{\mathrm{rep}}=60$ shared initializations, the SPDE kernel consistently outperforms the Euclidean baseline (Figures~\ref{fig:benchmarks_IGP-UCB} and \ref{fig:benchmarks_TS}), achieving
lower simple regret, higher reach rates, and fewer iterations to meet the tolerance. 
Empirically, IGP--UCB tends to outperform GP--TS, plausibly because the true kernel and its hyperparameters are learned online. With a single posterior draw per round, GP--TS adapts more slowly and may get stuck early. 
The relatively poor Euclidean performance on Ackley and L\'evy, compared with Rastrigin, is consistent with kernel misspecification on the metric graph being the dominant cause of underperformance.

\begin{figure}[H]
\centering

\begin{minipage}{\textwidth}\centering
  \subcaption*{Setting 1: Ackley}
\end{minipage}

\begin{subfigure}{0.32\textwidth}
  \centering
  \includegraphics[width=\linewidth]{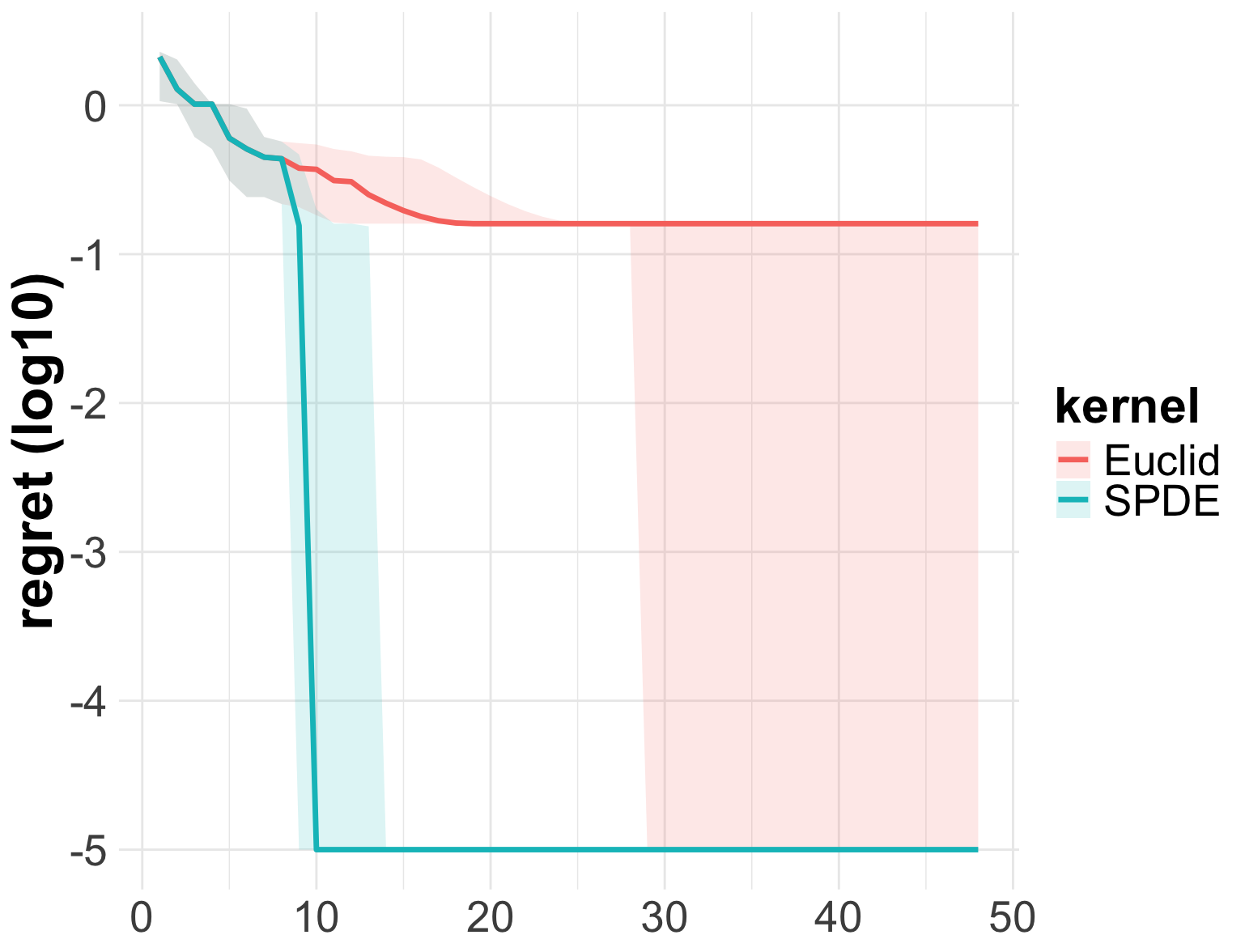}
\end{subfigure}\hfill
\begin{subfigure}{0.32\textwidth}
  \centering
  \includegraphics[width=\linewidth]{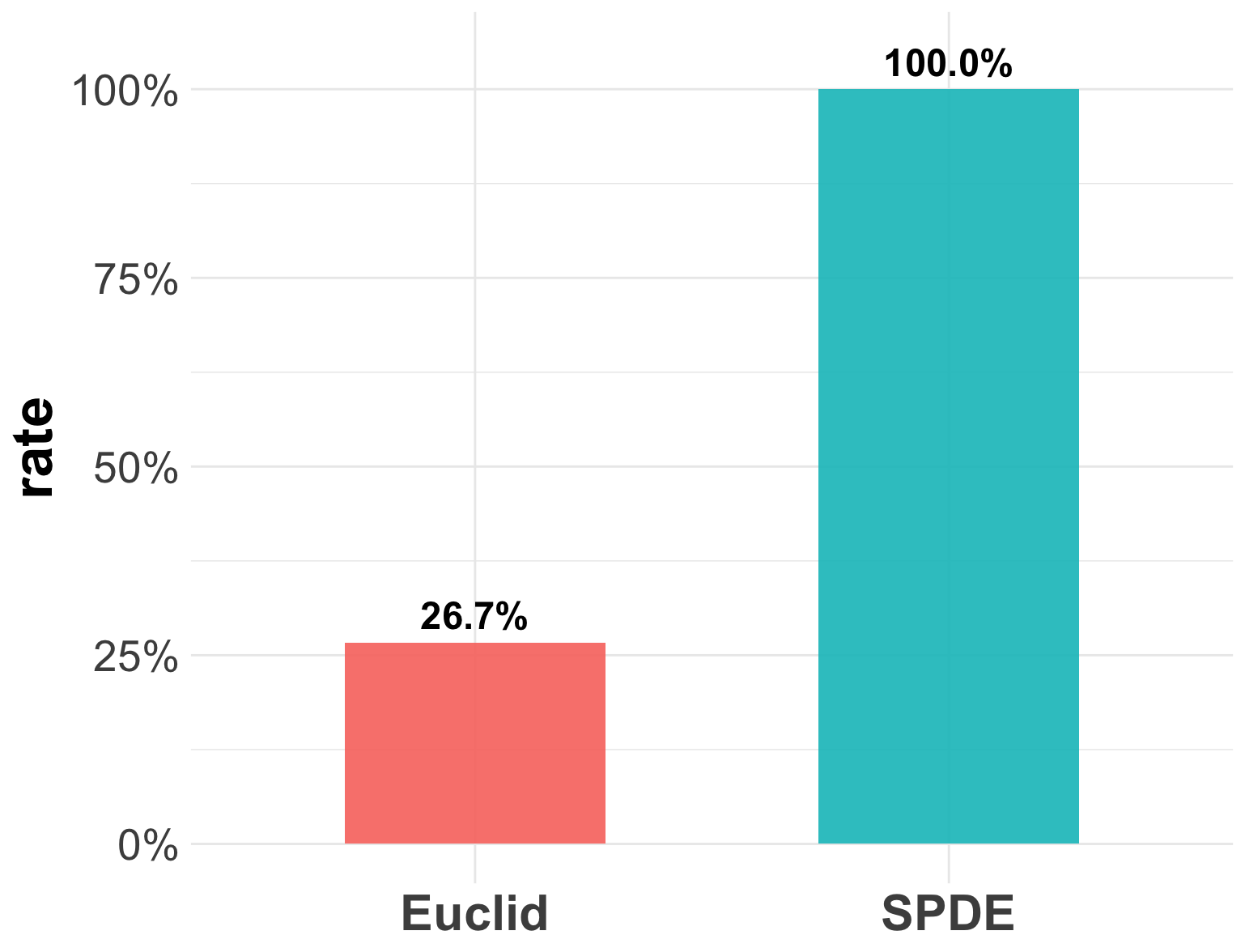}
\end{subfigure}
\begin{subfigure}{0.32\textwidth}
  \centering
  \includegraphics[width=\linewidth]{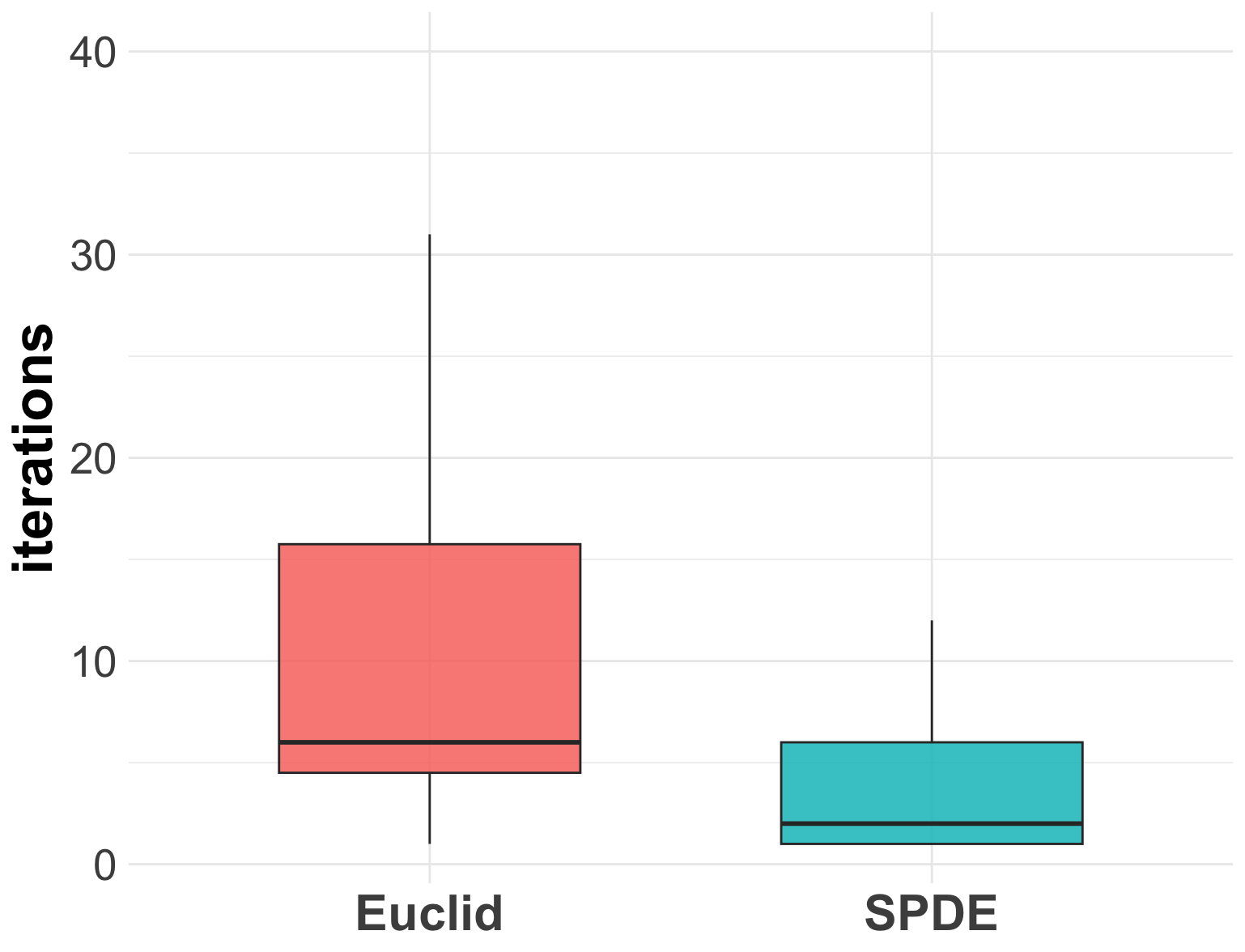}
\end{subfigure}\hfill

\begin{minipage}{\textwidth}\centering
  \subcaption*{Setting 2: Rastrigin}
\end{minipage}

\begin{subfigure}{0.32\textwidth}
  \centering
  \includegraphics[width=\linewidth]{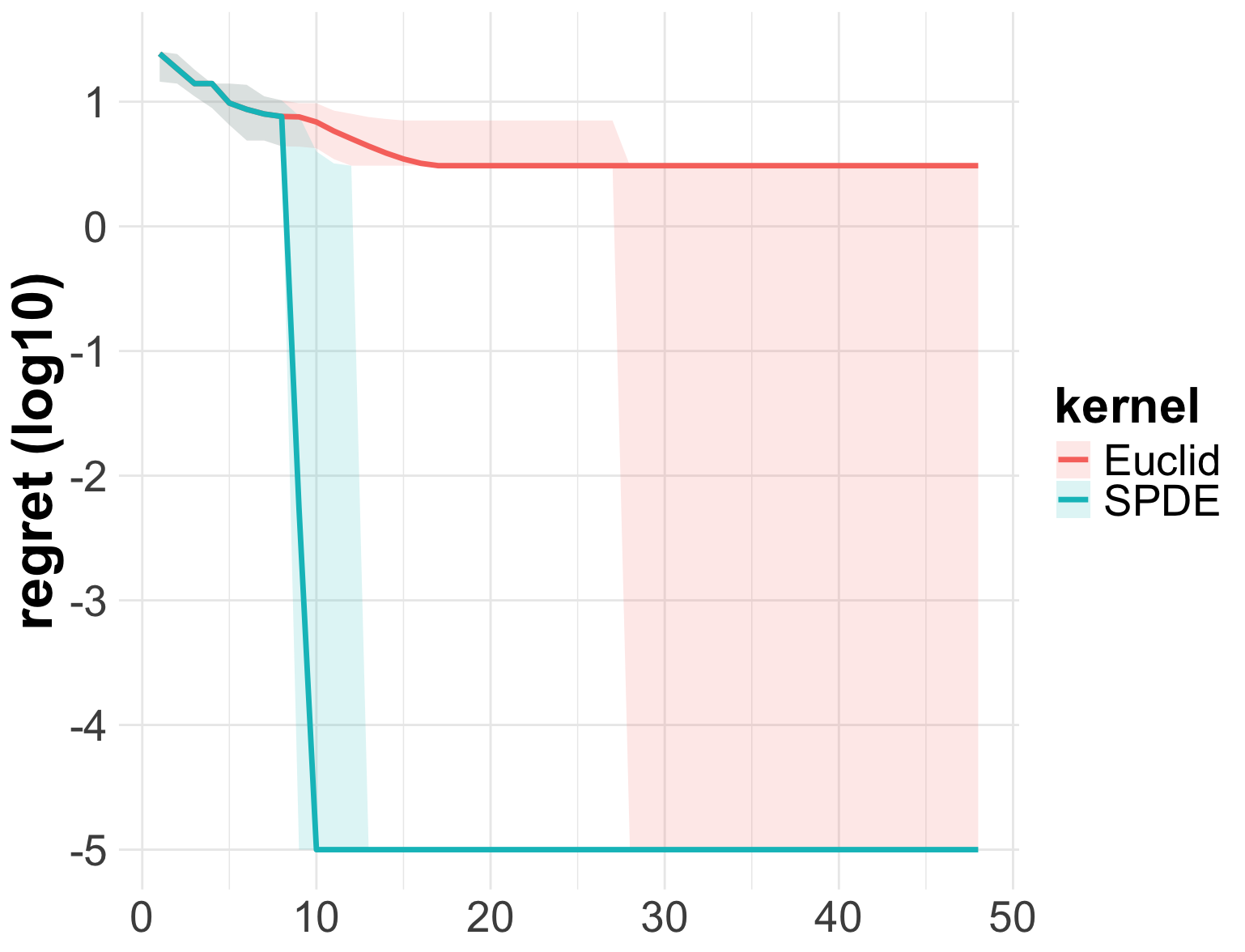}
\end{subfigure}\hfill
\begin{subfigure}{0.32\textwidth}
  \centering
  \includegraphics[width=\linewidth]{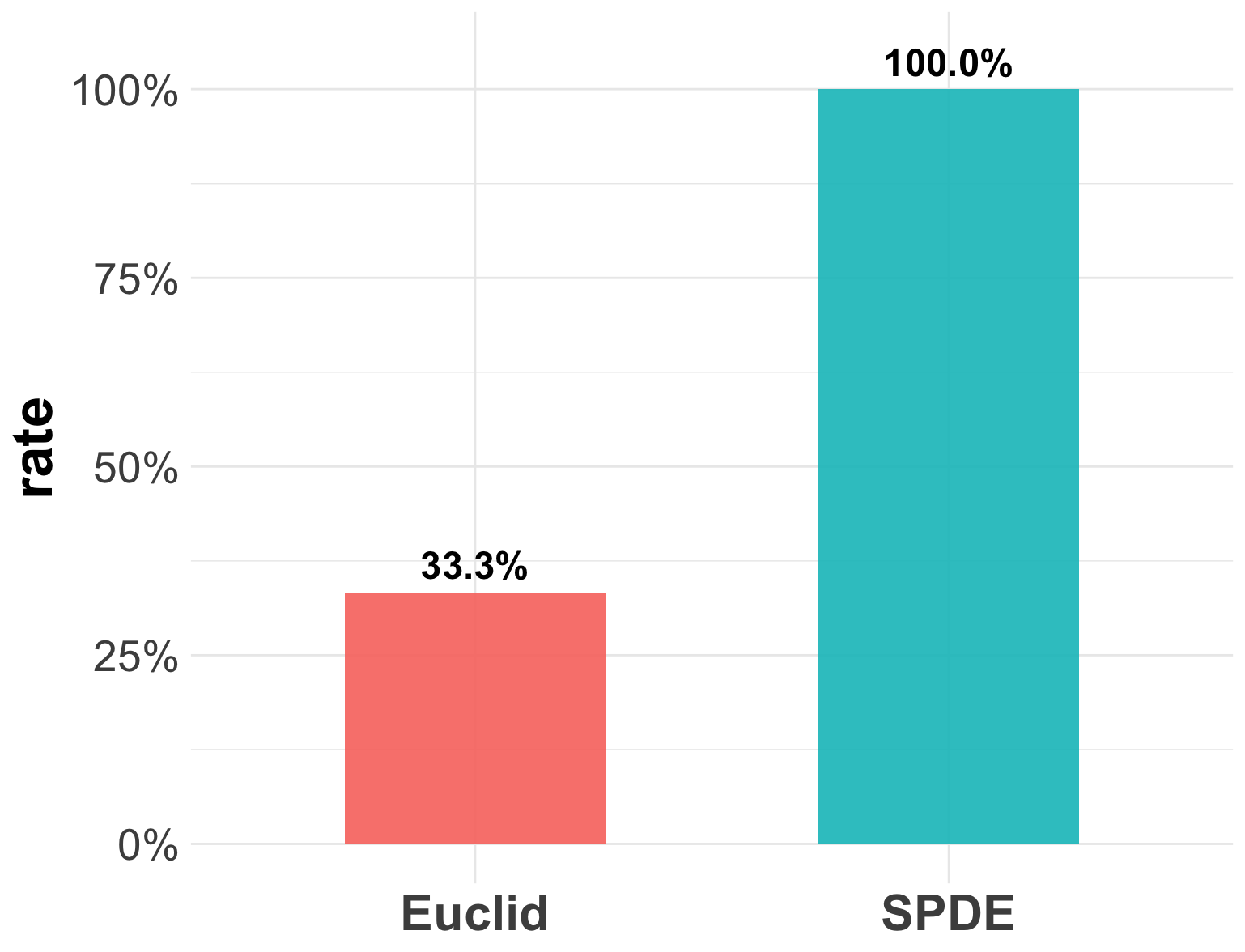}
\end{subfigure}
\begin{subfigure}{0.32\textwidth}
  \centering
  \includegraphics[width=\linewidth]{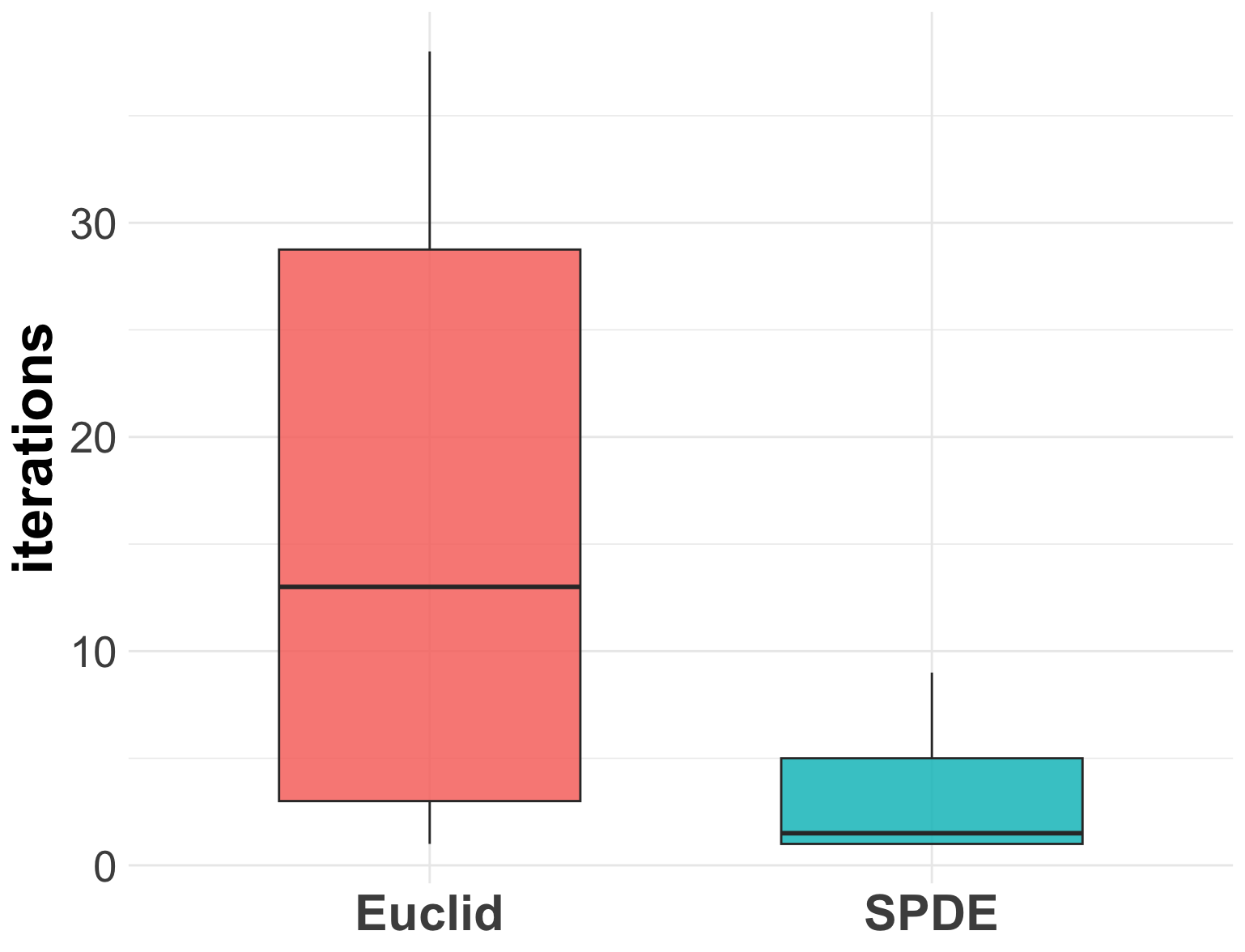}
\end{subfigure}\hfill

\vspace{0.6em}

\begin{minipage}{\textwidth}\centering
  \subcaption*{Setting 3: L\'evy}
\end{minipage}

\begin{subfigure}{0.32\textwidth}
  \centering
  \includegraphics[width=\linewidth]{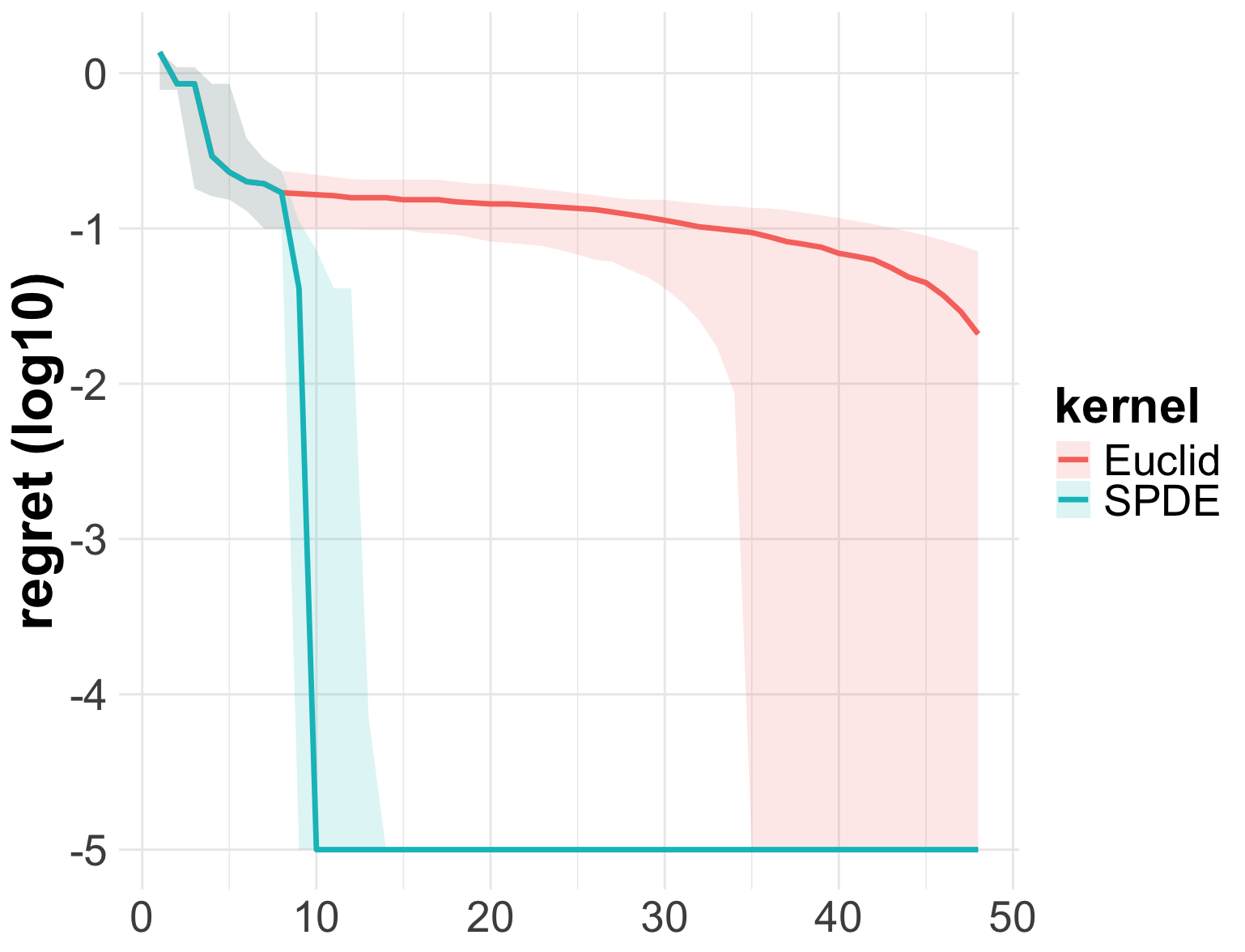}
  \caption{Simple regret}
\end{subfigure}\hfill
\begin{subfigure}{0.32\textwidth}
  \centering
  \includegraphics[width=\linewidth]{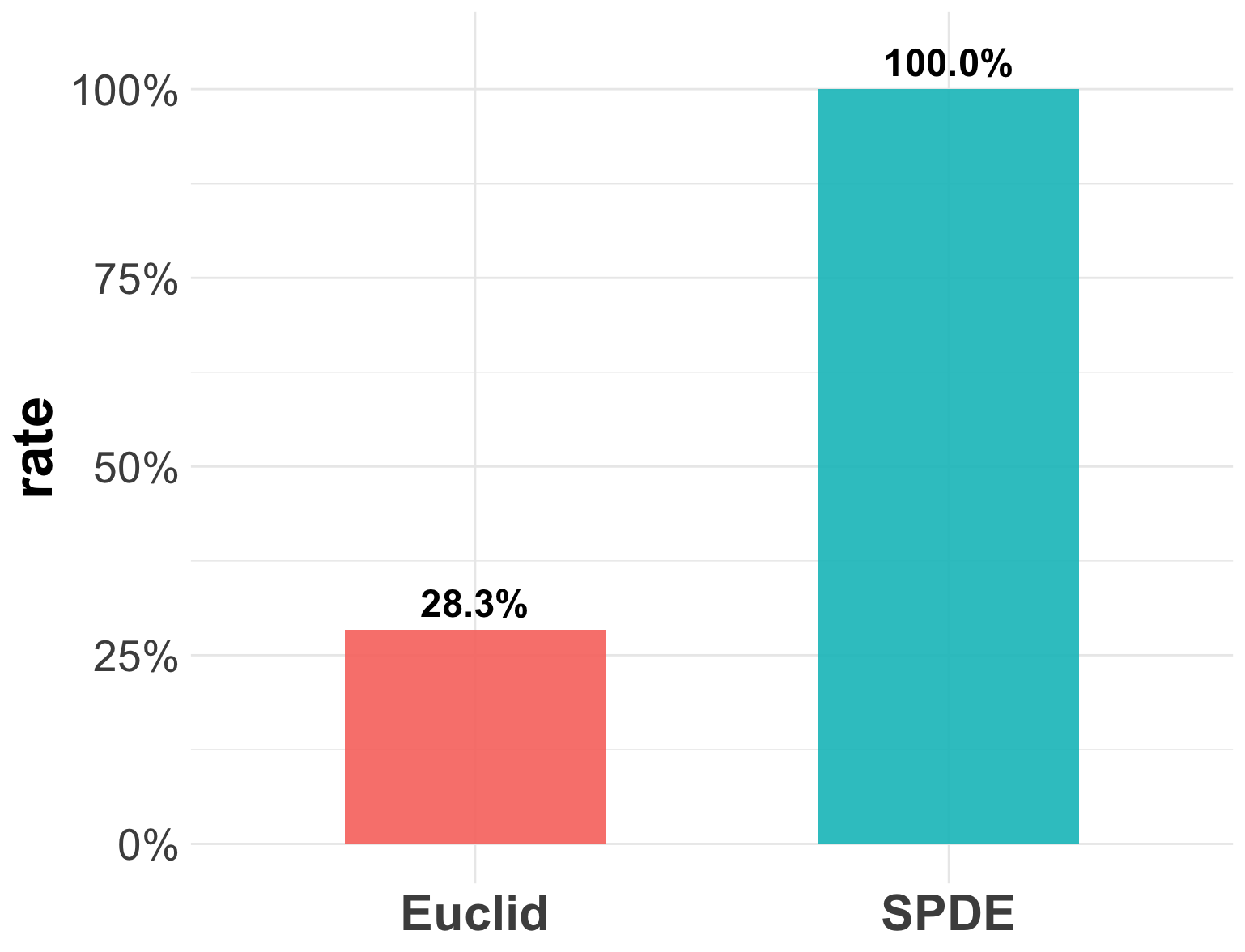}
  \caption{Reach rate}
\end{subfigure}
\begin{subfigure}{0.32\textwidth}
  \centering
  \includegraphics[width=\linewidth]{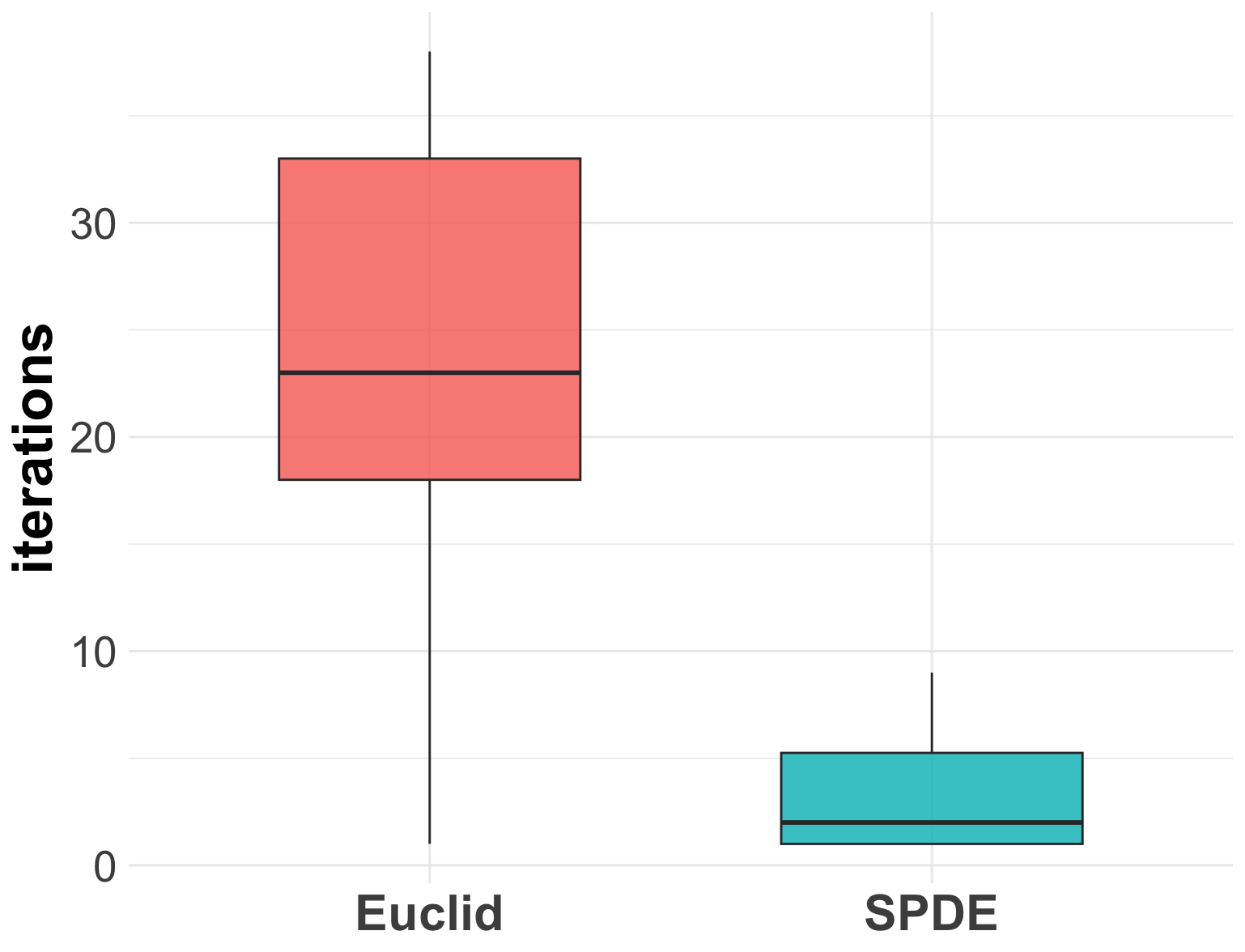}
  \caption{Iterations to $\mathsf{Tol}$}
\end{subfigure}\hfill

\caption{BO for benchmark functions with IGP-UCB. 
Columns show: (a) Simple regret across different initializations, with median in solid line and the shaded region representing the central $50\%$ band; (b) Reach rate; and (c) Iterations to $\mathsf{Tol}.$ The
rows correspond to the Ackley, Rastrigin, and L\'evy benchmark function on the open rectangle metric graph.}
\label{fig:benchmarks_IGP-UCB}
\end{figure}

\begin{figure}[H]
\centering

\begin{minipage}{\textwidth}\centering
  \subcaption*{Setting 1: Ackley}
\end{minipage}

\begin{subfigure}{0.32\textwidth}
  \centering
  \includegraphics[width=\linewidth]{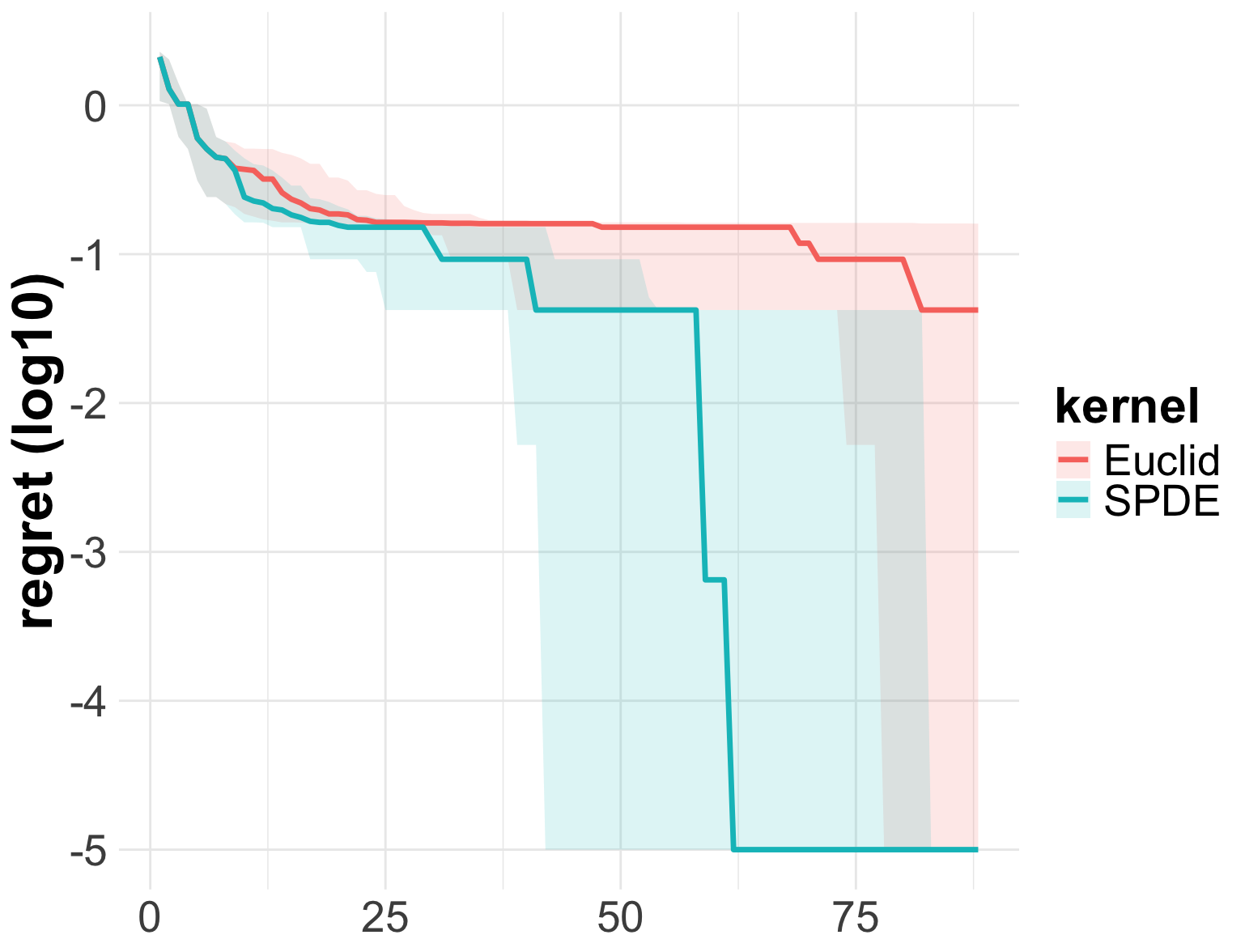}
\end{subfigure}\hfill
\begin{subfigure}{0.32\textwidth}
  \centering
  \includegraphics[width=\linewidth]{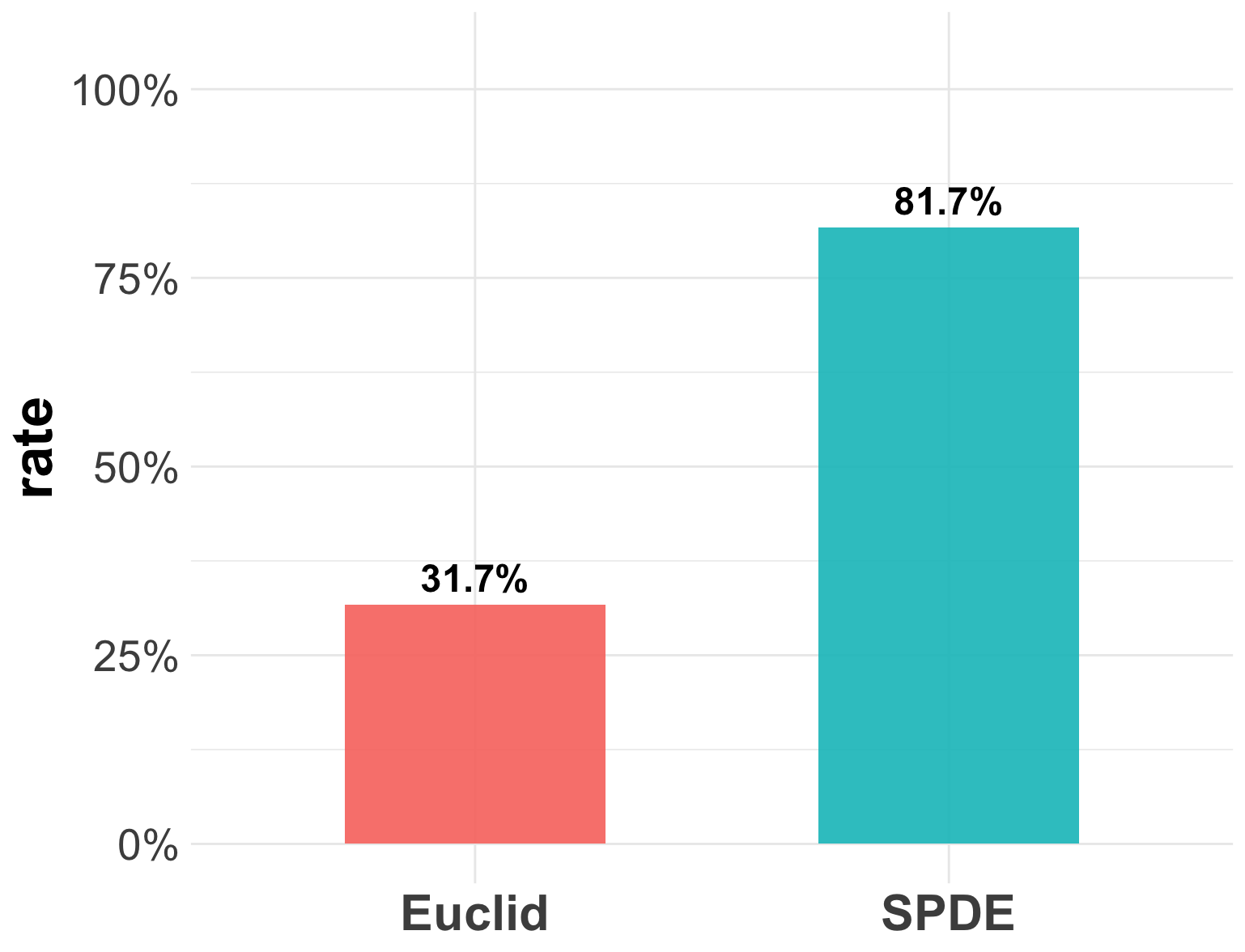}
\end{subfigure}
\begin{subfigure}{0.32\textwidth}
  \centering
  \includegraphics[width=\linewidth]{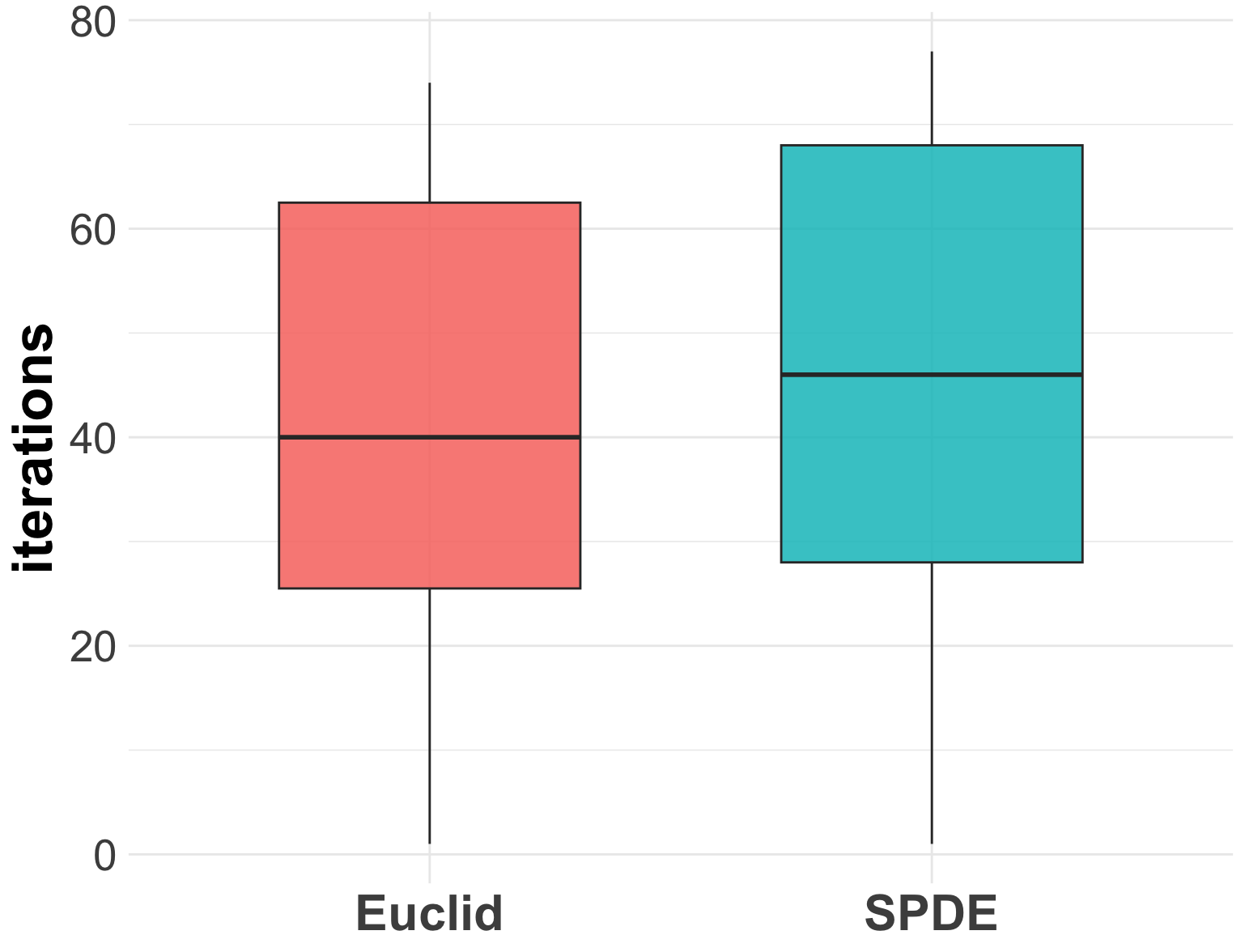}
\end{subfigure}\hfill

\begin{minipage}{\textwidth}\centering
  \subcaption*{Setting 2: Rastrigin}
\end{minipage}

\begin{subfigure}{0.32\textwidth}
  \centering
  \includegraphics[width=\linewidth]{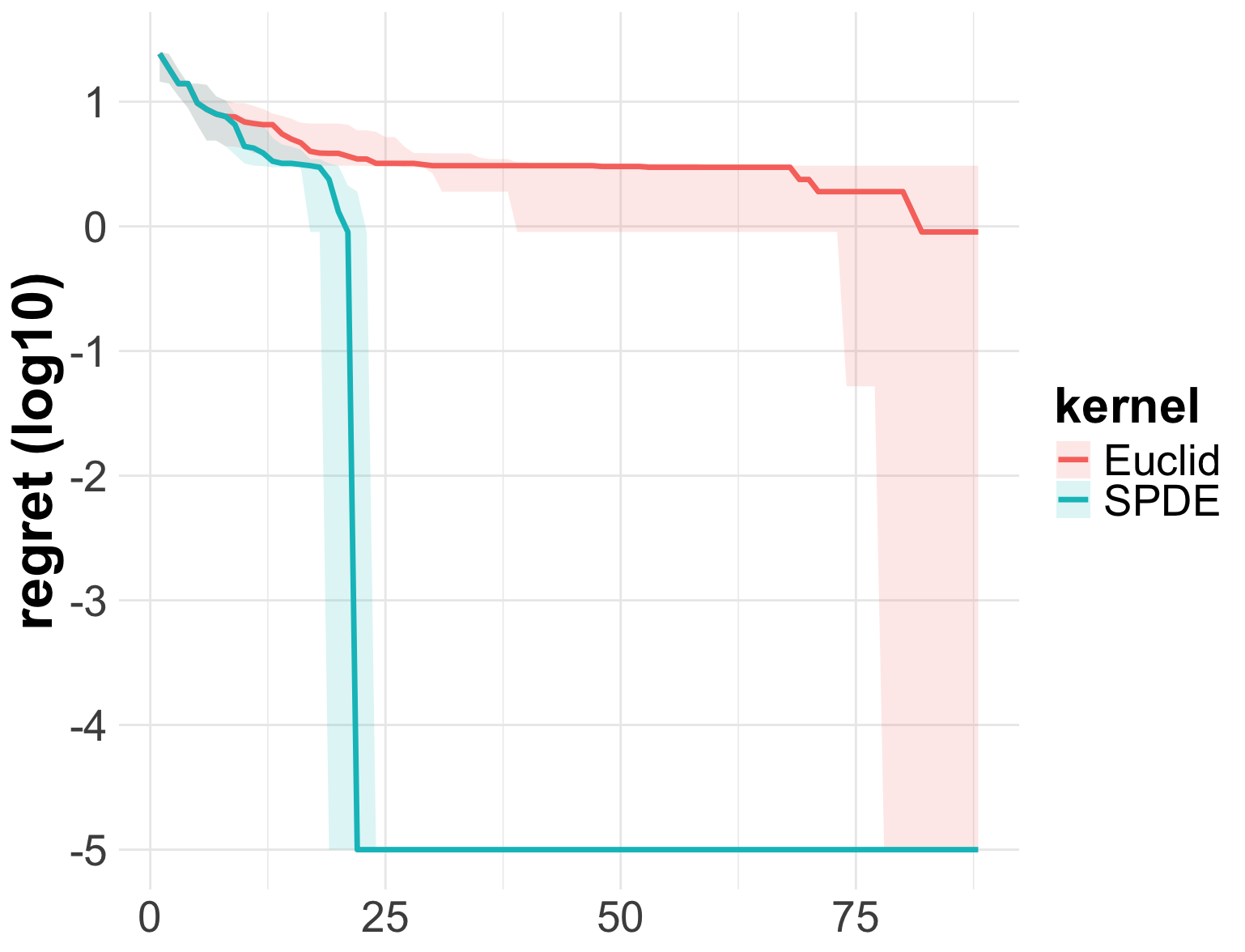}
\end{subfigure}\hfill
\begin{subfigure}{0.32\textwidth}
  \centering
  \includegraphics[width=\linewidth]{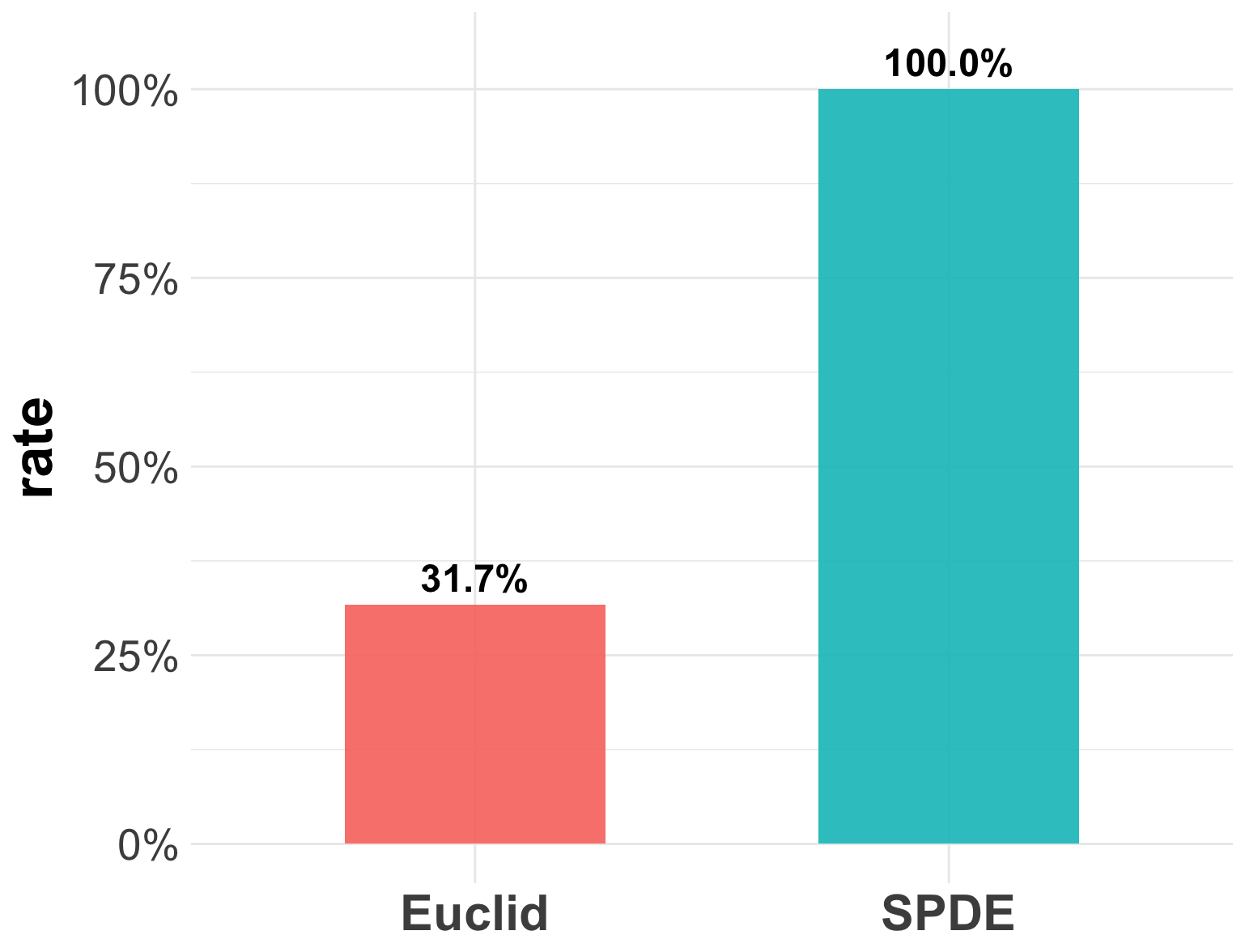}
\end{subfigure}
\begin{subfigure}{0.32\textwidth}
  \centering
  \includegraphics[width=\linewidth]{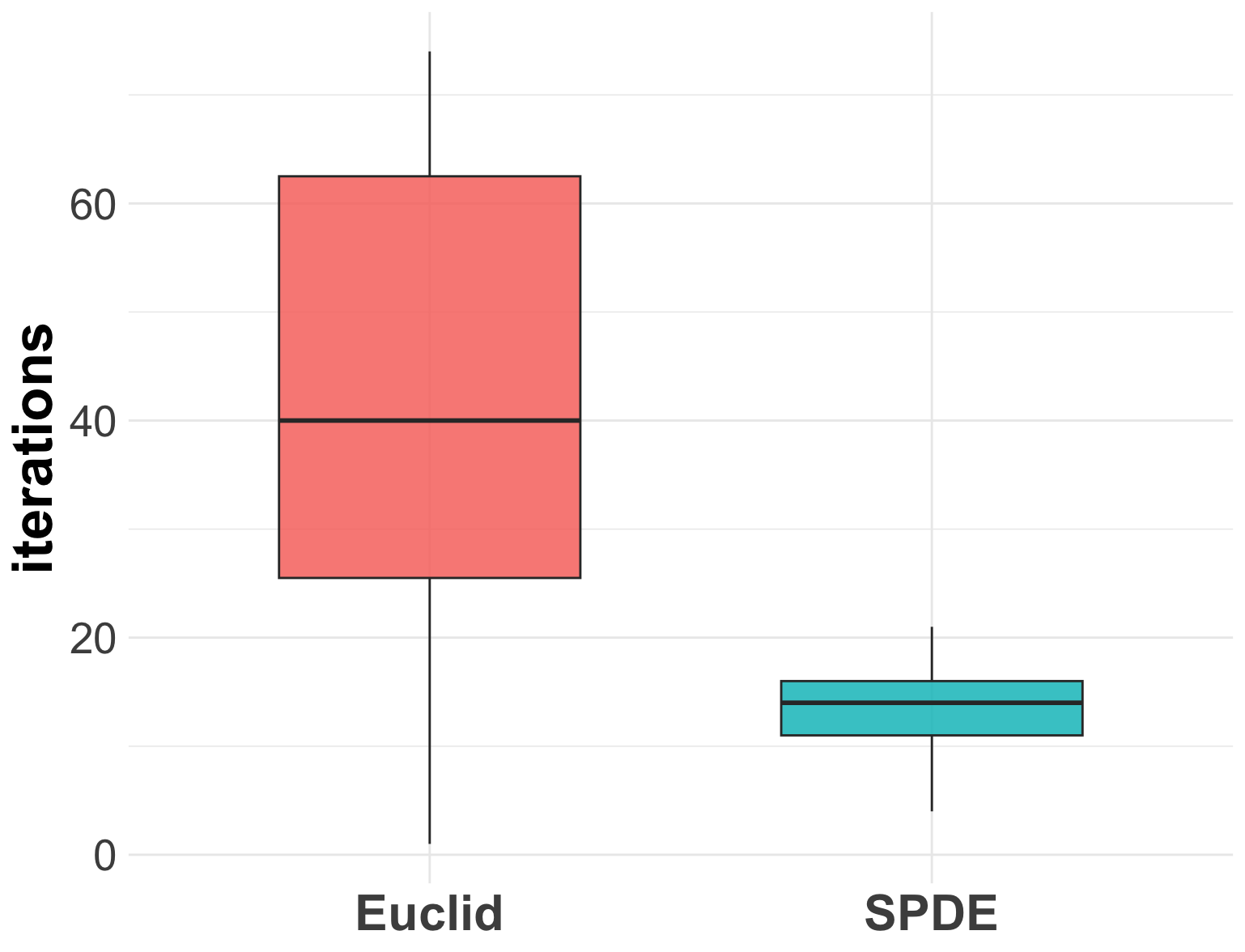}
\end{subfigure}\hfill

\vspace{0.6em}

\begin{minipage}{\textwidth}\centering
  \subcaption*{Setting 3: L\'evy}
\end{minipage}

\begin{subfigure}{0.32\textwidth}
  \centering
  \includegraphics[width=\linewidth]{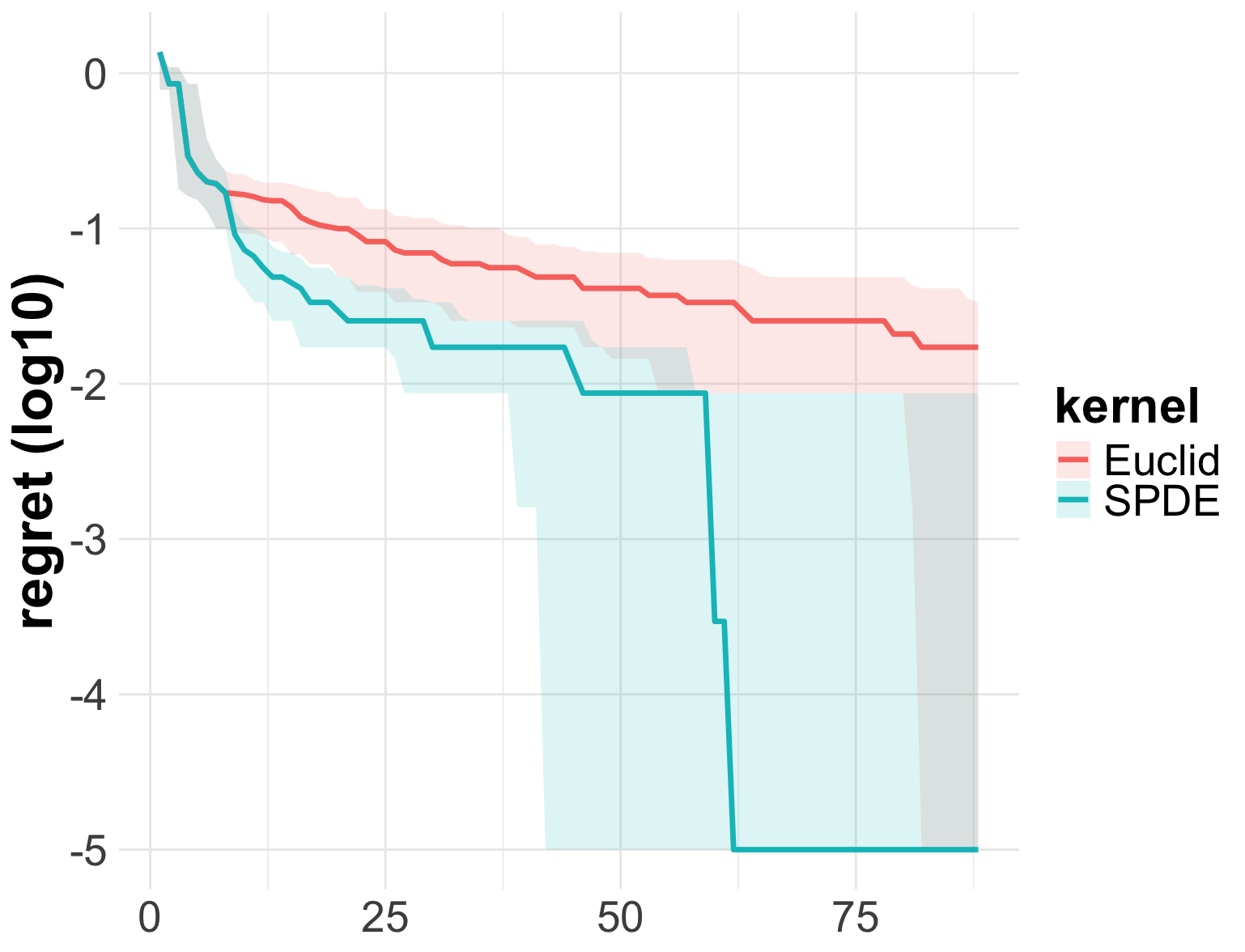}
  \caption{Simple regret}
\end{subfigure}\hfill
\begin{subfigure}{0.32\textwidth}
  \centering
  \includegraphics[width=\linewidth]{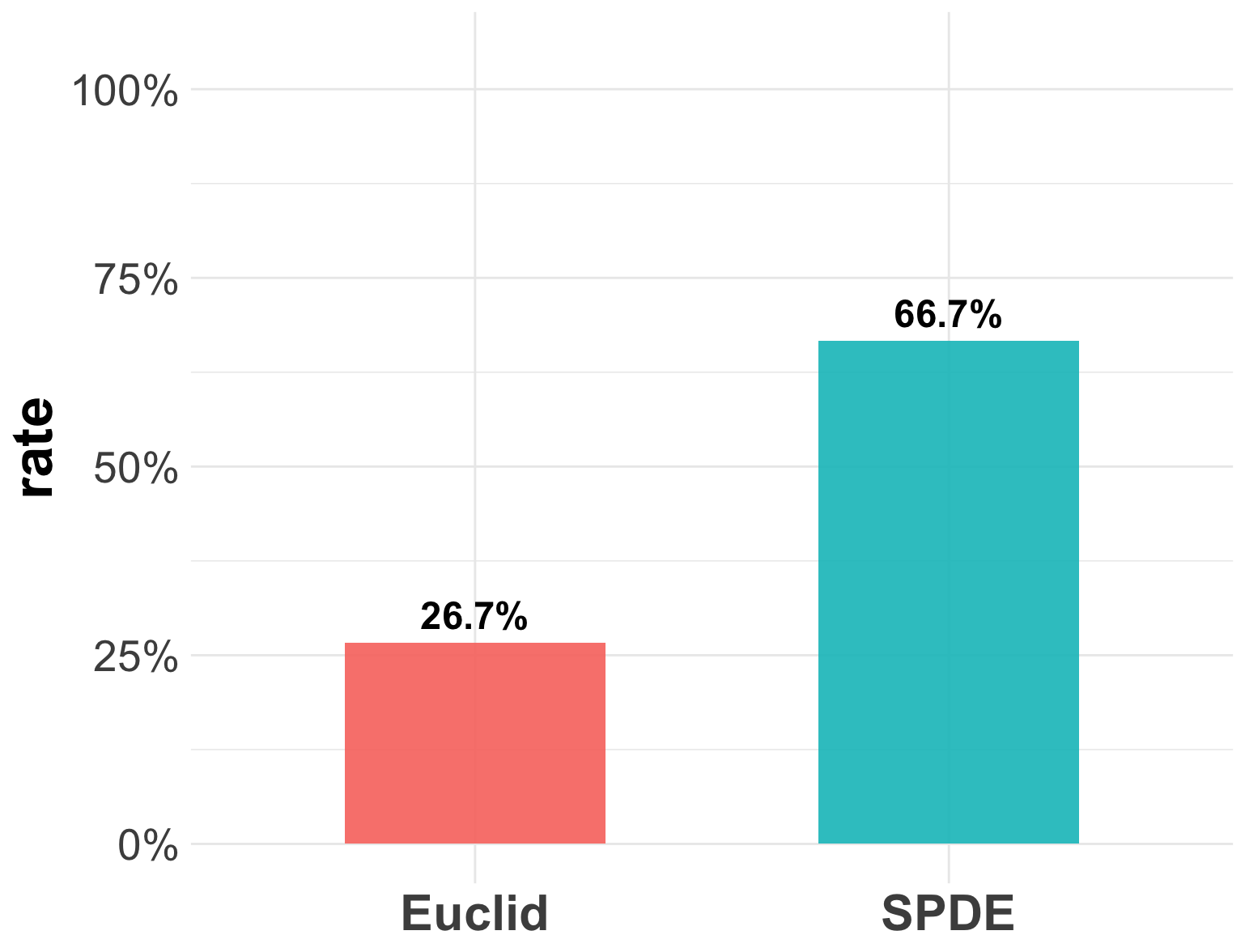}
  \caption{Reach rate}
\end{subfigure}
\begin{subfigure}{0.32\textwidth}
  \centering
  \includegraphics[width=\linewidth]{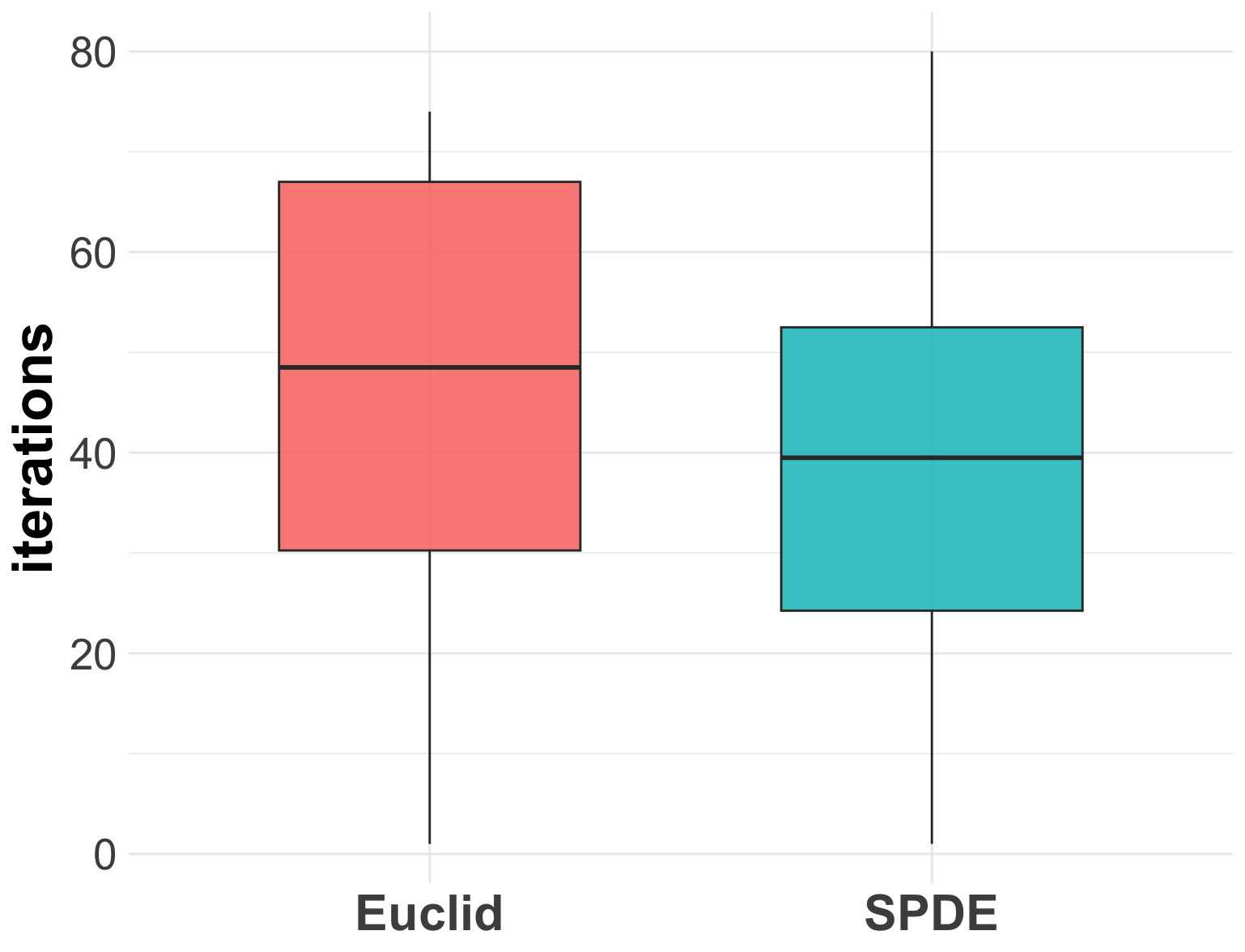}
  \caption{Iterations to $\mathsf{Tol}$}
\end{subfigure}\hfill

\caption{BO for benchmark functions with GP-TS. 
Layout matches Figure~\ref{fig:benchmarks_IGP-UCB}.}
\label{fig:benchmarks_TS}
\end{figure}

\subsection{MAP Estimation for Bayesian Inversion}\label{ssec:MAP}
\subsubsection{Problem Setting}
In this section, we apply Bayesian optimization for \emph{maximum a posteriori} estimation within a Bayesian formulation of a point–source identification inverse problem. Let $\Gamma$ be a compact metric graph and consider the elliptic equation
\begin{equation}
    \label{eq:solution p}
      \bigl( \chi^2 \nc - \Delta_{\Gamma} \bigr) p \;=\; g,
\end{equation}
where $\chi>0,$ $\Delta_\Gamma$ denotes the Kirchhoff Laplacian, and $g$ is a source function localized around an unknown point $x^\dagger \in \Gamma.$ We are interested in the inverse problem of determining the source location $x^\dagger$ from partial and noisy measurements of the PDE solution.
For convenience, we suppose that after discretization on a FEM mesh with node set
$\femnodes = \{x_i \}_{i=1}^{N_h},$ the right-hand side $g$ in \eqref{eq:solution p} is a tent basis function $e_{h,j}$ as defined in Section~\ref{ssec: fem space on Gamma},  centered at $x^\dagger = x_j \in \femnodes$ for some unknown $x_j \in \femnodes$ that we seek to determine from data. 

Let $L_h =  \chi^2 \, C + G$ denote the FEM discretization of the operator $\chi^2 \nc - \Delta_{\Gamma}.$ We define the forward map
\begin{align*}
   \mathcal{F}: \femnodes &\to \mathbb{R}^{N_h} \\
  x_i &\mapsto p_{h,i} \;:=\; L_h^{-1} g_{h,i} \, , 
\end{align*}
where, for $1 \le i \le N_h,$ $g_{h,i}\in\mathbb{R}^{N_h}$
denotes the FEM coefficient vectors of the tent function $e_{h,i}$ centered at $x_i \in \femnodes$, and $p_{h,i}$ denotes the coefficients of the corresponding approximate PDE solution. We also define the observation map 
\begin{align*}
    O : \R^{N_h} &\to \mathbb{R}^{N_{\mathrm{obs}}} \\
    v = (v_i)_{i=1}^{N_h} &\mapsto (v_i)_{i \in \mathcal{I}_\mathrm{obs}} \, , 
\end{align*}
where $\mathcal{I}_\mathrm{obs} \subseteq \{1, \ldots, N_h \}$ with $|\mathcal{I}_\mathrm{obs}| = N_{\mathrm{obs}} \le N_h$ denotes a subset of mesh locations at which measurements are taken. Finally, we define the forward model $\mathcal{G} := O \circ \mathcal{F}$ by composing the forward and observation maps. Assume that source location $x^\dagger \in \femnodes$ and observed data $\mathscr{D} \in \mathbb{R}^{N_{\mathrm{obs}}}$ are related by
\begin{equation}\label{eq:IP}
     \mathscr{D} \;=\; \mathcal{G}(x^\dagger) + \eta, 
  \qquad \eta \sim \mathcal{N}\!\big(0, \noisestdip^{2} I_{N_{\mathrm{obs}}}\big).
\end{equation}
Placing a prior $\pi(x)$ on the source location $x^\dagger$, we derive the posterior distribution:
\begin{align*}
    p(x \mid \mathscr{D}) &\propto \, p(\mathscr{D} \mid x) \pi(x) \\
    & \propto \exp\!\Big(-\tfrac{1}{2\noisestdip^{2}} \,\big\| \mathscr{D} - \mathcal{G}(x) \big\|_2^2\Big) \pi(x),
\end{align*}
where the expression for the Gaussian likelihood $p(\mathscr{D} \mid x) = \mathcal{N}(\mathcal{G}(x),\noisestdip^{2} I_{N_{\mathrm{obs}}}) $  follows from \eqref{eq:IP}. We adopt as prior a discretized version of the uniform distribution on $\Gamma$, defined via the mass matrix $C$: $\pi(x_i)\propto (C\mathbf{1})_i$ and $\sum_{i=1}^{N_h}\pi(x_i)=1$. Unlike choosing $\pi(x)=1/N_h$, this choice is mesh-independent and converges to the continuous uniform distribution on \(\Gamma\) as $h\to 0$. The \emph{maximum a posteriori} estimate $x^*$ for the unknown location $x^\dagger$ is given by the posterior mode \cite{sanzstuarttaeb}, or, equivalently, by maximizing the log-posterior given, up to an additive constant, by  
\begin{align}\label{eq:LP}
  \mathrm{LP}(x) \;=\; 
  -\tfrac{1}{2\noisestdip^{2}} \,\big\| \mathscr{D} - \mathcal{G}(x) \big\|_2^2 
  \;+\; \log \pi(x).
\end{align}
In the following subsection we explore the performance of IGP-UCB and GP-TS with FEM approximation for optimizing  $\mathrm{LP}(x)$ over $\femnodes.$   

\begin{figure}[H]
\begin{subfigure}{0.5\textwidth}
 \hfill \includegraphics[width=0.85\linewidth]{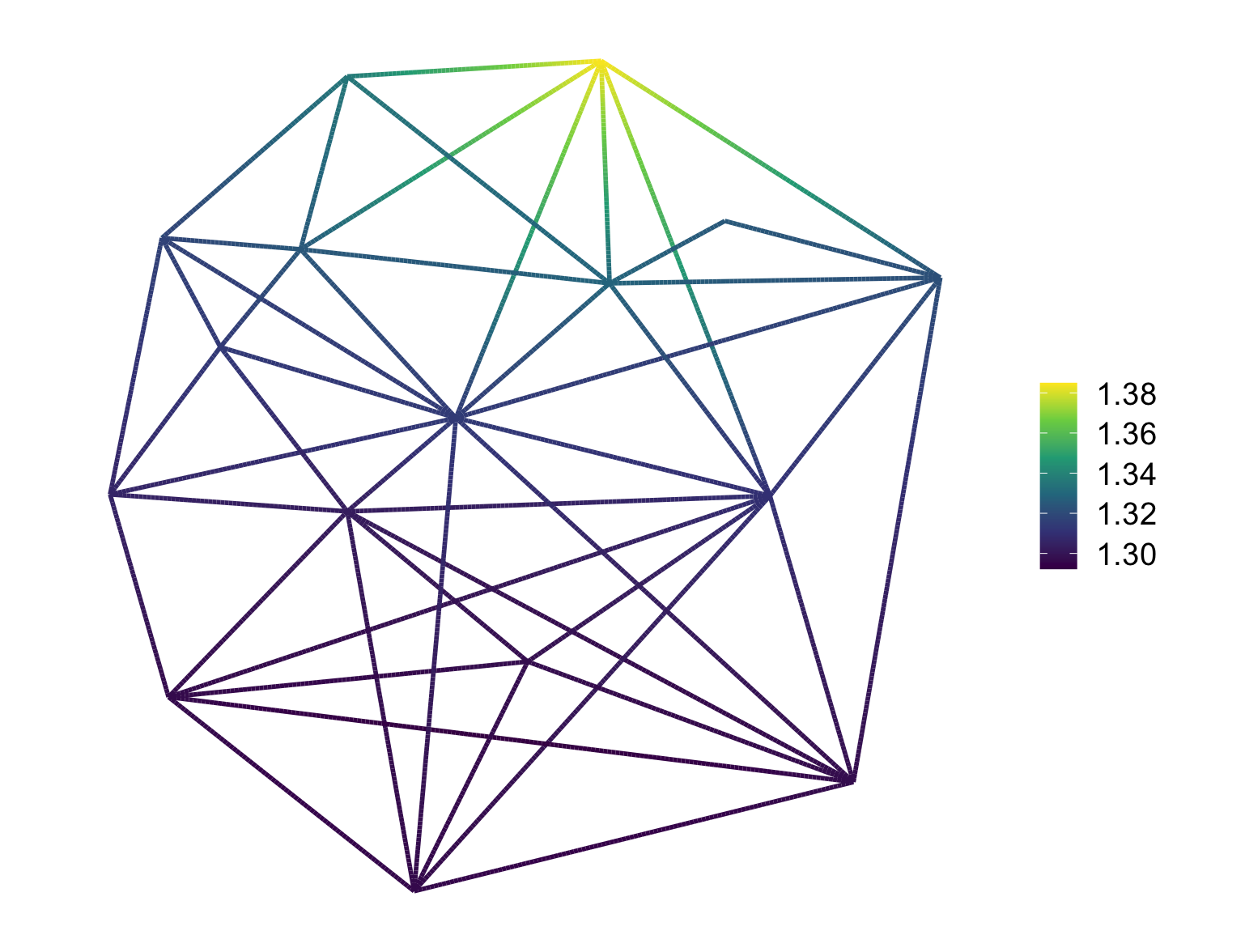}
  \hfill \caption{Solution $p_h$ computed with FEM}
\end{subfigure}
\begin{subfigure}{0.5\textwidth}
  \hfill \includegraphics[width=0.85\linewidth]{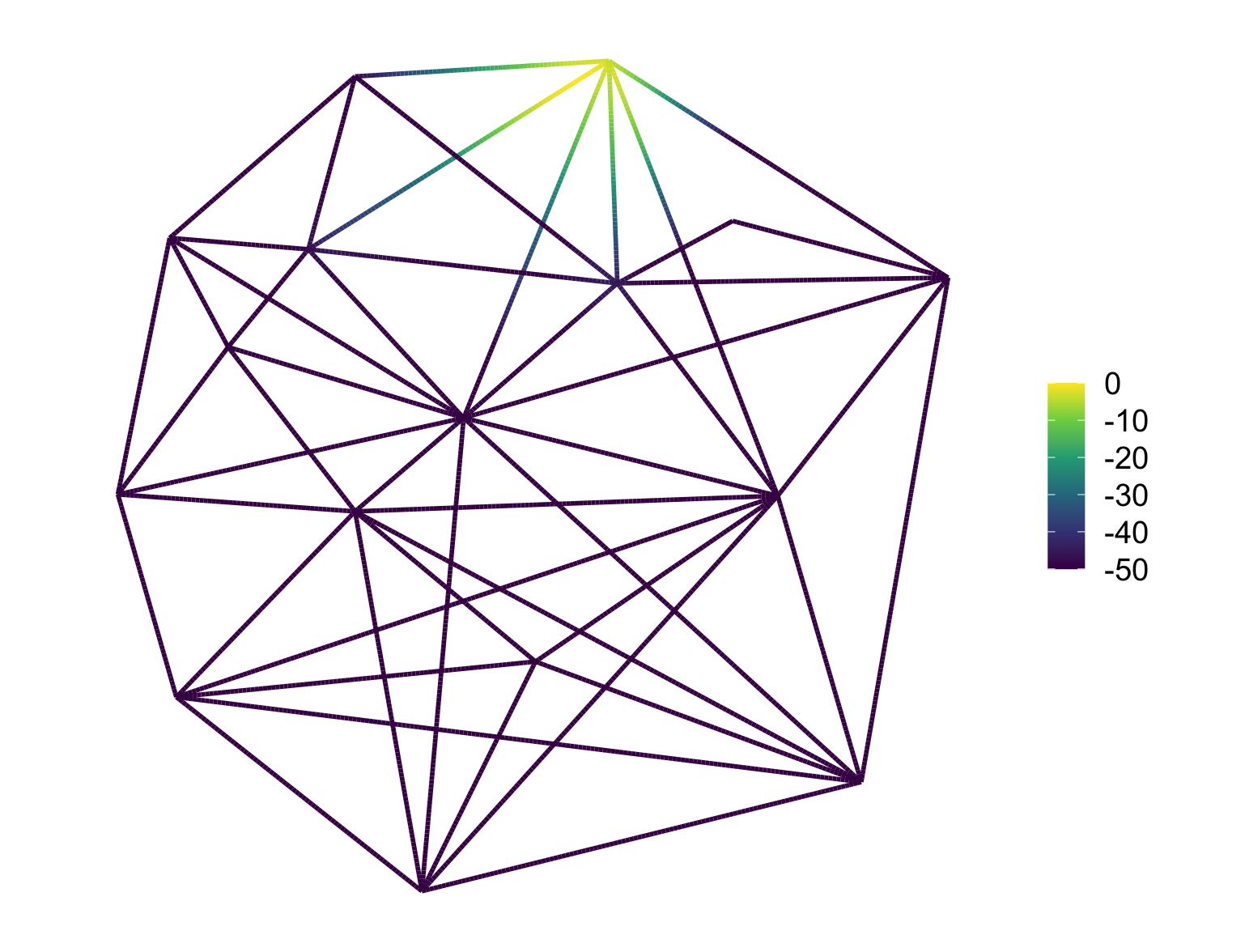}
   \hfill \caption{Log-posterior $\mathrm{LP}(x)$ for source location}
\end{subfigure}

\caption{ FEM solution \(p_h\) (left) and  log-posterior \(\mathrm{LP}(x)=\log p(x\mid \mathscr{D})\) (right).}
\label{fig:ipbo-solution_posterior_tele}
\end{figure}

\subsubsection{Numerical Results}
We work on the telecommunication network in New York from \cite{SNDlib10,OrlowskiPioroTomaszewskiWessaely2010} shown in
Figure~\ref{fig:metric graph example}(a). We take $\chi_0=0.2$ in \eqref{eq:solution p}. The FEM solution \(p_h\) is computed using the R package \textsf{MetricGraph}~\cite{MetricGraph},
setting the mesh size to \(h=0.25\), keeping \(N_h\) on the order of \(10^2\). We choose the set $\mathcal{I}_{\textrm{obs}}$ of observation locations to have size $N_h/2$ and to include the points at which the true solution $p_h$ attains the largest value. Guided by the ranges of $p_h$ and the log-posterior $\mathrm{LP}(x)$ in Figure~\ref{fig:ipbo-solution_posterior_tele}, we take \(B=1\) to be commensurate with the scale of the log-posterior, and set $\noisestdip=0.1$ to match the scale of $p_h$, and take $R=0$, since here $\mathrm{LP}$ is available in closed form from calculation rather than measurement. 
To guarantee numerical stability in this setting, we add a small nugget term (e.g., $10^{-2}$) when updating the GP parameters online. For simplicity, we still denote it by $\noisestd$ (see Algorithm~\ref{alg: benchmarks}).
We then run Algorithm~\ref{alg: benchmarks} with horizon \(T=40\) for both IGP\textendash UCB and GP\textendash TS, after \(N_{\rm init}=8\) initial design points. To assess average performance and reduce sensitivity to initialization, we repeat the loop over \(N_{\mathrm{rep}}=50\) randomly chosen initial designs.

The results in  Figure~\ref{fig:ipbo-ts-tele} show that Bayesian optimization with Euclidean kernel performs poorly: it rarely recovers the source within 40 iterations and exhibits a markedly lower reach rate. By contrast, algorithms with SPDE kernels succeed reliably, with the reach rate approaching \(1\) within 40 steps and the average iterations to \(\mathsf{Tol}\) falling in the \(20\)–\(30\) range. A plausible explanation is that the log-posterior is effectively unimodal and sharply concentrated in a small neighborhood of the true source \(x^\dagger\); a kernel that respects graph geometry (via shortest-path distance) localizes efficiently. In contrast, the Euclidean kernel shortcuts across gaps and parallel edges. Points that are close in Euclidean space but far in shortest-path distance are incorrectly treated as highly correlated, creating a covariance–objective mismatch and degrading optimization performance.

\begin{figure}[H]
\centering

\begin{minipage}{\textwidth}\centering
  \subcaption*{\textbf{IGP\textendash UCB}}
\end{minipage}

\begin{subfigure}{0.32\textwidth}
  \centering
  \includegraphics[width=\linewidth]{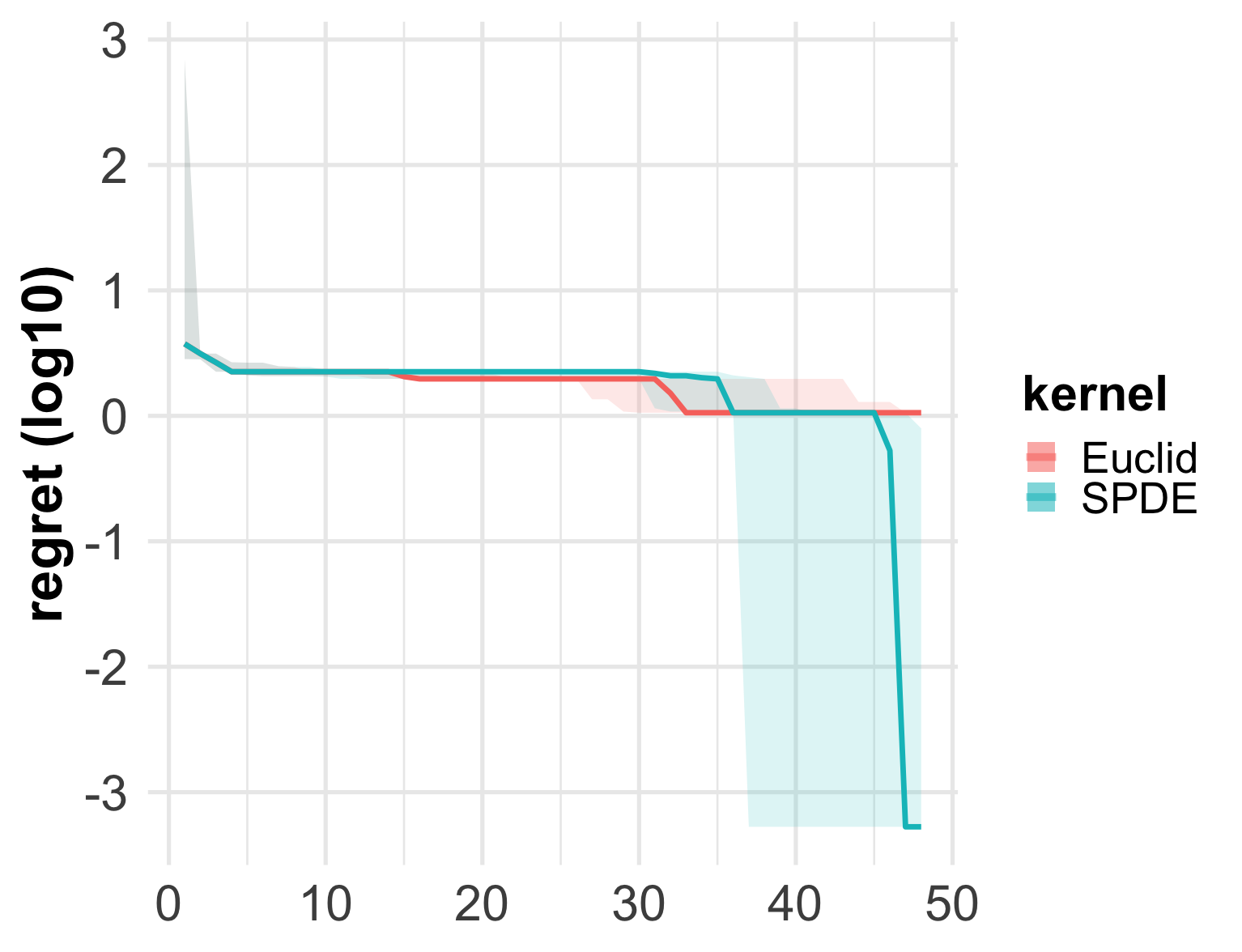}
\end{subfigure}\hfill
\begin{subfigure}{0.32\textwidth}
  \centering
  \includegraphics[width=\linewidth]{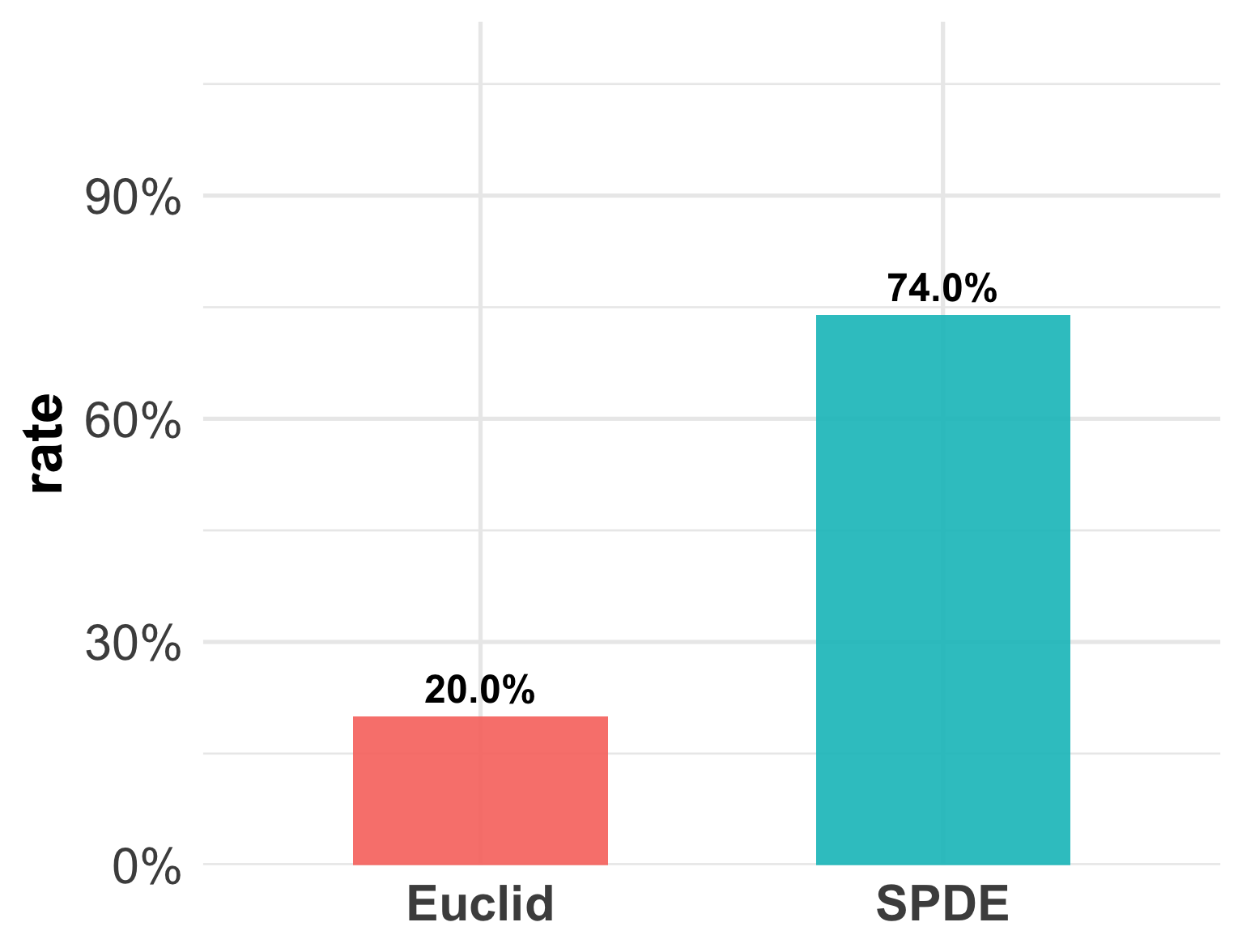}
\end{subfigure}
\begin{subfigure}{0.32\textwidth}
  \centering
  \includegraphics[width=\linewidth]{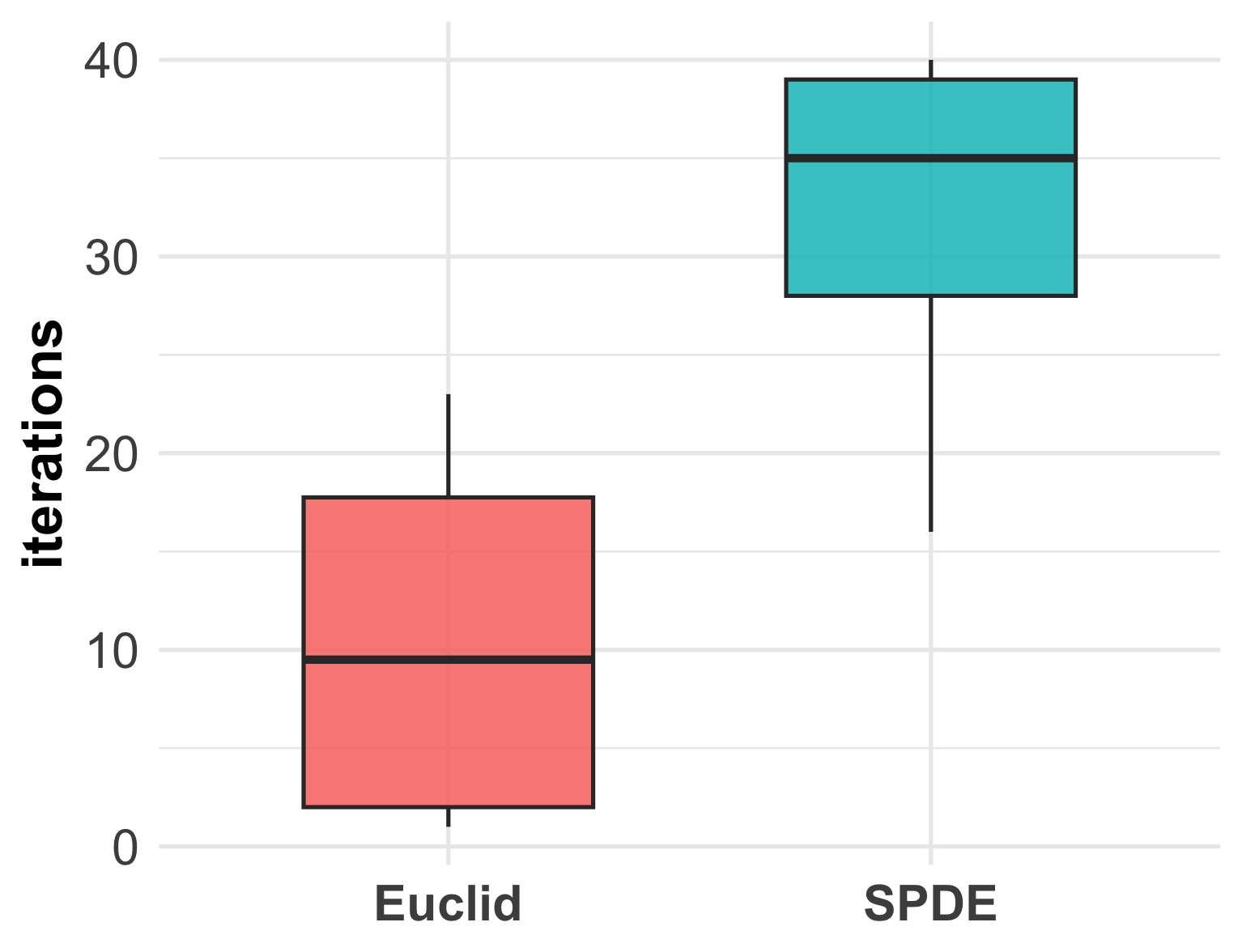}
\end{subfigure}\hfill

\vspace{0.6em}

\begin{minipage}{\textwidth}\centering
  \subcaption*{\textbf{GP\textendash TS}}
\end{minipage}

\begin{subfigure}{0.32\textwidth}
  \centering
  \includegraphics[width=\linewidth]{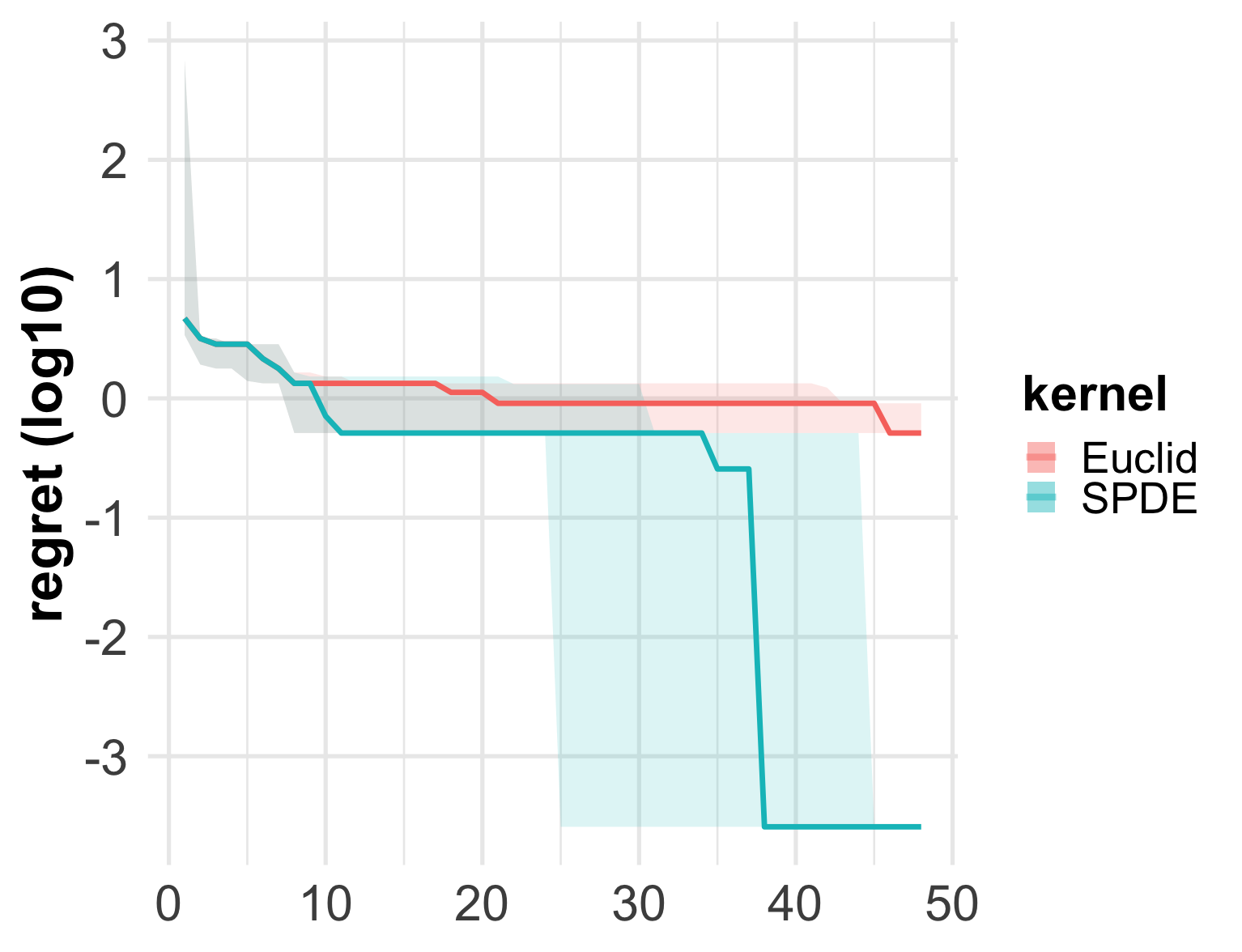}
  \caption{Simple regret}
\end{subfigure}\hfill
\begin{subfigure}{0.32\textwidth}
  \centering
  \includegraphics[width=\linewidth]{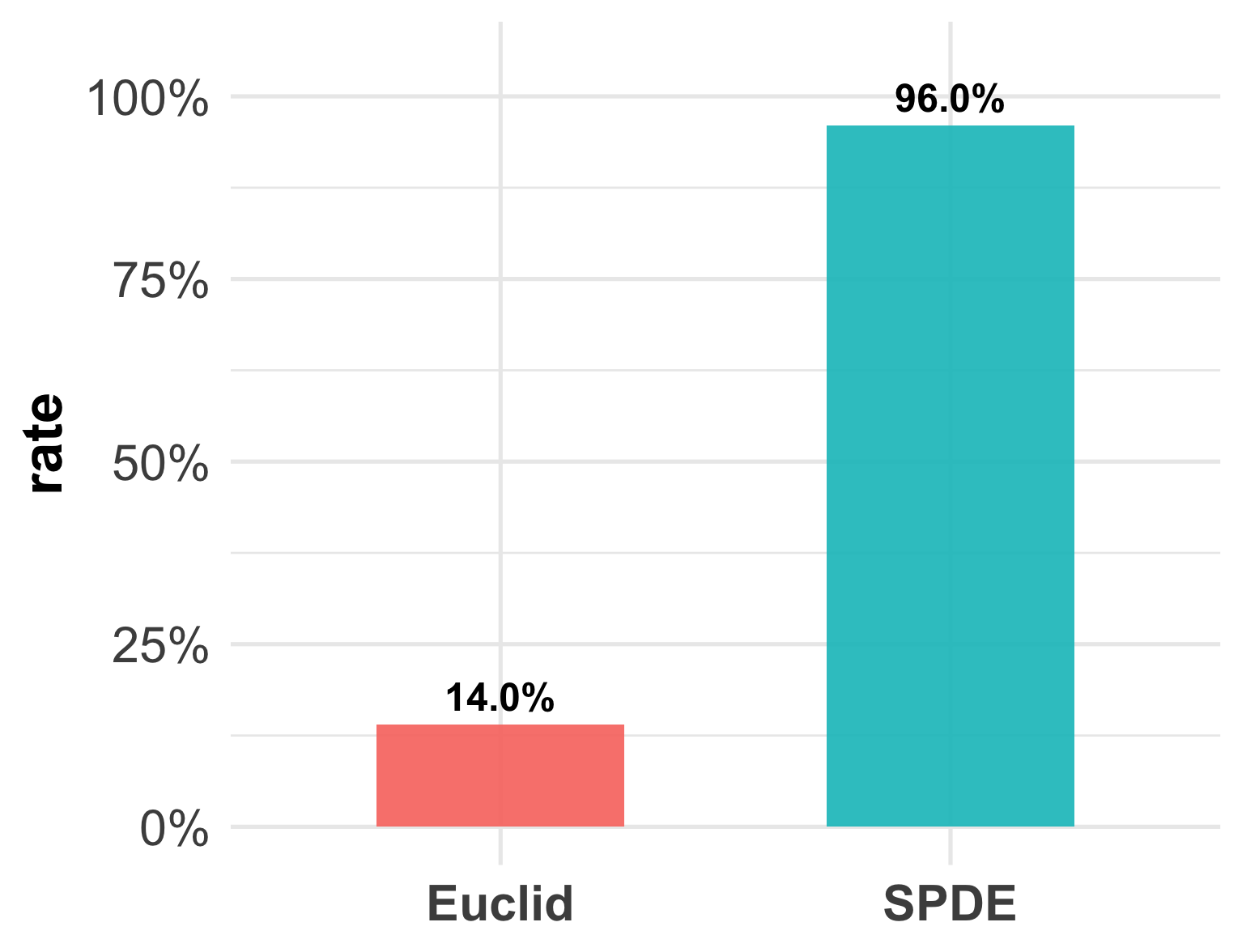}
  \caption{Reach rate}
\end{subfigure}
\begin{subfigure}{0.32\textwidth}
  \centering
  \includegraphics[width=\linewidth]{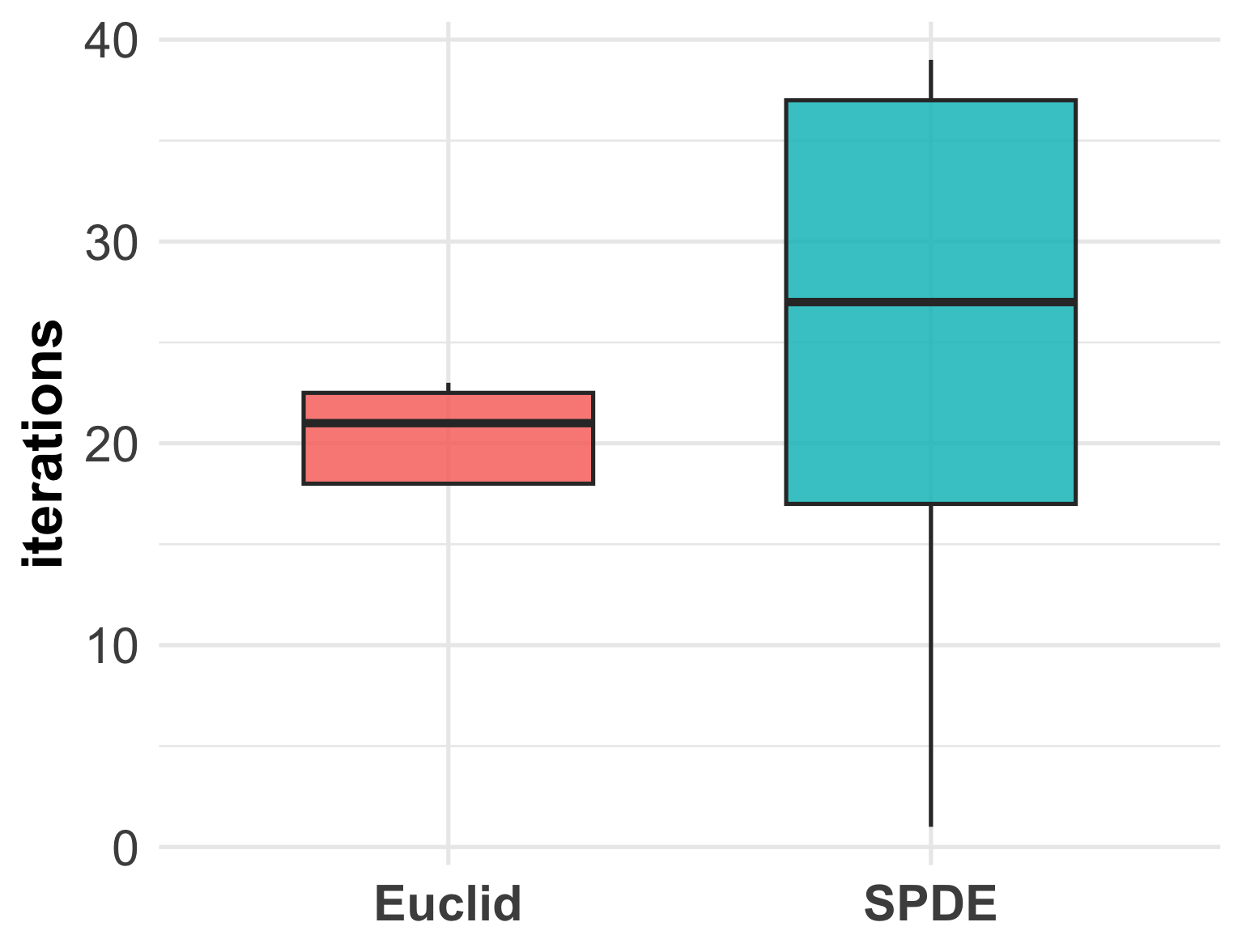}
  \caption{Iterations to $\mathsf{Tol}$}
\end{subfigure}\hfill

\caption{Inverse problem on telecommunication network. 
Layout matches Figure~\ref{fig:benchmarks_IGP-UCB}.}
\label{fig:ipbo-ts-tele}
\end{figure}

\section{Conclusions}\label{sec:conclusions}
This paper has investigated Bayesian optimization on networks modeled as metric graphs. Adopting Whittle-Matérn Gaussian priors, we have established in Theorem \ref{thm:exact kernel setting} regret bounds in an idealized setting in which the kernel is chosen to match the smoothness of the objective and the exact Whittle-Matérn kernel is used, without accounting for discretization error. We have also analyzed in Theorem \ref{thm:regret fem} the practical setting in which the smoothness of the objective is unknown and finite element representations of the Whittle-Matérn kernel are employed. In so doing, we have developed new theory for Bayesian optimization under kernel misspecification. Through numerical experiments, we demonstrated the advantage of Whittle--Mat\'ern kernels naturally adapted to the network geometry over standard kernels based on Euclidean distance, \red and showed that the proposed methods remain effective on multi-modal benchmark objectives.\nc

\red
An important direction for future work is to further improve performance in multi-modal settings, e.g. network inverse problems in which the posterior can exhibit multiple competing modes due to symmetries, limited sensing, or observational noise.
In such regimes, it can be valuable to move beyond returning a single maximizer and instead report several high-posterior candidates that capture different plausible explanations.
One natural approach is to draw multiple posterior sample paths and run GP--TS repeatedly, yielding a collection of distinct maximizers corresponding to different posterior realizations of the objective.
In parallel, IGP--UCB can be augmented with diversity or penalization mechanisms to discourage near-duplicate proposals and promote exploration across multiple high-posterior regions.
These ideas also motivate batch variants that query $K$ points per round, which can reduce the number of sequential rounds, and, in multi-modal settings, increase the chance of identifying multiple competing high-posterior regions within a fixed evaluation budget.
\nc

\section*{Acknowledgments}
Part of this research was performed while the authors were visiting the Institute for Mathematical and Statistical Innovation (IMSI), which is supported by the National Science Foundation (Grant No. DMS-2425650). DSA was partly funded by NSF CAREER DMS-2237628. 
The authors would like to thank Hwanwoo Kim for insightful discussions.

\bibliographystyle{siam}
\bibliography{references}

\appendix

\section{Misspecified IGP-UCB and GP-TS}\label{sec:misspecified BO}
In this appendix, we study misspecified UCB and TS where the function $\truth$ to be optimized does not belong to the RKHS of the kernel $k$ used for computation. Theory for misspecified GP-UCB has been studied in \cite{bogunovic2021misspecified} and here we extend the analysis to GP-TS. Since our theory applies beyond the metric graph setting, we consider a general abstract setup. 

Let $(D,d,\mu)$ be a metric measure space and the objective function $\truth\in L^\infty(\mu)$. Suppose $k:D\times D \rightarrow \R$ is the kernel used for computation with RKHS $\Hexact$, satisfying $|k(x,x')|\leq \overline{k}$ for some constant $\overline{k}<\infty$. 
 Assume that the noise sequence $\{\varepsilon_t\}$ in \eqref{eq:obsevation} is conditionally $R$-sub-Gaussian for some fixed constant $R>0$  with respect to the history up to time $t-1$. Let $D_t\subset D$ be a finite subset chosen at each iteration for TS acquisition function optimization satisfying 
\begin{align}\label{eq:appendix reduce to finite set}
    |\truth(x)-\truth([x]_t)| \leq 1/t^2, \qquad \forall x\in D,
\end{align}
where $[x]_t:=\operatorname{argmin}_{z\in D_t} d(x,z)$
is the point in $D_t$ closest to $x$. 
The next result bounds the simple regret (cf. \eqref{equa:simple regret}) of IGP-UCB and GP-TS when $\truth$ potentially does not belong to $\Hexact$. Specifically, we analyze Algorithm \ref{alg:GP-bandits} with $\Gamma$ and $\Gamma_t$ replaced with $D$ and $D_t$, as well as a modified choice of $\beta_t$ and $v_t$.   
 Again, we shall assume $\Ninit=0$ since the initial design does not affect the convergence rate.

\begin{theorem}\label{thm:regret mis TS}
Let $\truth\in L^\infty(\mu)$. Set in Algorithm \ref{alg:GP-bandits} 
\begin{align*}
     \beta_t &= B + R\sqrt{2\bigl(\gamma_{t-1}(k)+t(\lambda-1)/2+\ln(1/\delta)\bigr)}+b\frac{\sqrt{t-1}}{\sqrt{1+2/t}} \, ,\\
  v_t &= B + R\sqrt{2\bigl(\gamma_{t-1}(k)+t(\lambda-1)/2+\ln(2/\delta)\bigr)} +b\frac{\sqrt{t-1}}{\sqrt{1+2/t}} \, ,
\end{align*}
where $R$ is the sub-Gaussianity constant of the noise, $\delta\in(0,1)$, $\lambda=1+2/t$, and
\begin{align*}
    B:=\|f\|_{\mathcal{H}_k},\quad b=\|\truth-f\|_{L^\infty(\mu)}
\end{align*}
for an arbitrary $f\in \Hexact$.    
With probability at least $1-\delta$,
\begin{align*}
    r_T^{\UCB}&=O\left(\frac{\gamma_T(k)}{\sqrt{T}}+\sqrt{\frac{\gamma_T(k)}{T}}\left(B+\sqrt{\log(1/\delta)}\right)+b\sqrt{\gamma_T(k)}\right) \, ,\\
    r_T^{\TS}&=O\left(\sqrt{\log (|D_T|T^2)}\left[\frac{\gamma_T(k)}{\sqrt{T}}+\sqrt{\frac{\gamma_T(k)}{T}}\left(B+\|f\|_{L^\infty(\mu)}\sqrt{\log(1/\delta)}\right)+b\sqrt{\gamma_T(k)}\right]\right) \,.
\end{align*}
\end{theorem}

\begin{proof}
Since $\truth(x_t^*)\geq \truth(x_t)$, we have 
\begin{align*}
    r_T\leq \frac{1}{T}\sum_{t=1}^T \truth(x^*)-\truth(x_t)=:\frac{R_T}{T},
\end{align*}
so it suffices to bound the cumulative regret $R_T$.
The regret bound for UCB follows from \cite[Theorem 1]{bogunovic2021misspecified}. 
To show the bound for TS, notice that by \eqref{eq:appendix reduce to finite set}
\begin{align*}
    \frac{R_T}{T}&=\frac{1}{T}\sum_{t=1}^T \big[\truth(x^*)-\truth([x^*]_t)\big]+\big[\truth([x^*]_t)-f([x^*]_t)\big] +\big[f([x^*]_t)-f(x_t)\big] \\
    &\leq\frac{\pi^2}{6T}+ 2b+ \frac{1}{T}\sum_{t=1}^T f([x^*]_t)-f(x_t), 
\end{align*}
so it suffices to control the last term. 
By \cite[Lemma 4]{chowdhury2017kernelized} that $\sum_{t=1}^T \sigma_{t-1}(x_t) = O(\sqrt{T\gamma_T(k)})$ and Lemma \ref{lemma:cumulative regret TS} below, we have with probability at least $1-\delta$,
\begin{align*}
   & \sum_{t=1}^T f([x^*]_t)-f(x_t)\\
    &\leq\frac{11c_T}{p} \sum_{t=1}^T \sigma_{t-1}(x_t)+\frac{2(\|f\|_{L^\infty(\mu)}+1)\pi^2}{6}+\frac{(4\|f\|_{L^\infty(\mu)}+11)c_T}{p}\sqrt{2T\log (2/\delta)}\\
    & = O\left(c_T\left[\sqrt{T\gamma_T(k)}+\|f\|_{L^\infty(\mu)} \sqrt{T\log (1/\delta)}\right]\right)\\
    & = O\left(\sqrt{\log (|D_T|T^2)}\left[\sqrt{T}\gamma_T(k)+\sqrt{T\gamma_T(k)}\left(B+\|f\|_{L^\infty(\mu)}\sqrt{\log(1/\delta)}+b\sqrt{T}\right)\right]\right),
\end{align*}
where $c_T$ is defined in \eqref{eq:ct} below. 
The result follows. 
\end{proof}

Next we will show three lemmas that will be used to establish Lemma \ref{lemma:cumulative regret TS}. 
First, we derive a high probability concentration result for the posterior mean \eqref{eq:postmean} in terms of the standard deviation \eqref{eq:poststd}. 
We remark that there are two sources of randomness that we will explicate: (i) randomness of the observation noise (denoted below with superscript $\varepsilon$) and (ii) randomness of the Thompson sampling step (denoted below with superscript TS). 
We denote by $\F_t$ the filtration generated by the history up to time $t$.

\begin{lemma}\label{lemma:concentration}
Consider the two events 
\begin{equation}\label{eq:events}
\begin{aligned}
    E^{\varepsilon}(t)&:=\big\{|f(x)-\mu_{t-1}(x)|\leq v_t \sigma_{t-1}(x),\forall x\in D \bigr\},\\
    E^{TS}(t)&:= \bigl\{|f_{t}(x)-\mu_{t-1}(x)|\leq v_tw_t \sigma_{t-1}(x),\forall x\in D_t \bigr\},
\end{aligned}
\end{equation}
where $\lambda=1+2/t$ in \eqref{eq:postmean} and \eqref{eq:poststd}, and for $B=\|f\|_{\mathcal{H}_k}$ and $R$ the sub-Gaussianity constant of the noise
\begin{align*}
    v_t & = B+R\sqrt{2[\gamma_{t-1}(k)+1+ \log (2/\delta) ]} + b\frac{\sqrt{t-1}}{\sqrt{1+2/t}},\qquad 
    w_t = \sqrt{2\log |D_t|t^2} .
\end{align*} 
Then 
\begin{align}
    \P^{\varepsilon}\big\{E^{\varepsilon}(t)\big\}&\geq  1-\delta/2,\label{eq:concentration 1}\\
    \P^{TS}\big\{E^{TS}(t)\,|\,\F_{t-1}\big\} &\geq 1-t^{-2}. \label{eq:concentration 2}
\end{align}
\end{lemma}

\begin{proof}
We recall that the observations \eqref{eq:obsevation} are generated from the truth $\truth$ rather than its approximation $f$. 
This motivates defining the following surrogate data
\begin{align}
    \widetilde{y}_t = f(x_t) +\varepsilon_t
\end{align}
and the associated posterior mean 
\begin{align}\label{eq:surrogate pm}
    \widetilde{\mu}_{t}(x) = k_{t}(x)^\top (K_{t}+\lambda I_t)^{-1} \widetilde{Y}_t,\qquad \widetilde{Y}_t = [\widetilde{y}_1,\ldots,\widetilde{y}_t]^\top. 
\end{align}
The posterior standard deviation \eqref{eq:poststd} is independent of the data so it remains the same. 
By \cite[Theorem 2]{chowdhury2017kernelized} applied to $f$, we have with probability  $1-\delta/2$, 
\begin{align*}
    |f(x)-\widetilde{\mu}_{t-1}(x)|\leq \Big[B+R\sqrt{2[\gamma_{t-1}(k)+1+ \log (2/\delta) ]}\Big]\sigma_{t-1}(x), \qquad \forall x\in D, \quad \forall t \geq 1.
\end{align*}
Since $|\widetilde{y}_t-y_t|\leq \|\truth-f\|_{L^\infty(\mu)}\leq b$, 
\cite[Lemma 2]{bogunovic2021misspecified} implies that 
\begin{align*}
    |\widetilde{\mu}_{t-1}(x)-\mu_{t-1}(x)| \leq \frac{b\sqrt{t-1}}{\sqrt{1+2/t}}\sigma_{t-1}(x),\qquad \forall x\in D,\quad \forall t\geq 2.
\end{align*}
Therefore,
\begin{align*}
    \P^{\varepsilon}\Big\{|f(x)-\mu_{t-1}(x)|\leq v_t \sigma_{t-1}(x) \,\,\forall x\in D, \forall t \geq 1\Big\}\geq  1-\delta/2,
\end{align*}
where $v_t=B+R\sqrt{2[\gamma_{t-1}(k)+1+ \log (2/\delta)]}+b\sqrt{t-1}/\sqrt{1+2/t}$, establishing the first claim. 

For $f_t\sim \mathcal{N}(\mu_{t-1},v_t^2 \sigma_{t-1}^2)$, \cite[Lemma 5]{chowdhury2017kernelized} gives 
\begin{align*}
    \P^{TS}\Big\{|f_t(x)-\mu_{t-1}(x)|\leq  v_t\sqrt{2\log |D_t|t^2}\sigma_{t-1}(x)\,\, \forall x\in D_t \,\big |\, \F_{t-1}\Big\} \geq 1-t^{-2},
\end{align*}
establishing the second claim. 
\end{proof}

With such new concentration results, we shall proceed as in \cite{chowdhury2017kernelized} to bound the regret of the surrogate function $f$: $f([x^*]_t)-f(x_t)$. 
To make our presentation self-contained, we lay out the main ideas from \cite{chowdhury2017kernelized} below, 
which is to relate $f([x^*]_t)-f(x_t)$ with the posterior standard deviation $\sigma_{t-1}(x_t)$ whose sum can be controlled via maximum information gain.  
We classify the candidate $x_t\in D_t$ as 
\begin{align}\label{eq:def saturated points}
    \begin{cases}
    \text{ unsaturated } & \text{if } f([x^*]_t) - f(x_t) \leq c_t \sigma_{t-1}(x_t)\\
    \text{ saturated } & \text{if } f([x^*]_t) - f(x_t) >  c_t \sigma_{t-1}(x_t)
    \end{cases},
\end{align}
where 
\begin{align}\label{eq:ct}
    c_t := v_t (1+ w_t) = \left[B+R\sqrt{2[\gamma_{t-1}(k)+1+ \log (2/\delta)]}+b\frac{\sqrt{t-1}}{\sqrt{1+2/t}}\right]\left(1+\sqrt{2\log |D_t|t^2}\right).
\end{align}
The key of the analysis centers around the fact that one can show $x_t$ has a positive probability of being unsaturated at each iteration.

\begin{lemma}\label{lemma:unsaturated probability}
In the event where $E^\varepsilon(t)$ is true under $\F_{t-1}$,
\begin{align*}
    \P^{TS}\{x_t \text{ is unsaturated}\,| \, \F_{t-1}\} \geq p-t^{-2}, \qquad \forall t\geq 1,
\end{align*}    
where $p=\frac{1}{4e\sqrt{\pi}}$. 
\end{lemma}

\begin{proof}
Assuming $E^\varepsilon(t)$ is true, we shall prove that  
\begin{align}\label{eq: x unsaturated condition}
    x_t \text{ is unsaturated if } f_{t}([x^*]_t) > f([x^*]_t) \text{ and } E^{TS}(t).
\end{align} 
To see this, let $x_t$ be saturated, i.e., 
\begin{align*}
    f([x^*]_t)-f(x) > c_t \sigma_{t-1}(x) \quad \Longleftrightarrow \quad f(x)+c_t \sigma_{t-1}(x) <f([x^*]_t) . 
\end{align*}
This together with the events $E^\varepsilon(t)$ and $E^{TS}(t)$ implies that 
\begin{align*}
    f_{t}(x) \leq f(x)+v_t(1+w_t)\sigma_{t-1}(x) = f(x)+c_t\sigma_{t-1}(x) < f([x^*]_t),\qquad \forall x \in D_t.
\end{align*}
Therefore if further $f_{t}([x^*]_t) > f([x^*]_t)$ as in \eqref{eq: x unsaturated condition}, then 
\begin{align*}
    f_{t}(x)  < f_{t}([x^*]_t) \qquad \forall x \text{ saturated}. 
\end{align*}
This means that when maximizing $f_{t}$ at the $t$-th iteration in Thompson sampling, the candidate $x_t$ has to be unsaturated. 

To lower bound the probability, we have from \eqref{eq: x unsaturated condition} that 
\begin{align*}
   \P^{TS}\{x_t \text{ is unsaturated}\,| \, \F_{t-1}\} &\geq  \P^{TS}\{ f_{t}([x^*]_{t}) > f([x^*]_{t}) \text{ and } E^{TS}(t)\}\\
   &\geq  \P^{TS}\{ f_{t}([x^*]_t) > f([x^*]_t) \} - \P^{TS}\{ E^{TS}(t)^c\} \\
   &\geq \P^{TS}\{ f_{t}([x^*]_t) > f([x^*]_t) \} - t^{-2}
\end{align*}
by Lemma \ref{lemma:concentration}. 
Now since $f_{t}([x^*]_t)\sim \mathcal{N}(\mu_{t-1}([x^*]_t),v_t^2 \sigma_{t-1}^2([x^*]_t))$, 
\begin{align*}
    \P\left\{f_{t}([x^*]_t) > f([x^*]_t)\right\}&= \P\left\{\frac{f_{t}([x^*]_t)-\mu_{t-1}([x^*]_t)}{v_t \sigma_{t-1}([x^*]_t)} > \frac{f([x^*]_t)-\mu_{t-1}([x^*]_t)}{v_t \sigma_{t-1}([x^*]_t)}\right\}\\
    & \geq \P\left\{\frac{f_{t}([x^*]_t)-\mu_{t-1}([x^*]_t)}{v_t \sigma_{t-1}([x^*]_t)} > \frac{|f([x^*]_t)-\mu_{t-1}([x^*]_t)|}{v_t \sigma_{t-1}([x^*]_t)}\right\}\\
    & = \P \{ \mathcal{N}(0,1) > \theta_t\}, \qquad\qquad  \theta_t:=\frac{|f([x^*]_t)-\mu_{t-1}([x^*]_t)|}{v_t \sigma_{t-1}([x^*]_t)}.
\end{align*}
Since we are in the event of $E^\varepsilon(t)$, $0\leq \theta_t\leq 1$, therefore the last quantity is uniformly lower bounded by $\prob\{\mathcal{N}(0,1)>1\}\geq \frac{1}{4e\sqrt{\pi}}=:p$, proving the lemma. 
\end{proof}

\begin{lemma}\label{eq:expexted pstd bound}
In the event where $E^\varepsilon(t)$ is true under $\F_{t-1} $,
\begin{align*}
    \E^{TS} [f([x^*]_t) - f(x_t)\,|\,\F_{t-1}] \leq \frac{11c_t}{p}\E^{TS}[\sigma_{t-1}(x_t)\,|\,\F_{t-1}]+ 2\|f\|_{L^\infty(\mu)}t^{-2}.
\end{align*}
\end{lemma}

\begin{proof}
With Lemma \ref{lemma:unsaturated probability}, let us define 
\begin{align*}
    \overline{x}_t : =\underset{x \in D_t \text{ unsaturated }}{\operatorname{argmin}} \,\, \sigma_{t-1}(x).
\end{align*}
Then in the event of $E^\varepsilon(t)$, we have 
\begin{equation}\label{eq:lower bound xbart}
\begin{aligned}
    \E^{TS} [\sigma_{t-1}(x_t)\,|\, \F_{t-1}] &\geq \E^{TS} [\sigma_{t-1}(x_t)\,|\, \F_{t-1},x_t \text{ is unsaturated}] \P^{TS}\{x_t \text{ is unsaturated}\,|\, \F_{t-1}\}\\
    &\geq  (p-t^{-2})\sigma_{t-1}(\overline{x}_t).
\end{aligned}
\end{equation}
Now we can bound 
\begin{align*}
    f([x^*]_t) - f(x_t) &= [f([x^*]_t) - f(\overline{x}_t)] + [f(\overline{x}_t) - f(x_t)].
\end{align*}
The first term is bounded by $c_t\sigma_{t-1}(\overline{x}_t)$ by definition \eqref{eq:def saturated points} since $\overline{x}_t$ is unsaturated. 
In the event of $E^{TS}(t)$, the second term can be bounded by 
\begin{align*}
    f_t(\overline{x}_t) + c_t\sigma_{t-1}(\overline{x}_t) - f_t(x_t) + c_t\sigma_{t-1}(x_t) \leq c_t\sigma_{t-1}(\overline{x}_t) + c_t\sigma_{t-1}(x_t)
\end{align*}
because $x_t$ is chosen so that $f_t(x_t)\geq f_t(x) \,\, \forall x\in D_t$, in particular $f_t(\overline{x}_t)-f_t(x_t)\leq 0$.
Combining these and \eqref{eq:lower bound xbart}, we have 
\begin{align*}
    f([x^*]_t) - f(x_t) \leq 2c_t\sigma_{t-1}(\overline{x}_t) + c_t\sigma_{t-1}(x_t)
\end{align*}
in the event of $E^\varepsilon(t)$ and $E^{TS}(t)$, and hence 
\begin{align*}
    \E^{TS} [f([x^*]_t) - f(x_t)\,|\,\F_{t-1}] & =\E^{TS} \Big[\Big(f([x^*]_t) - f(x_t)\Big)\mathbf{1}_{E^{TS}(t)}\,|\,\F_{t-1}\Big] \\
     &\hspace{4cm} + \E^{TS} \Big[\Big(f([x^*]_t) - f(x_t)\Big)\mathbf{1}_{E^{TS}(t)^c}\,|\,\F_{t-1}\Big] \\
    &\leq 2c_t \E^{TS}[\sigma_{t-1}(\overline{x}_t)\,|\,\F_{t-1}]+ c_t \E^{TS}[\sigma_{t-1}(x_t)\,|\,\F_{t-1}] + 2\|f\|_{L^\infty(\mu)}t^{-2}\\
    &\leq c_t\left(\frac{2}{p-t^{-2}}+1\right)\E^{TS}[\sigma_{t-1}(x_t)\,|\,\F_{t-1}]+ 2\|f\|_{L^\infty(\mu)}t^{-2}\\
    & \leq \frac{11c_t}{p}\E^{TS}[\sigma_{t-1}(x_t)\,|\,\F_{t-1}]+ 2\|f\|_{L^\infty(\mu)}t^{-2},
\end{align*}
where we have used the fact that $1/(p-t^{-2})<5/p$. 
\end{proof}

\begin{lemma}\label{lemma:cumulative regret TS}
With probability at least $1-\delta$,
\begin{align*}
    \sum_{t=1}^T f([x^*]_t)-f(x_t) \leq \frac{11c_T}{p} \sum_{t=1}^T \sigma_{t-1}(x_t)+\frac{2(\|f\|_{L^\infty(\mu)}+1)\pi^2}{6}+\frac{(4\|f\|_{L^\infty(\mu)}+11)c_T}{p}\sqrt{2T\log (2/\delta)}\, .
\end{align*}
\end{lemma}
\begin{proof}
This follows from the proof of \cite[Lemma 13]{chowdhury2017kernelized} by replacing their $B$ with $\|f\|_{L^\infty(\mu)}$. 
\end{proof}

\section{Proofs of Technical Lemmas}\label{sec:tech proof}

\begin{proof}[Proof of Lemma \ref{lemma:uniform boundedness of fem eigenfunction}]
On each edge $e$, the eigenfunctions $\psi_i $  satisfy the equation $(\psi_i )^{''} +\mu_i \psi_i =0,$ where $\mu_i=\lambda_i-\kappa^2$, so that  
\begin{align*}
\psi_i  (z)= A_e \sin (\sqrt{\mu_i}z)+B_e\cos(\sqrt{\mu_i}z)=\sqrt{A_e^2+B_e^2}\cos(\sqrt{\mu_i}z+w_e), \quad z\in [0,L_e]
\end{align*}
for some constants $A_e,B_e,w_e$.
Since $\|\psi_i \|_{L^2(e)}\leq  \|\psi_i\|_{L^2(\Gamma)}= 1$, we have 
\begin{align*}
\|\psi_i \| _{L^\infty(e)}^2  \int_0^{L_e} \cos^2(\sqrt{\mu_i}z+w_e) \, dz\leq (A_e^2+B_e^2) \int_0^{L_e} \cos^2(\sqrt{\mu_i}z+w_e) \, dz\leq 1.
\end{align*}
Now notice that 
\begin{align*}
\int_0^{L_e} \cos^2(\sqrt{\mu_i}z+w_e) \, dz& =  \int_0^{L_e}\frac{1}{2}-\frac12 \cos(2\sqrt{\mu_i}z+2w_e) \, dz \\
&= \frac{L_e}{2} - \frac{1}{4\sqrt{\mu_i}} \sin(2\sqrt{\mu_i}z+2w_e)\Big|_0^{L_e} 
\geq \frac{L_e}{2}-\frac{1}{2\sqrt{\mu_i}}.
\end{align*}
The last step implies that 
\begin{align*}
\inf_i \int_0^{L_e} \cos^2(\sqrt{\mu_i}z+w_e) \, dz \geq \left[\underset{i=1,\ldots,I}{\operatorname{min}\,}\int_0^{L_e} \cos^2(\sqrt{\mu_i}z+w_e) \, dz \right]\wedge \left[\frac{L_e}{2}-\frac{1}{2\sqrt{\mu_I}}\right] >0,
\end{align*}
where $I$ is the smallest index such that $\frac{L_e}{2}-\frac{1}{2\sqrt{\mu_I}}>0$. 
Therefore, 
\begin{align*}
\underset{i}{\operatorname{sup}}\,\, \|\psi_i\|_{L^\infty{(\Gamma)}} \leq \underset{i}{\operatorname{sup}}\,\,\underset{e}{\operatorname{sup}}\,\, \|\psi_i \|_{L^\infty{(e)}}\leq \frac{1}{\inf_i \sqrt{\int_0^{L_e} \cos^2(\sqrt{\mu_i}z+w_e) \,  dz }}=:\Psi <\infty .
\end{align*}
Now by Weyl’s law (Theorem~2.12 in \cite{bolin2024regularity}), the eigenvalues satisfy $\lambda_i \asymp i^2$.  Hence, we have 
\begin{align*}
    |k(x,x')| \leq \Psi^2\tau^{-2}\sum_{i=1}^\infty i^{-4\alpha} =:\overline{k} <\infty
\end{align*}
provided that $4\alpha>1$. 

To see the last assertion, suppose first that $x$ and $x'$ belong to the same edge $e$. 
By a similar argument as above, we can show that $\|\nabla \psi_i\|_{L^\infty(e)}\leq \Psi \sqrt{\mu_i}$ so that 
\begin{align*}
    |\psi_i(x)-\psi_i(x')|\leq  \Psi \sqrt{\mu_i} d(x,x')\leq  \Psi \sqrt{\lambda_i} d(x,x'). 
\end{align*}
The general case follows by applying triangle inequality with the last bound along the shortest path between $x$ and $x'$. 
Therefore, we have
\begin{align*}
    |k(x'',x)-k(x'',x')| =\left| \tau^{-2} \sum_{i=1}^\infty \lambda_i^{-2\alpha} \psi_i(x'')[\psi_i(x)-\psi_i(x')]\right| \leq \Psi^2  \tau^{-2} \sum_{i=1}^\infty \lambda_i^{-2\alpha+1/2} d(x,x'), 
\end{align*}
where $\sum_{i=1}^\infty \lambda_i^{-2\alpha+1/2}\lesssim \int_1^\infty w^{-4\alpha+1} \, dw<\infty$ for $\alpha>\frac12$.
\end{proof}

\begin{lemma}\label{lemma:maximum info gain bound exact}
We have the following bound on the maximum information gain of the kernel \eqref{eq:kernel exact}
\begin{align*}
    \gamma_T(k) &= O(T^{1/(4\alpha)}\log T) \, .
\end{align*}
\end{lemma}
\begin{proof}
For a general Mercer kernel of the form $\mathcal{K}(x,x')=\sum_{i=1}^\infty \Lambda_i \Psi_i(x)\Psi_i(x')$, \cite[Theorem 3]{vakili2021information} gives a bound on the maximum information gain as  
\begin{align*}
    \gamma_T(\mathcal{K}) \leq \frac12 L \log \left(1+\frac{\overline{k}T}{\lambda L}\right)+ \frac{T}{2\lambda}\sum_{i=L+1}^\infty \Lambda_i \|\Psi_i\|_\infty^2 \, ,
\end{align*}
where we recall that $\lambda=1+2/T$ is the regularization parameter,  $\overline{k}$ is a constant satisfying $|\mathcal{K}(x,x')|\leq \overline{k}$, and $L\in \N$ is a suitable truncation level. 
We remark that the original result from \cite[Theorem 3]{vakili2021information} assumes a uniform upper bound on $\|\Psi_i\|_\infty$ but their proof indeed only requires an upper bound on the tail kernel $\sum_{i=L+1}^\infty \Lambda_i \Psi_i(x)\Psi_i(x')$. 

For the kernel \eqref{eq:kernel exact}, $\overline{k}$ can be taken as a uniform constant and the eigenfunctions are uniformly bounded by Lemma \ref{lemma:uniform boundedness of fem eigenfunction} so that the information gain scales as $O(L\log T+T\sum_{i=L+1}^\infty \Lambda_i)$. 
Since $\Lambda_i \asymp i^{-4\alpha}$ by Weyl's law \cite[Theorem~2.12]{bolin2024regularity}, we have $\sum_{i=L+1}^\infty \Lambda_i\asymp L^{1-4\alpha}$. Therefore, setting $L\asymp T^{1/(4\alpha)}$, we have 
\begin{align*}
    \gamma_T(k) = O(T^{1/(4\alpha)}\log T) \, . 
\end{align*}
\end{proof}

\begin{lemma}\label{lemma:equal white noise}
The white noise processes (cf. Remark \ref{remark:match with package})
\begin{align*}
    \mathcal{W}_h = \sum_{i=1}^{N_h} \xi_i \psi_{h,i},\qquad \widetilde{\mathcal{W}}_h = \sum_{i=1}^\infty \zeta_i P_h \psi_i
\end{align*}
are equal in distribution. 
\end{lemma}
\begin{proof}
    To prove the claim, we notice that $\W_h$ takes values in $V_h$ so that we have 
\begin{align}\label{eq:Wh second expansion}
    \widetilde{\W}_h=\sum_{k=1}^{N_h} \langle \widetilde{\W}_h,\psi_{h,k}\rangle \psi_{h,k} =: \sum_{k=1}^{N_h} \zeta_k \psi_{h,k}.
\end{align}
Furthermore, since $P_h$ is self-adjoint and $P_h \psi_{h,k}=\psi_{h,k}$ we have 
\begin{align*}
    \zeta_k = \langle \widetilde{\W}_h,\psi_{h,k}\rangle = \sum_{i=1}^\infty \xi_i \langle P_h \psi_i,\psi_{h,k}\rangle =\sum_{i=1}^\infty \xi_i \langle  \psi_i, \psi_{h,k}\rangle,
\end{align*}
is Gaussian, and 
\begin{align*}
    \E \zeta_\ell\zeta_k &= \E \sum_{i=1}^\infty\xi_i\langle \psi_i,\psi_{h,\ell}\rangle \sum_{j=1}^\infty \xi_j\langle \psi_j,\psi_{h,k}\rangle \\
    &= \sum_{i=1}^\infty  \langle \psi_i,\psi_{h,\ell}\rangle \langle \psi_i,\psi_{h,k}\rangle = \langle \psi_{h,\ell},\psi_{h,k}\rangle = \delta_{k\ell}, 
\end{align*}
where the last two steps are due to the orthonormality of $\{\psi_i\}_{i=1}^\infty$ and $\{\psi_{h,i}\}_{i=1}^{N_h}$ respectively. 
In other words, the $\zeta_i$'s are i.i.d. standard Gaussians so that \eqref{eq:Wh second expansion} establishes the claim. 
\end{proof}

\begin{lemma}\label{lemma:L infnity spectral error}
There exists a constant $C$ depending only on the metric graph such that 
\begin{align*}
    |\lambda_i-\lambda_{h,i}|&\leq C\lambda_i^2 h^2  \leq C \lambda_i,\\
    \|\psi_i-\psi_{h,i}\|_\infty &\leq C \lambda_i h^{3/2}\leq C\lambda_i^{1/4}.
\end{align*}
\end{lemma}
\begin{proof}
The standard results (see e.g. \cite[Section 8]{boffi2010finite} which indeed only requires to work with bilinear forms and extends to the metric graph setting) for FEM spectral approximation with linear finite elements give 
\begin{align*}
    |\lambda_i-\lambda_{h,i}|&\leq C\lambda_i^2 h^2, \\
    \|\psi_i-\psi_{h,i}\|_2 &\leq C \lambda_i h^2,
\end{align*}
and the Galerkin projection satisfies
\begin{align*}
    \|\psi_i-P_h\psi_i\|_2 &\leq C h^2\|\psi_i\|_{H^2}\leq C\lambda_ih^2,\\
    \|\psi_i-P_h\psi_i\|_\infty &\leq C h^2|\log h|\|\psi_i\|_{W^2_\infty}\leq C\lambda_ih^2|\log h|,
\end{align*}
where the last step follows from the uniform boundedness of $L^2$ normalized $\psi_i$'s. 
To establish the $L^\infty$ bound, we write
\begin{align*}
    \|\psi_i-\psi_{h,i}\|_\infty \leq \|\psi_i-P_h\psi_{i}\|_\infty + \|P_h\psi_{i}-\psi_{h,i}\|_\infty.
\end{align*}
By inverse estimates 
\begin{align*}
    \|P_h\psi_{i}-\psi_{h,i}\|_\infty \leq h^{-1/2} \|P_h\psi_{i}-\psi_{h,i}\|_2 \leq h^{-1/2} (\|P_h\psi_{h,i}-\psi_i\|_2+\|\psi_i-\psi_{h,i}\|_2)\leq C\lambda_i h^{3/2},
\end{align*}
which proves the first desired upper bounds. 
To see the second ones, notice that Weyl's law implies $\lambda_i\asymp i^2$ so that 
\begin{align*}
    h \asymp N_h^{-1} \asymp \lambda_{h,N_h}^{-1/2} \leq \lambda_{h,i}^{-1/2}.  
\end{align*}
\end{proof}

\begin{proof}[Proof of Lemma \ref{lemma:kernel in terms of fem basis}]
Expanding $\psi_{h,i}$ in terms of the FEM basis, we have 
\begin{align*}
    \psi_{h,i}= \sum_{j=1}^{N_h} \langle \psi_{h,i},e_{h,j}\rangle_{L^2(\Gamma)} e_{h,j} =: \sum_{j=1}^{N_h} U_{ij} e_{h,j}.
\end{align*}
Since $\psi_{h,i}$'s (variational) eigenfunctions of $\L_h$ associated with eigenvalues $\lambda_{h,i}$, we have 
\begin{align*}
    \kappa^2\langle \psi_{h,i},e_{h,k} \rangle_{L^2(\Gamma)}+\langle \nabla \psi_{h,i},\nabla e_{h,k}\rangle_{L^2(\Gamma)}  = \lambda_{h,i}\langle \psi_{h,i},e_{h,k}\rangle_{L^2(\Gamma)}
\end{align*}
for all $i,k=1,\ldots,N_h$. 
In other words, 
\begin{align*}
    \sum_{j=1}^{N_h} U_{ij}\langle \nabla e_{h,j},\nabla e_{h,k}\rangle_{L^2(\Gamma)}  = (\lambda_{h,i}-\kappa^2)\sum_{j=1}^{N_h}U_{ij}\langle e_{h,j},e_{h,k}\rangle_{L^2(\Gamma)},\qquad \forall i,k=1,\ldots,N_h,  
\end{align*}
which in terms of \eqref{eq:mass and stifness} reads  
\begin{align}\label{eq:fem matrix eq 1}
    \sum_{j=1}^{N_h} U_{ij} G_{jk} = (\lambda_{h,i}-\kappa^2) \sum_{j=1}^{N_h} U_{ij} C_{jk}\qquad \forall i,k=1,\ldots,N_h  \quad \Longleftrightarrow \quad UG = (\Lambda-\kappa^2 I) U C,
\end{align}
where $\Lambda = \operatorname{diag}(\lambda_{h,1},\ldots,\lambda_{h,N_h})$. 
Moreover, we have 
\begin{align*}
    \delta_{ij} = \langle \psi_{h,i},\psi_{h,j}\rangle_{L^2(\Gamma)} &= \left\langle \sum_{k=1}^{N_h} U_{ik} e_{h,k}, \sum_{k=1}^{N_h} U_{jk} e_{h,k}\right\rangle_{L^2(\Gamma)} \\
    &= \sum_{k=1}^{N_h}\sum_{\ell=1}^{N_h} U_{ik}U_{j\ell} \langle e_{h,k},e_{h,\ell}\rangle_{L^2(\Gamma)} =  (UCU^\top)_{ij},
\end{align*}
i.e., 
\begin{align}\label{eq:C=inv(UUT)}
    UCU^\top=I \quad  \Longleftrightarrow\quad  C=U^{-1}U^{-\top}.  
\end{align}
This together with \eqref{eq:fem matrix eq 1} implies that 
\begin{align}\label{eq:fem matrix eq 2}
    \Lambda = \kappa^2 I+ UGU^\top = U(\kappa^2C+G)U^\top.
\end{align}
Denoting $\psi(x) = (\psi_{h,1}(x),\ldots,\psi_{h,N_h}(x))^\top$, we have $\psi(x)=U e(x)$. 
Together with \eqref{eq:fem matrix eq 2}, 
\begin{align*}
k_h(x,x') = \tau^{-2} \psi(x)^\top \Lambda^{-2\alpha} \psi(x') = \tau^{-2} e(x)^\top U^\top [U(\kappa^2C+G)U^\top]^{-2\alpha}U e(x').    
\end{align*}
Now if $2\alpha \in \mathbb{N}$, the above expression simplifies to 
\begin{align*}
    k_h(x,x') = \tau^{-2} e(x)^\top Q^{-1} e(x'), \qquad Q=C [\kappa^2I+C^{-1}G]^{2\alpha},
\end{align*}
proving the desired result.

To show the second part, notice that 
$K^h_{t-1} = \tau^{-2}E_{t-1}Q^{-1}E_{t-1}^\top$ and we have 
\begin{align*}
    (\tau^{-2}E_{t-1}^\top Q^{-1}E_{t-1} + \lambda I)^{-1} = \lambda^{-1}I -\lambda^{-1} E_{t-1}^\top(\tau^2\lambda  Q +E_{t-1}E_{t-1}^\top)^{-1}E_{t-1}.
\end{align*}
We also have $k_{t-1}(x)^\top=\tau^{-2}e(x)^\top Q^{-1}E_{t-1}$ and so
\begin{align*}
    \mu_{t-1}^h  &= \tau^{-2}e(x)^\top Q^{-1} E_{t-1} \left[ \lambda^{-1}I -\lambda^{-1} E_{t-1}^\top (\tau^2\lambda Q +E_{t-1}E_{t-1}^\top)^{-1}E_{t-1}\right]Y_{t-1}\\
    & = \tau^{-2}e(x)^\top Q^{-1}\lambda^{-1}\left[I -E_{t-1}E_{t-1}^\top(\tau^2\lambda Q +E_{t-1}E_{t-1}^\top)^{-1}\right]E_{t-1}Y_{t-1}\\
    &= \tau^{-2}e(x)^\top Q^{-1}\lambda^{-1}\tau^2\lambda Q (\tau^2\lambda Q+E_{t-1}E_{t-1}^\top)^{-1}E_{t-1}Y_{t-1}\\
    & = e(x)^\top (\tau^2\lambda Q+E_{t-1}E_{t-1}^\top)^{-1}E_{t-1}Y_{t-1}.
\end{align*}
Lastly, 
\begin{align*}
    &k_{t-1}^h(x,x') \\
    &=\tau^{-2} e(x)^\top Q^{-1} e(x')  \\
    & \hspace{0.75cm} - \tau^{-2}e(x)^\top Q^{-1}E_{t-1}\left[ \lambda^{-1}I -\lambda^{-1} E_{t-1}^\top(\tau^2\lambda Q +E_{t-1}E_{t-1}^\top)^{-1}E_{t-1}\right] E_{t-1}^\top Q^{-1}e(x')\tau^{-2}\\
    & = \tau^{-2}e(x)^\top Q^{-1} e(x')  \\
    & \hspace{0.75cm} - \tau^{-2}e(x)^\top Q^{-1}\lambda^{-1}\left[E_{t-1}E_{t-1}^\top -E_{t-1}E_{t-1}^\top (\tau^2\lambda Q +E_{t-1}E_{t-1}^\top )^{-1}E_{t-1}E_{t-1}^\top \right]Q^{-1}e(x')\tau^{-2}\\
    & = \tau^{-2}e(x)^\top Q^{-1} e(x') - \tau^{-2}e(x)^\top Q^{-1} E_{t-1}E_{t-1}^\top (\tau^2\lambda Q +E_{t-1}E_{t-1}^\top )^{-1} e(x')\\
    & = \tau^{-2}e(x)^\top Q^{-1} \left[I- E_{t-1}E_{t-1}^\top (\tau^2\lambda Q +E_{t-1}E_{t-1}^\top )^{-1}\right]e(x')\\
    & = \lambda e(x)^\top (\tau^2\lambda Q +E_{t-1}E_{t-1}^\top)^{-1} e(x').
\end{align*}
\end{proof}

\begin{proof}[Proof of Lemma \ref{lemma:rational kernel in terms of fem basis}]
First, note that we can write as in the proof of Lemma \ref{lemma:kernel in terms of fem basis} that 
\begin{align}\label{eq:fem matrix eq 3}
    k_h^{\texttt{r}}(x,x') = \tau^{-2} e(x)^\top U^\top  r(\Lambda^{-1})^2 Ue(x') = \tau^{-2} e(x)^\top U^\top p_r(\Lambda)p_\ell(\Lambda)^{-2}p_r(\Lambda) Ue(x'),
\end{align}
where $U_{ij}=\langle \psi_{h,i},e_{h,j}\rangle_{L^2(\Gamma)}$.
We remark that $p_\ell(\Lambda)$ commutes with $p_r(\Lambda)$ since both are diagonal and the order of the product in the middle is chosen for convenience. 
Recall from \eqref{eq:fem matrix eq 2} and \eqref{eq:C=inv(UUT)} that $\Lambda = U(\kappa^2 C+G)U^\top$ and $C=(U^\top U)^{-1}$. 
Therefore for any integer $k$, $\Lambda^k  = UC(\kappa^2I+C^{-1}G)^kU^\top $ so that for any polynomial $p$, \[p(\Lambda) = UCp(\kappa^2I+C^{-1}G)U^\top .\]
Hence using that $C=(U^\top U)^{-1}$
\begin{align*}
    &U^\top p_r(\Lambda)p_\ell(\Lambda)^{-2}p_r(\Lambda) U\\
    &= U^\top \left[UCp_r(\kappa^2I+C^{-1}G)U^\top\right]
    \left[U^{-\top}p_\ell(\kappa^2I+C^{-1}G)^{-1}C^{-1}U^{-1}\right]\left[U^{-\top}p_\ell(\kappa^2I+C^{-1}G)^{-1}C^{-1}U^{-1}\right]\cdot\\
    &\qquad \left[UCp_r(\kappa^2I+C^{-1}G)U^\top\right]U\\
    &= p_r(\kappa^2I+C^{-1}G)p_\ell(\kappa^2I+C^{-1}G)^{-1}p_\ell(\kappa^2I+C^{-1}G)^{-1}p_r(\kappa^2I+C^{-1}G)U^\top U\\
    &=P_r P_\ell^{-1}P_\ell^{-1} P_rC^{-1}. 
\end{align*}
We notice that both $CP_\ell$ and $CP_r$ are symmetric so that $P_\ell^{-1} = C^{-1}P_\ell^{-\top} C$ and $CP_rC^{-1}=P_r^\top$, implying that the we can further write the last expression as $P_rP_\ell^{-1}C^{-1}P_\ell^{-\top} CP_rC^{-1} = P_rP_\ell^{-1}C^{-1}P_\ell^{-\top}P_r^\top$
so that 
\begin{align*}
    k_h^{\texttt{r}}(x,x') = \tau^{-2}e(x)^\top P_r (P_\ell^\top C P_\ell)^{-1} P_r^\top e(x'). 
\end{align*}

To show the second part, notice that 
\begin{align*}
    K_{t-1}^{h,\texttt{r}} = \tau^{-2} (E_{t-1}P_r) (P_\ell^\top C P_\ell)^{-1} (P_r^\top E_{t-1}^\top)&=: \tau^{-2}\widetilde{E}_{t-1} \widetilde{Q}^{-1} \widetilde{E}_{t-1}^\top , \\
    k_{t-1}^{h,\texttt{r}}(x)^\top = \tau^{-2} (e(x)^\top P_r)  (P_\ell^\top C P_\ell)^{-1} (P_r^\top E_{t-1}) &=: \tau^{-2}\widetilde{e}(x)^\top \widetilde{Q}^{-1} \widetilde{E}_{t-1}.
\end{align*}
By a similar argument as in the proof of Lemma \ref{lemma:kernel in terms of fem basis}, we have
\begin{align*}
    \mu_{t-1}^{h,\texttt{r}}& = \widetilde{e}(x)^\top (\tau^2\lambda \widetilde{Q} +\widetilde{E}_{t-1}\widetilde{E}_{t-1}^\top)^{-1}\widetilde{E}_{t-1}Y_{t-1},\\
    k_{t-1}^{h,\texttt{r}}(x,x')&= \lambda \widetilde{e}(x)^\top (\tau^2\lambda \widetilde{Q}+\widetilde{E}_{t-1}\widetilde{E}_{t-1}^\top)^{-1}\widetilde{e}(x'),
\end{align*}
 which gives the desired result by plugging in the expressions for $\widetilde{E}_{t-1},\widetilde{Q},\widetilde{e}$. 
\end{proof}

\begin{lemma}\label{lemma:uniform bound on rational kernel}
Suppose $m$ is chosen so that $\pi\sqrt{|\alpha-m_\alpha|m}\gtrsim -(1\vee\alpha)\log h$. We have  
\begin{align*}
    |k_h^{\texttt{r}}(x,x')| \lesssim h^{(4\alpha-2)\wedge 0},\qquad \forall x,x'\in \Gamma.
\end{align*}
\end{lemma}
\begin{proof}
Let \(s_h\) be defined in Eq.~\eqref{eq:s_h}.
We have 
\begin{align*}
    |s_h(\lambda_{h,i}^{-1})-\lambda_i^{-\alpha}| \leq |s_h(\lambda_{h,i}^{-1})-\lambda_{h,i}^{-\alpha}|+|\lambda_{h,i}^{-\alpha}-\lambda_i^{-\alpha}|.
\end{align*}
The first term can be controlled as in \cite[Appendix B]{bolin2020rational} by 
\begin{align*}
    \underset{i=1,\ldots,N_h}{\operatorname{max}}\,\, |\lambda_{h,i}^{-\alpha} - s_h(\lambda_{h,i}^{-1})| \lesssim \lambda_{h,N_h}^{(1-\alpha)\vee 0} e^{-2\pi \sqrt{|\alpha-m_\alpha|m}}.
\end{align*}    
For the second term, notice that for $x<y$ positive, we have $|x^{-\alpha}-y^{-\alpha}|\leq \alpha \xi^{-\alpha-1} |x-y|$ for some $x\leq \xi\leq y$. 
Using the fact that $\lambda_i\leq \lambda_{h,i}$, we have 
\begin{align*}
    |\lambda_i^{-\alpha}-\lambda_{h,i}^{-\alpha}|\leq \alpha \lambda_i^{-\alpha-1} |\lambda_i-\lambda_{h,i}| \leq \alpha \lambda_i^{-\alpha+1}h^2,
\end{align*}
where we used Lemma \ref{lemma:L infnity spectral error} in the last step. 
Furthermore, notice that 
\begin{align*}
    h^2 \lesssim N_h^{-2} \lesssim \lambda_{h,N_h}^{-1}\leq \lambda_{h,i}^{-1}\leq \lambda_i^{-1},\qquad \forall i=1,\ldots,N_h,
\end{align*}
so that 
\begin{align*}
    |\lambda_i^{-\alpha}-\lambda_{h,i}^{-\alpha}| \lesssim \lambda_i^{-\alpha}.
\end{align*}
Combining these bounds we have 
\begin{align}\label{eq:bound on r(lambda)}
    |s_h(\lambda_{h,i}^{-1})-\lambda_{i}^{-\alpha}|\lesssim \lambda_i^{-\alpha}+ \lambda_{h,N_h}^{(1-\alpha)\vee 0} e^{-2\pi \sqrt{|\alpha-m_\alpha|m}}\lesssim \lambda_i^{-\alpha}
\end{align}
when $m$ is chosen as in the statement of the lemma. 
Together with Lemma \ref{lemma:L infnity spectral error} that $\|\psi_{h,i}\|_{L^\infty(\Gamma)} \lesssim \lambda_i^{1/4}$, we have 
\begin{align*}
    |k_h^{\texttt{r}}(x,x')| &\lesssim \sum_{i=1}^{N_h} \lambda_i^{-2\alpha}\lambda_i^{1/2}
    \lesssim  h^{(4\alpha-2)\wedge 0}.
\end{align*}
\end{proof}

\begin{lemma}\label{lemma:maximum info gain bound}
For $\alpha>\frac12$ and $\pi\sqrt{|\alpha-m_\alpha|m}\gtrsim -(1\vee\alpha)\log h$, we have the following bound on the maximum information gain
\begin{align*}
    \gamma_T(k_h^{\texttt{r}}) &=O(T^{1/(4\alpha-1)}\log T).
\end{align*}
\end{lemma}
\begin{proof}
For a general Mercer kernel of the form $\mathcal{K}(x,x')=\sum_{i=1}^\infty \Lambda_i \Psi_i(x)\Psi_i(x')$, \cite[Theorem 3]{vakili2021information} gives a bound on the maximum information gain as  
\begin{align*}
    \gamma_T(\mathcal{K}) \leq \frac12 L \log \left(1+\frac{\overline{k}T}{\lambda L}\right)+\frac12 \frac{T}{\lambda}\sum_{i=L+1}^\infty \Lambda_i \|\Psi_i\|_\infty^2,
\end{align*}
where we recall that $\lambda$ is the regularization parameter,  $\overline{k}$ is a constant satisfying $|\mathcal{K}(x,x')|\leq \overline{k}$, and $L\in \N$ is a suitable truncation level. 
We remark that the original result from \cite[Theorem 3]{vakili2021information} assumes a uniform upper bound on $\|\Psi_i\|_\infty$ but their proof indeed only requires an upper bound on the tail kernel $\sum_{i=L+1}^\infty \Lambda_i \Psi_i(x)\Psi_i(x')$. 

For the approximate kernel, we have $\overline{k} \lesssim h^{(4\alpha-2)\wedge 0 }$  as in Lemma \ref{lemma:uniform bound on rational kernel}, and $\|\psi_{h,i}\|_\infty \lesssim \lambda_i^{1/4}$ as in Lemma \ref{lemma:L infnity spectral error}, where then $\overline{k}/L  \lesssim h^{4\alpha-1}\lesssim 1$ since $\alpha>\frac14$. 
Moreover, for $\alpha>\frac12$ we have by \eqref{eq:bound on r(lambda)}
\begin{align*}
    \left|\sum_{i=L+1}^{N_h} s_h(\lambda_{h,i}^{-1})^2 \psi_{h,i}(x)\psi_{h,i}(x')\right| \lesssim \sum_{i=L+1}^{N_h} \lambda_i^{-2\alpha} \lambda_i^{1/2} \lesssim L^{2-4\alpha}
\end{align*}
so that we have $\gamma_T(k_h^{\texttt{r}})=O(L\log T+TL^{2-4\alpha})$.  
Setting $L\asymp T^{1/(4\alpha-1)}$, we have 
\begin{align*}
    \gamma_T(k_h^{\texttt{r}}) =O(T^{1/(4\alpha-1)}\log T). 
\end{align*}
\end{proof}


\section{Online Kernel Estimation}\label{appendix:online}
\begin{minipage}{1\linewidth}
\begin{algorithm}[H]
\caption{IGP-UCB and GP-TS with Online Kernel Estimation ($\alpha=1$)}
\label{alg: benchmarks}
\begin{algorithmic}[1]
\Require FEM mesh nodes $\femnodes$; Mat\'ern kernel via Euclidean distances (see Eq.~\ref{equa: eucl baseline}); SPDE/FEM kernel (see Eq.~\ref{eq:fem-spde}); parameters $B,R,b,\lambda,\delta$; noise level $\noisestd$; horizon $T$; initial size $N_{\rm init}$; prior hyperparameters $\theta$; fixed GP parameters $\sigma_0,\tau_0$; tolerance threshold $\mathsf{Tol}.$

\State Choose $X_{\mathrm{init}}=\{x_i^{(0)}\}_{i=1}^{N_{\mathrm{init}}}\subset \femnodes$.
\State Observe $y_i^{(0)} = \truth(x_i^{(0)}) + \varepsilon_i^{(0)}$, with $\varepsilon_i^{(0)} \stackrel{\text{i.i.d.}}{\sim}\mathcal N(0,\noisestd^2)$ for $i=1,\dots,N_{\mathrm{init}}$.
\State Initialize $\mcD_0 \gets \{(x_i^{(0)}, y_i^{(0)})\}_{i=1}^{N_{\mathrm{init}}},$ $\theta_0 = \ell_0$ (Euclidean), and $\theta_0 = \kappa_0$ (SPDE). 

\For{$t =1, \dots, T$}
  \State \textbf{Posterior update using current hyperparameters $\theta_{t-1}$:}
\Statex \hspace{\algorithmicindent}\textit{\textbf{Euclidean:}} Assemble the kernel Gram matrix
$\mathcal{K}^{\mathrm{Eucl}}_{t-1}=k_{\mathrm{Eucl}}(\theta_{t-1},\sigma_0)$ via Eq.~\eqref{equa: eucl baseline}. 

\Statex \hspace{\algorithmicindent}\textit{\textbf{SPDE:}} Form the precision matrix on $\femnodes$
\begin{equation}
\label{eq:posterior-precision}
  Q_{t-1}=C [\theta_{t-1}^2I+C^{-1}G]^{2},
\end{equation}
where $C$ and $G$ are defined in Lemma~\ref{lemma:kernel in terms of fem basis}, and compute the SPDE/FEM Gram matrix via
$\mathcal{K}^{\mathrm{SPDE}}_{t-1}=\tau_{0}^{-2}\,Q_{t-1}^{-1}$.

 \Statex \hspace{\algorithmicindent}We then compute the posterior mean $\mu_{t-1}^{h}$, covariance $k_{t-1}^{h}$, $K_{t-1}^h$ and standard deviation $\sigma_{t-1}^{h}$ on $\femnodes$ from the first $t-1$ acquisitions and observations via \cref{eq:fem postmean,eq:fem postcov,eq:fem poststd}.

  \State Define the acquisition function $\mathrm{acq}_{t}(x)$ for $x\in \femnodes$ as
\[
\mathrm{acq}_t(x)=
\begin{cases}
\mu_{t-1}^{h}(x) + \beta_t^{h}\, \sigma_{t-1}^{h}(x), & \text{(IGP--UCB)},\\[6pt]
f_t^{h}(x),\ \ f_t^{h} \sim \mathcal{GP}\!\bigl(\mu_{t-1}^{h},\,(v_t^{h})^2\, k_{t-1}^{h}\bigr), & \text{(GP--TS)},
\end{cases}
\]
with
\begin{align*}
  \beta_t^h &=  B+ R\sqrt{2\big(\gamma_{t-1}(k_h)+(\Ninit+t)(\lambda-1)/2+\log(1/\delta)\big)}+\frac{b\sqrt{\Ninit+t-1}}{\sqrt{1+2/(\Ninit+t)}},  
  \\
  v_t^h &=  B+ R\sqrt{2\big(\gamma_{t-1}(k_h)+(\Ninit+t)(\lambda-1)/2+\log(2/\delta)\big)}+\frac{b\sqrt{\Ninit+t-1}}{\sqrt{1+2/(\Ninit+t)}}. 
  \end{align*}

 \State Select $x_t \in \arg\max_{x\in \femnodes} \mathrm{acq}_t(x)$. Observe $y_t = \truth(x_t)+\varepsilon_t$, with $\varepsilon_t\sim \mathcal{N}(0,\noisestd^2)$.
 \State Update $\mcD_t \gets \mcD_{t-1} \cup \{(x_t, y_t)\}$.

  \State \textbf{Online MLE (hyperparameter update):} obtain $\theta_{t}$ by maximizing the marginal likelihood given history $\mathcal{D}_{t-1}$:
  \[
  \theta_{t}=\arg\min_{\theta}\Big\{\tfrac12\,y_{1:t}^\top \big(K_{t-1}^h(\theta)+\noisestd^{2}I_t\big)^{-1} y_{1:t}+\tfrac12\log\big|K^h_{t-1}(\theta)\big|\Big\}.
  \]
  \State \textbf{Stop:} if simple regret $\le \mathsf{Tol}$, record iterations to $\mathsf{Tol}$.
\EndFor
\end{algorithmic}
\end{algorithm}
\end{minipage}

\begin{remark}
In the numerical tests, we choose $\alpha=1$ and derive the posterior covariance from the precision matrix assembled from $C$ and $G$ via \ref{eq:posterior-precision}. We can also handle $\alpha=2$ and fractional $\alpha$ via the rational SPDE approximation as stated in Lemma~\ref{lemma:kernel in terms of fem basis} and \ref{lemma:rational kernel in terms of fem basis}. Thus, Algorithm \ref{alg: benchmarks} extends to the fractional-$\alpha$ setting by changing only the posterior derivation from the precision matrix.
\end{remark}

\begingroup
\color{black}
\section{Complementary Numerical Results}
\subsection{Sensitivity to the Misspecification-correction Parameter $b$}
\label{app:b-sensitivity}

This appendix reports an ablation study for Algorithm~\ref{alg: benchmarks}, focusing on the
\emph{misspecification-correction} term $b$ that appears in $\beta_t^{h,\texttt{r}}$ (IGP-UCB) and $v_t^{h,\texttt{r}}$ (GP-TS).
Recall that, in the analysis of Section~\ref{sec:FEM}, the misspecification parameter $b$ is required to satisfy
$$
b \;\ge\; \inf_{f\in \mathcal{H}_{k_h^{\star}}} \|\truth-f\|_{L^\infty(\Gamma)},
$$
where $k_h^{\star}$ denotes either $k_h$ in Algorithm~\ref{alg:GP-bandits-FEM} ($2\alpha\in\mathbb{N}$) or $k_h^{\texttt{r}}$ in Algorithm~\ref{alg:GP-bandits-rational} ($2\alpha\notin\mathbb{N}$).
This term accounts for potential model mismatch between the unknown objective $\truth$ and the RKHS induced by the prior kernel, and is typically treated as an unknown hyperparameter in numerical tests.

We consider the open-rectangle metric graph shown in Figure~\ref{fig:true-benchmarks-openrectangle}.
The search domain is discretized by a continuous piecewise linear FEM space, and we restrict acquisitions to the FEM mesh nodes
$\femnodes$ with cardinality $N_h := |\femnodes| \approx 300$.
We use the normalized benchmark functions on metric graphs (e.g.\ Ackley, Rastrigin, and L\'evy), as defined in Subsection~\ref{ssec:benchmarks},
so that the common choices of hyperparameters in Algorithm~\ref{alg: benchmarks} are meaningful across all benchmarks.

We run both algorithms IGP-UCB and GP-TS in Algorithm~\ref{alg: benchmarks} under the same experimental design
(e.g.\ $\alpha,m,B,R,\lambda,\delta$, the horizon $T$, and $\Ninit$) as in the baseline $b=0$ case in Subsection~\ref{ssec:benchmarks}.
We consider the grid of misspecification-correction parameters
$$
b \in \{0, 0.05, 0.1, 0.2, 0.5\},
$$
and for each benchmark and each $b$, we repeat the experiment over 60 independent random seeds with different initializations.
The sensitivity results with respect to $b$ are reported in Tables~\ref{tab:benchmarks_IGP-UCB_b_sensitivity}--\ref{tab:benchmarks_GP-TS_b_sensitivity} and visualized in Figures~\ref{fig:benchmarks_IGP-UCB_b_sensitivity}--\ref{fig:benchmarks_GP-TS_b_sensitivity}.

\begin{table}[htbp]
\centering
\begin{minipage}{0.9\linewidth}
\centering
\resizebox{\linewidth}{!}{%
\begin{tabular}{r r r r r}
\toprule
$b$ & reach rate (SPDE) & mean iters to $\mathsf{Tol}$ (SPDE) & reach rate (Euclid) & mean iters to $\mathsf{Tol}$ (Euclid) \\
\midrule
\multicolumn{5}{c}{\textbf{Ackley}}\\
\midrule
0    & 100.0\% & 4.27 & 41.7\% & 11.16 \\
0.05 & 100.0\% & 4.45 & 40.0\% & 9.83 \\
0.1  & 100.0\% & 4.52 & 38.3\% & 10.26 \\
0.2  & 100.0\% & 4.63 & 40.0\% & 10.00 \\
0.5  & 100.0\% & 4.75 & 38.3\% & 9.48 \\
\midrule
\multicolumn{5}{c}{\textbf{Rastrigin}}\\
\midrule
0    & 100.0\% & 2.55 & 30.0\% & 12.83 \\
0.05 & 100.0\% & 2.65 & 30.0\% & 12.89 \\
0.1  & 100.0\% & 2.75 & 30.0\% & 12.89 \\
0.2  & 100.0\% & 2.85 & 30.0\% & 12.89 \\
0.5  & 100.0\% & 3.28 & 30.0\% & 13.00 \\
\midrule
\multicolumn{5}{c}{\textbf{L\'evy}}\\
\midrule
0    & 100.0\% & 4.28 & 5.0\%  & 19.67 \\
0.05 & 100.0\% & 4.42 & 5.0\%  & 11.33 \\
0.1  & 100.0\% & 4.53 & 3.3\%  & 16.50 \\
0.2  & 100.0\% & 4.60 & 3.3\%  & 16.50 \\
0.5  & 100.0\% & 4.68 & 3.3\%  & 16.50 \\
\bottomrule
\end{tabular}%
}
\end{minipage}
\caption{Sensitivity to the misspecification-correction parameter $b$ for Algorithm~\ref{alg: benchmarks} (\textbf{IGP-UCB}).}
\label{tab:benchmarks_IGP-UCB_b_sensitivity}
\end{table}

\begin{table}[H]
\centering
\begin{minipage}{0.9\linewidth}
\centering
\resizebox{\linewidth}{!}{%
\begin{tabular}{r r r r r}
\toprule
$b$ & reach rate (SPDE) & mean iters to $\mathsf{Tol}$ (SPDE) & reach rate (Euclid) & mean iters to $\mathsf{Tol}$ (Euclid) \\
\midrule
\multicolumn{5}{c}{\textbf{Ackley}}\\
\midrule
0    & 83.3\% & 15.68 & 16.7\% & 16.90 \\
0.05 & 81.7\% & 14.98 & 20.0\% & 18.25 \\
0.1  & 80.0\% & 14.83 & 16.7\% & 19.30 \\
0.2  & 78.3\% & 15.19 & 16.7\% & 18.80 \\
0.5  & 73.3\% & 14.57 & 15.0\% & 19.11 \\
\midrule
\multicolumn{5}{c}{\textbf{Rastrigin}}\\
\midrule
0    & 71.7\% & 14.72 & 16.7\% & 18.90 \\
0.05 & 81.7\% & 14.33 & 15.0\% & 18.44 \\
0.1  & 78.3\% & 13.79 & 15.0\% & 18.44 \\
0.2  & 83.3\% & 12.76 & 15.0\% & 18.22 \\
0.5  & 95.0\% & 11.05 & 15.0\% & 16.11 \\
\midrule
\multicolumn{5}{c}{\textbf{L\'evy}}\\
\midrule
0    & 85.0\% & 13.63 & 3.3\%  & 15.50 \\
0.05 & 78.3\% & 13.28 & 3.3\%  & 18.50 \\
0.1  & 76.7\% & 13.02 & 3.3\%  & 18.50 \\
0.2  & 75.0\% & 13.47 & 3.3\%  & 18.50 \\
0.5  & 71.7\% & 13.05 & 1.7\%  & 1.00 \\
\bottomrule
\end{tabular}
}
\end{minipage}
\caption{Sensitivity to the misspecification-correction parameter $b$ for Algorithm~\ref{alg: benchmarks} (\textbf{GP-TS}).}
\label{tab:benchmarks_GP-TS_b_sensitivity}
\end{table}

\begin{figure}[H]
\centering

\begin{minipage}{\textwidth}\centering
  \subcaption*{Setting 1: Ackley}
\end{minipage}

\begin{subfigure}{0.32\textwidth}
  \centering
  \includegraphics[width=\linewidth]{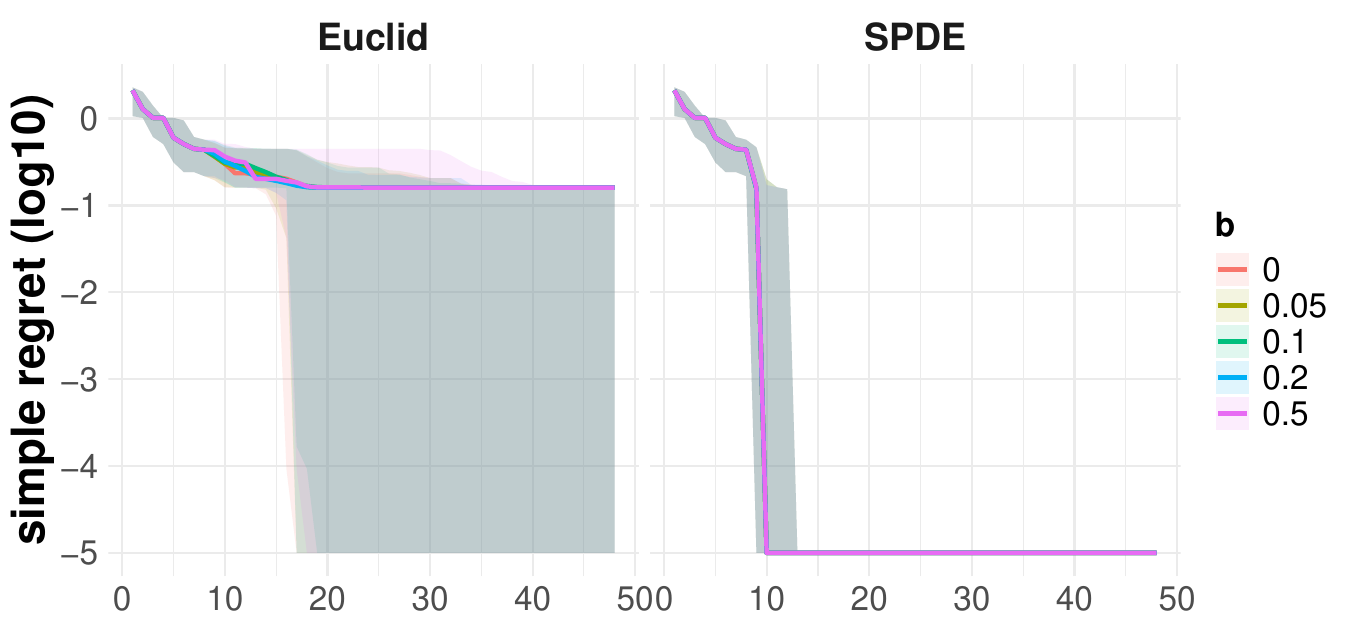}
\end{subfigure}\hfill
\begin{subfigure}{0.32\textwidth}
  \centering
  \includegraphics[width=\linewidth]{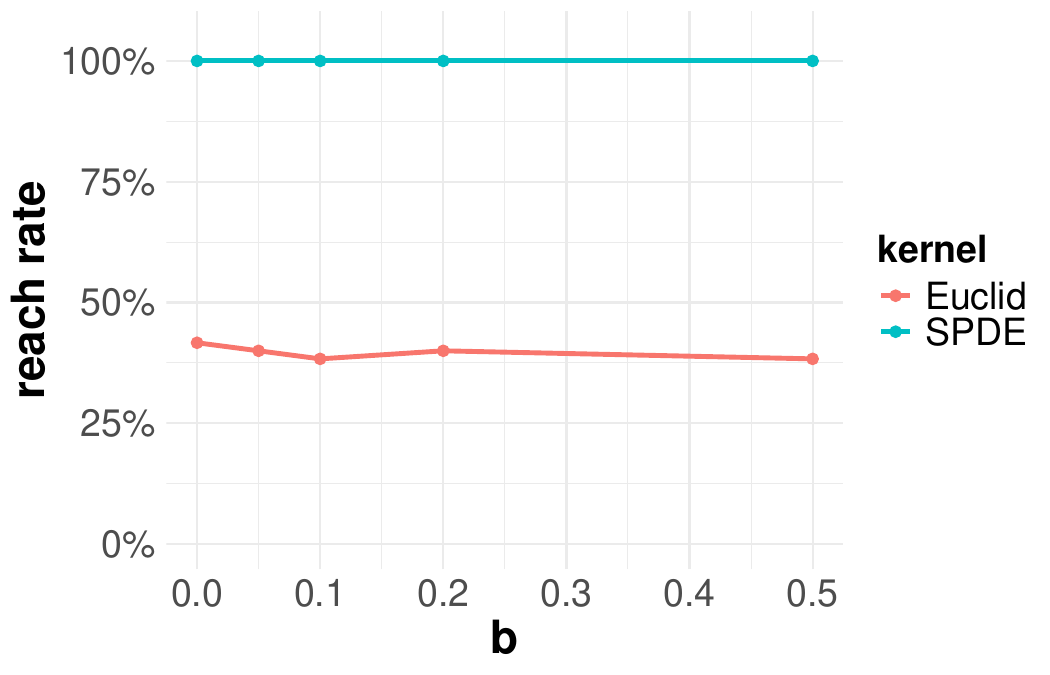}
\end{subfigure}
\begin{subfigure}{0.32\textwidth}
  \centering
  \includegraphics[width=\linewidth]{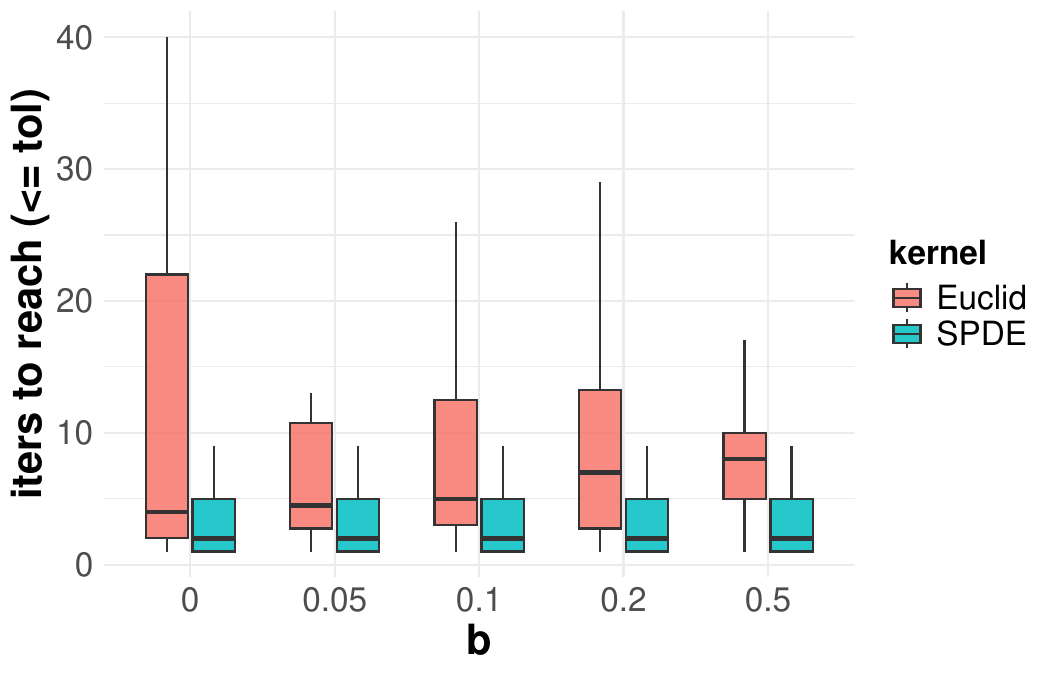}
\end{subfigure}\hfill

\begin{minipage}{\textwidth}\centering
  \subcaption*{Setting 2: Rastrigin}
\end{minipage}

\begin{subfigure}{0.32\textwidth}
  \centering
  \includegraphics[width=\linewidth]{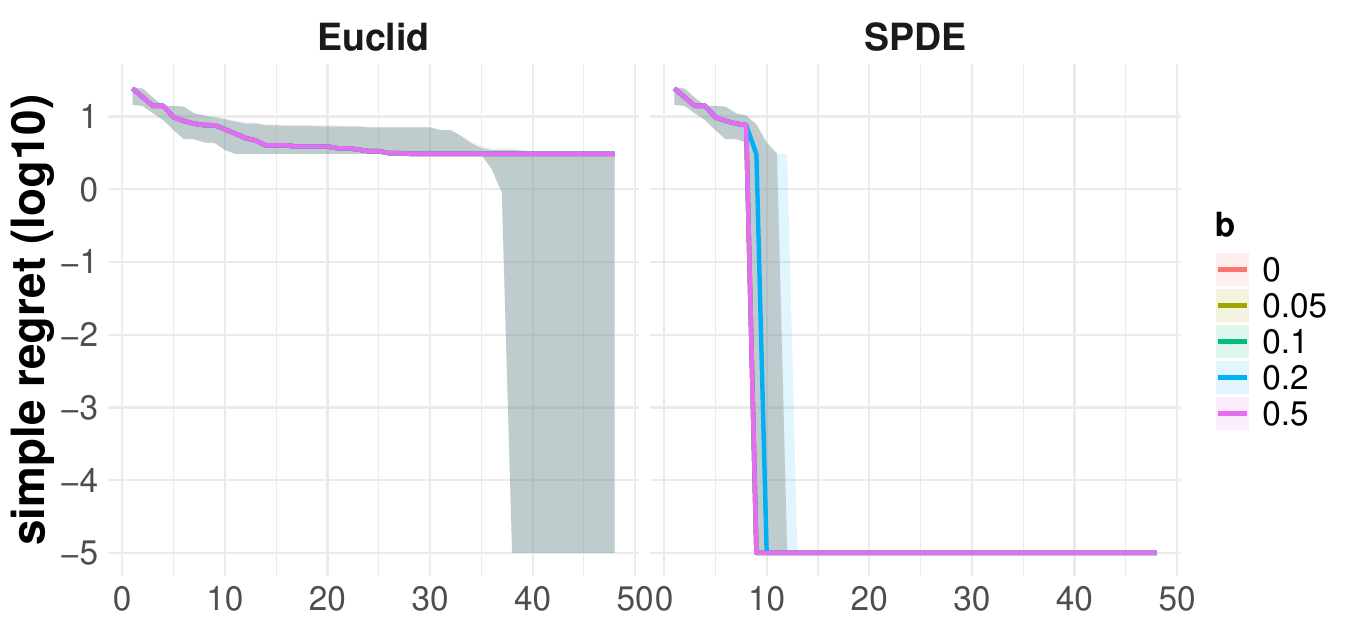}
\end{subfigure}\hfill
\begin{subfigure}{0.32\textwidth}
  \centering
  \includegraphics[width=\linewidth]{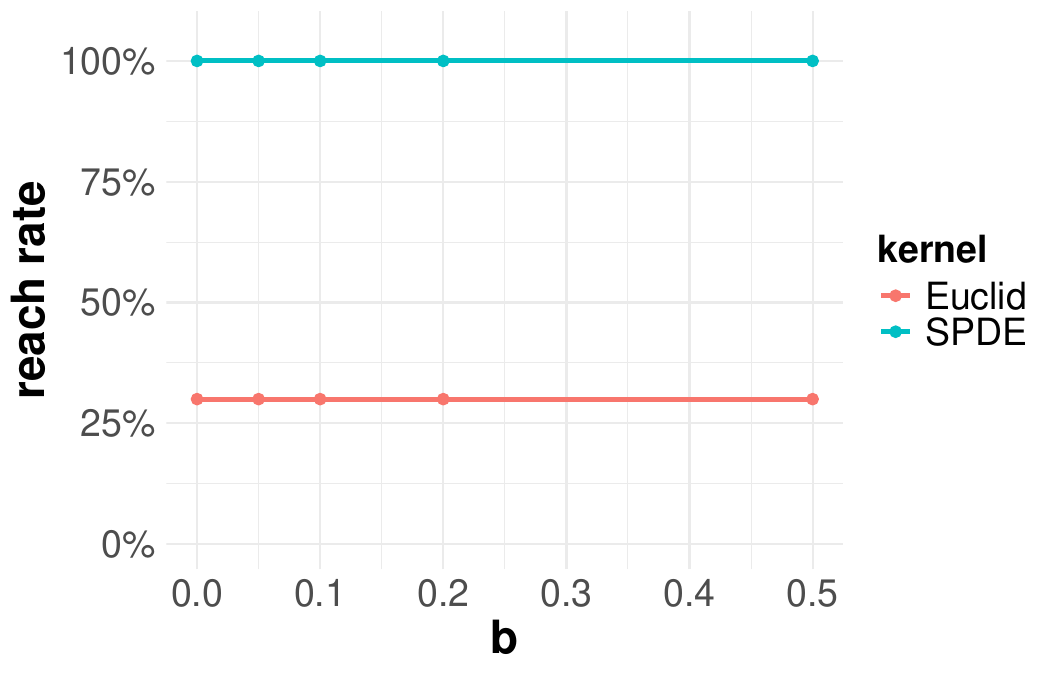}
\end{subfigure}
\begin{subfigure}{0.32\textwidth}
  \centering
  \includegraphics[width=\linewidth]{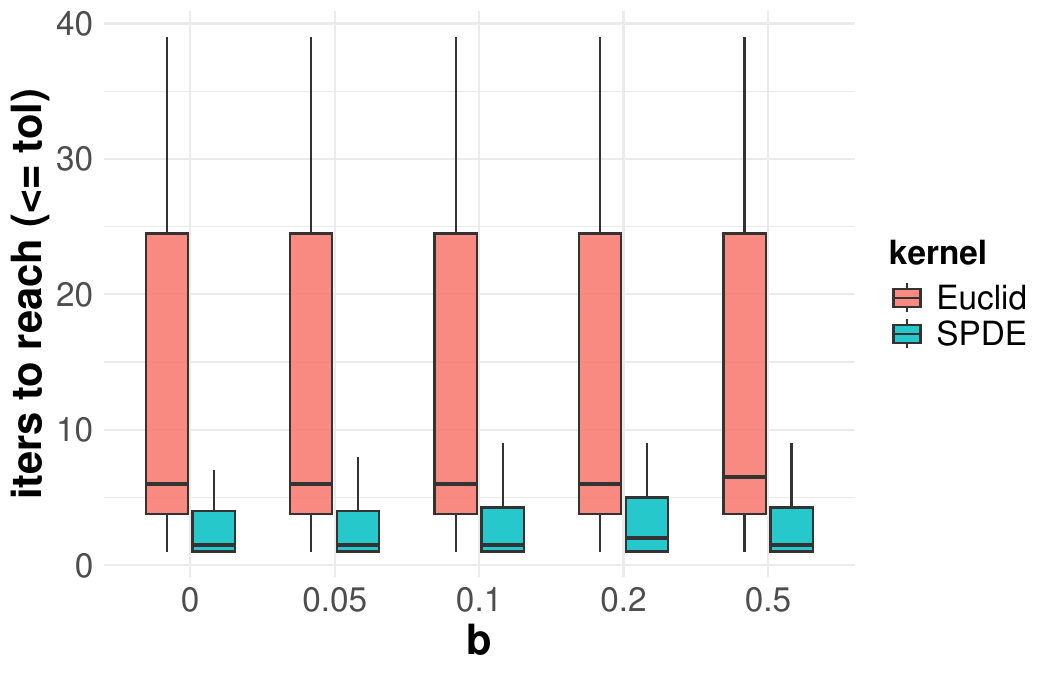}
\end{subfigure}\hfill

\vspace{0.6em}

\begin{minipage}{\textwidth}\centering
  \subcaption*{Setting 3: L\'evy}
\end{minipage}

\begin{subfigure}{0.32\textwidth}
  \centering
  \includegraphics[width=\linewidth]{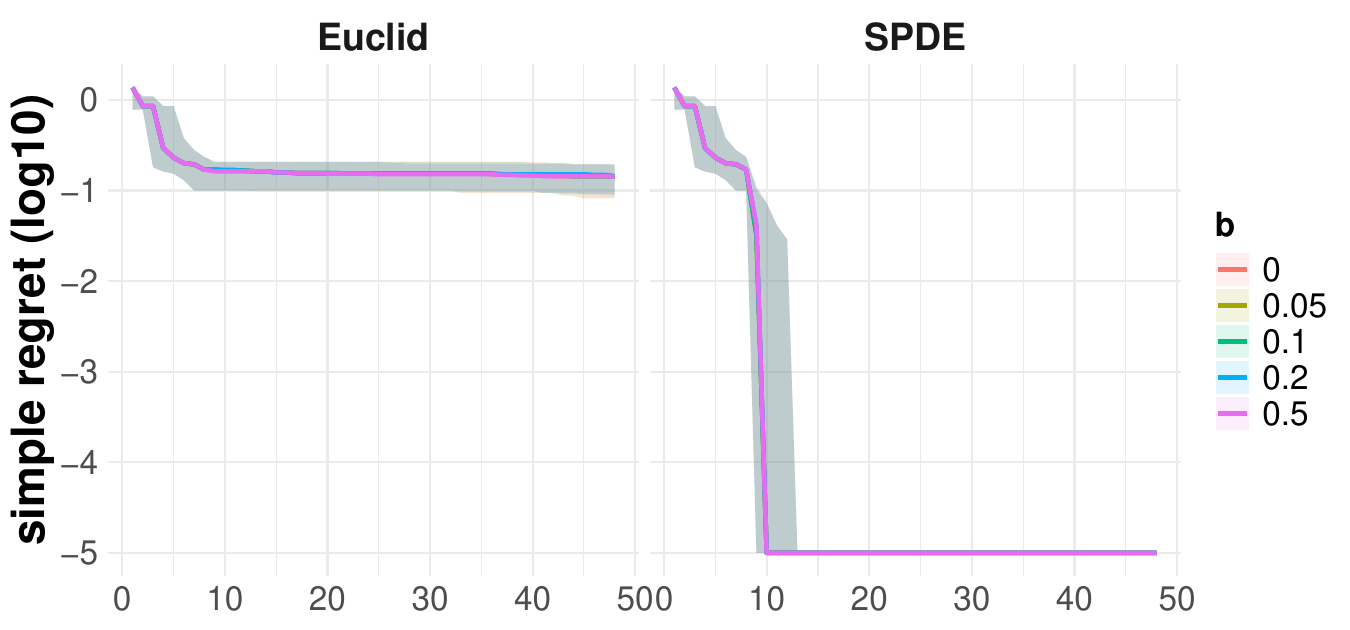}
  \caption{Simple regret}
\end{subfigure}\hfill
\begin{subfigure}{0.32\textwidth}
  \centering
  \includegraphics[width=\linewidth]{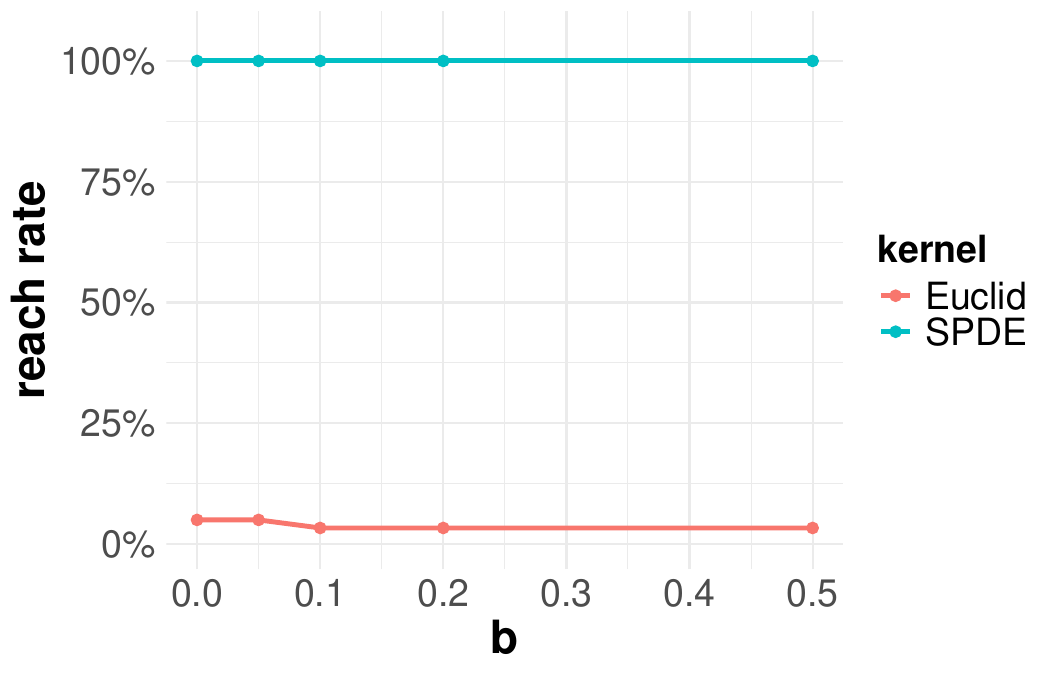}
  \caption{Reach rate}
\end{subfigure}
\begin{subfigure}{0.32\textwidth}
  \centering
  \includegraphics[width=\linewidth]{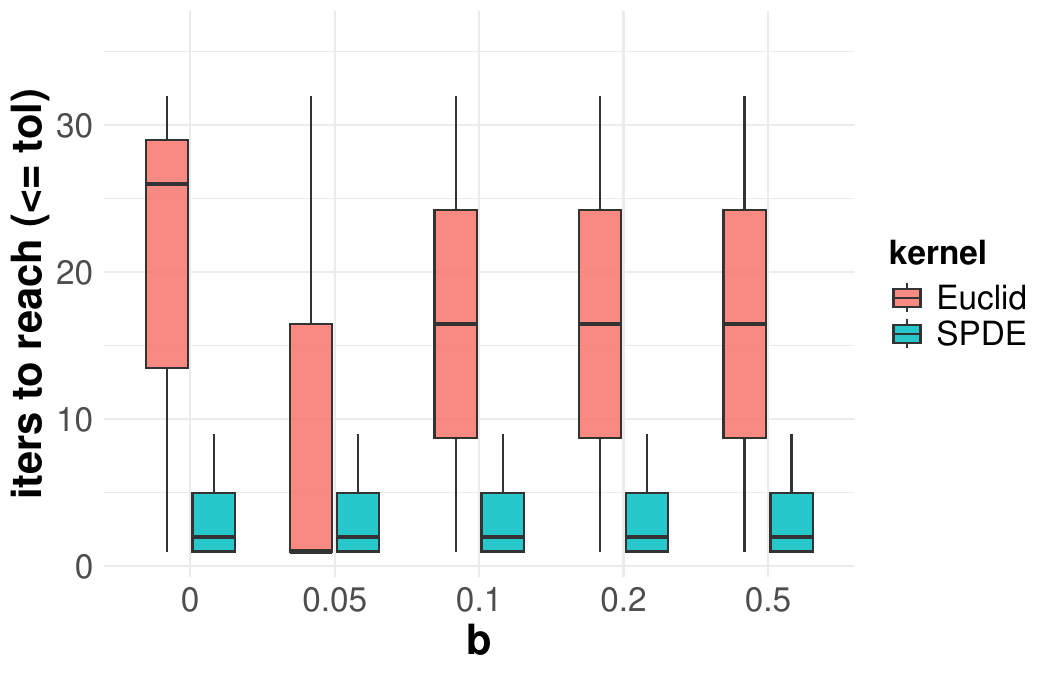}
  \caption{Iterations to $\mathsf{Tol}$}
\end{subfigure}\hfill

\caption{ Sensitivity to the misspecification-correction parameter $b$ for Algorithm~\ref{alg: benchmarks} (\textbf{IGP-UCB}).
Columns show: (a) simple regret across different initializations, with the median in a solid line and the shaded region representing the central $50\%$ band; (b) reach rate; and (c) iterations to reach $\mathsf{Tol}$.
Rows correspond to the Ackley, Rastrigin, and L\'evy benchmarks on the open-rectangle metric graph.}
\label{fig:benchmarks_IGP-UCB_b_sensitivity}
\end{figure}

\begin{figure}[H]
\centering

\begin{minipage}{\textwidth}\centering
  \subcaption*{Setting 1: Ackley}
\end{minipage}

\begin{subfigure}{0.32\textwidth}
  \centering
  \includegraphics[width=\linewidth]{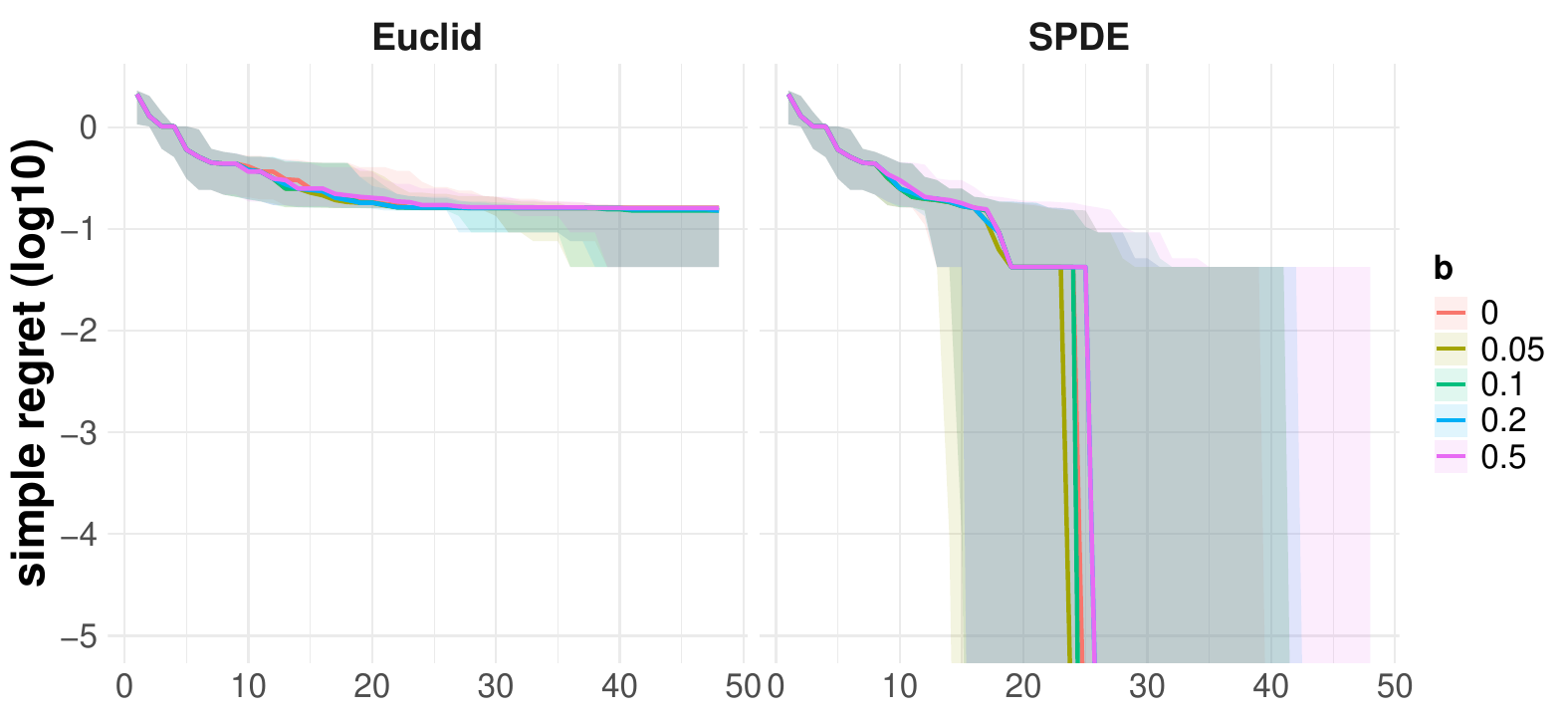}
\end{subfigure}\hfill
\begin{subfigure}{0.32\textwidth}
  \centering
  \includegraphics[width=\linewidth]{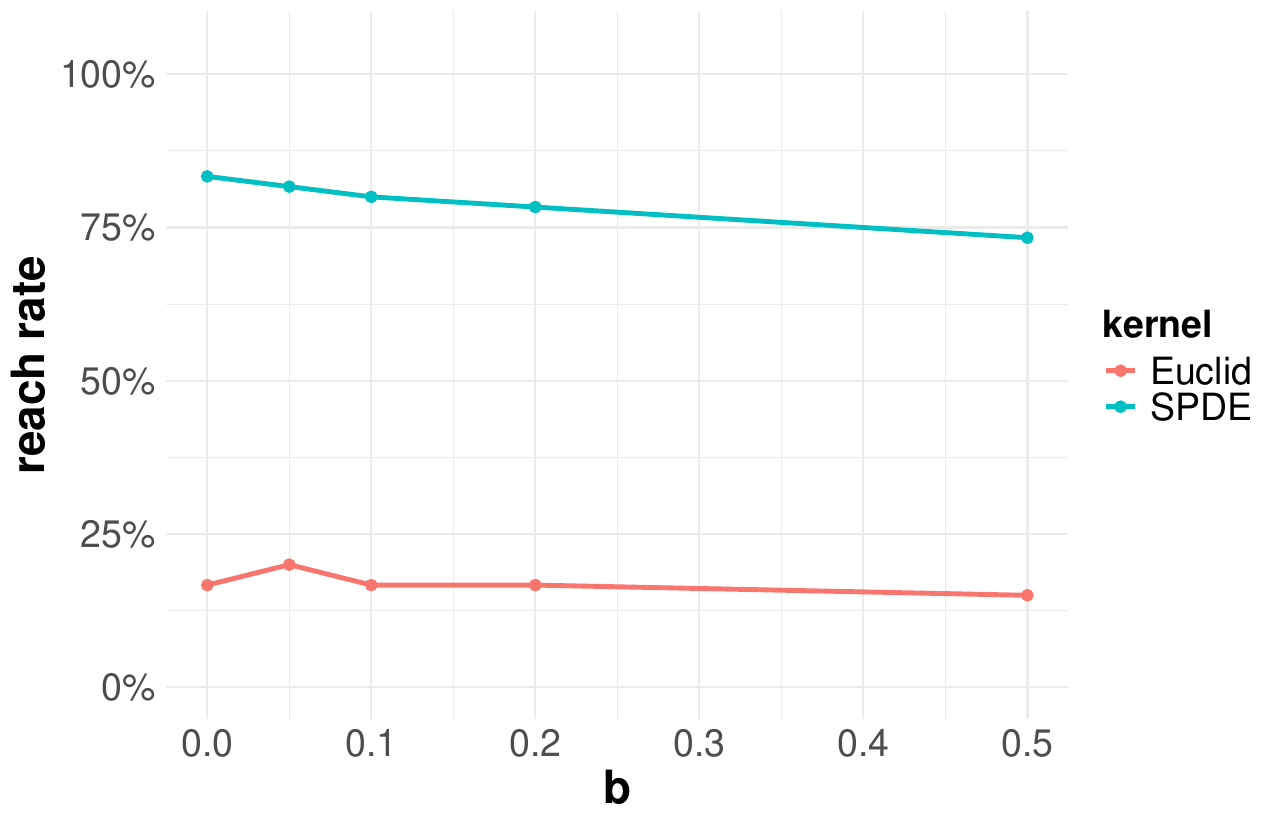}
\end{subfigure}\hfill
\begin{subfigure}{0.32\textwidth}
  \centering
  \includegraphics[width=\linewidth]{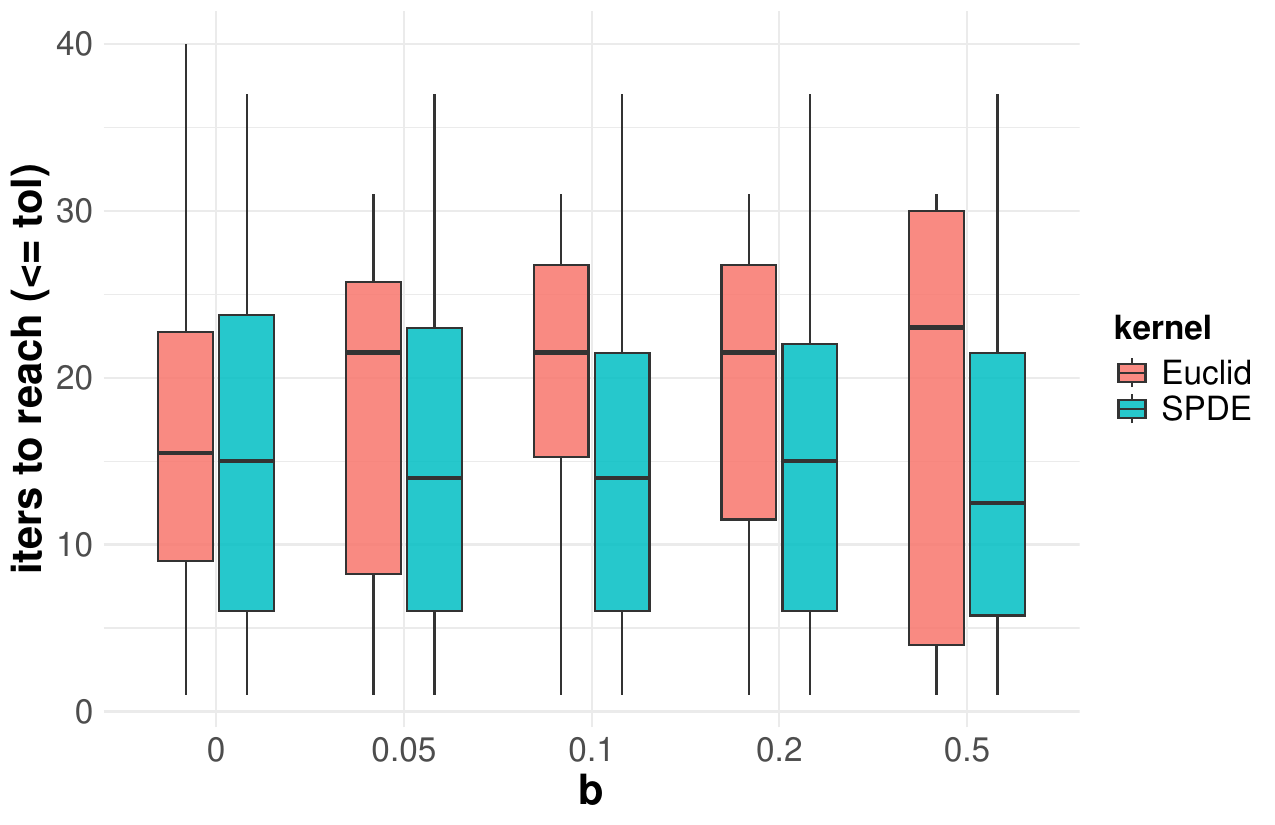}
\end{subfigure}

\begin{minipage}{\textwidth}\centering
  \subcaption*{Setting 2: Rastrigin}
\end{minipage}

\begin{subfigure}{0.32\textwidth}
  \centering
  \includegraphics[width=\linewidth]{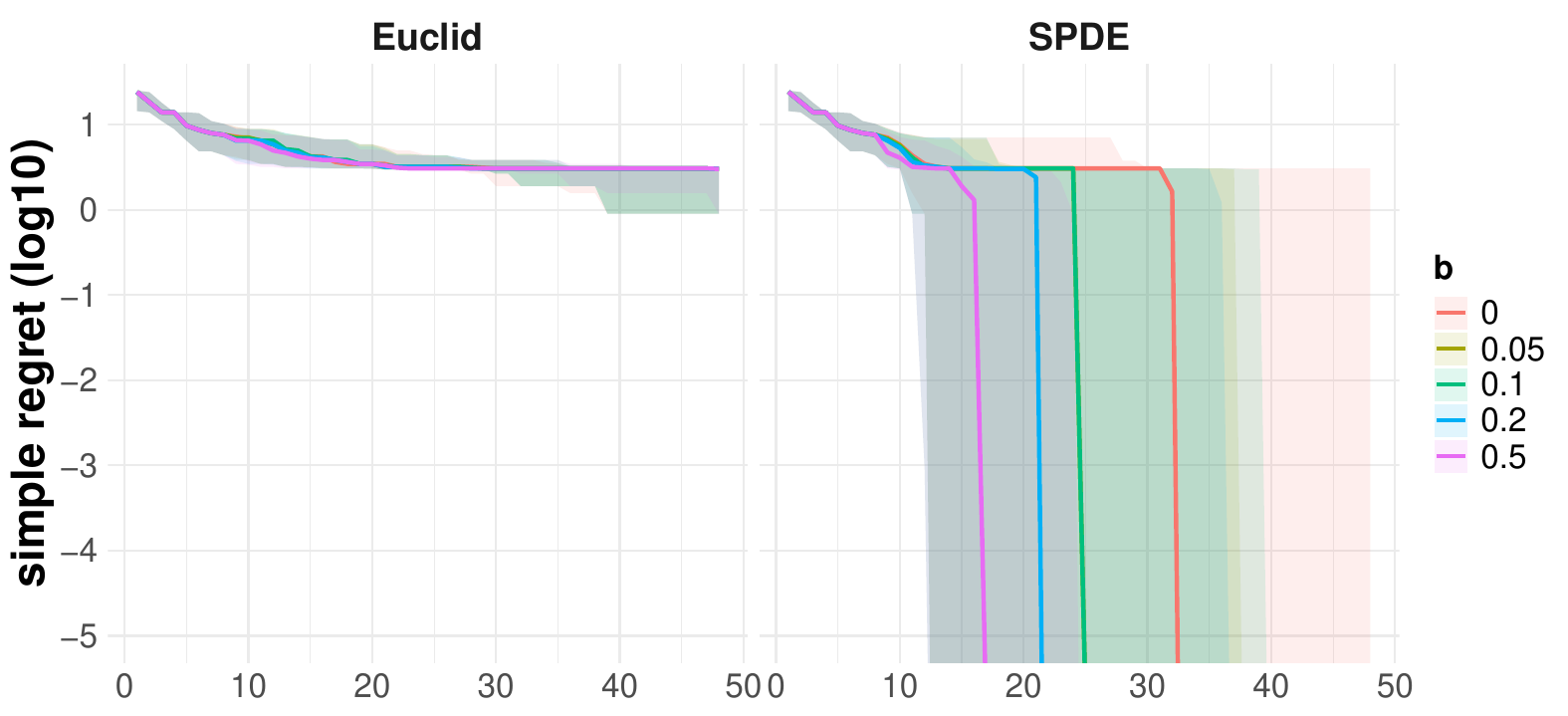}
\end{subfigure}\hfill
\begin{subfigure}{0.32\textwidth}
  \centering
  \includegraphics[width=\linewidth]{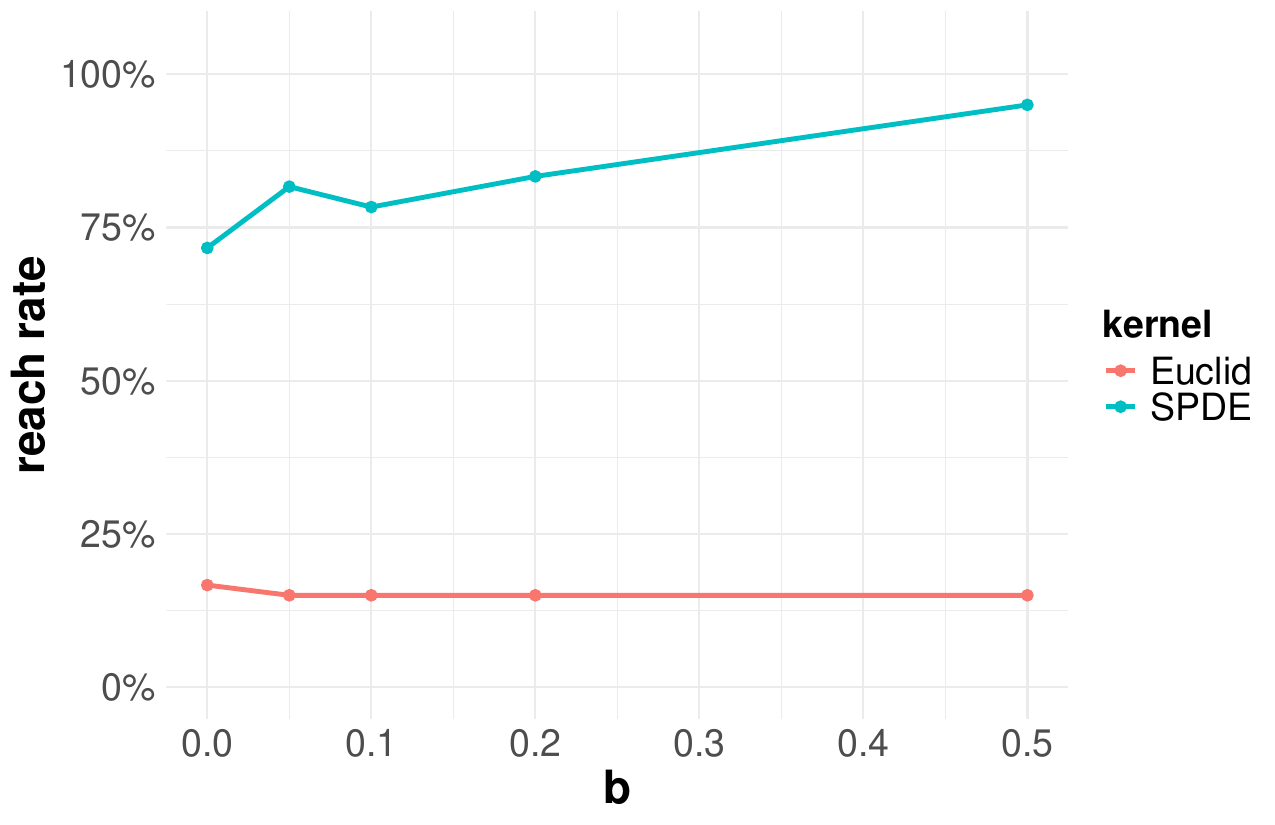}
\end{subfigure}\hfill
\begin{subfigure}{0.32\textwidth}
  \centering
  \includegraphics[width=\linewidth]{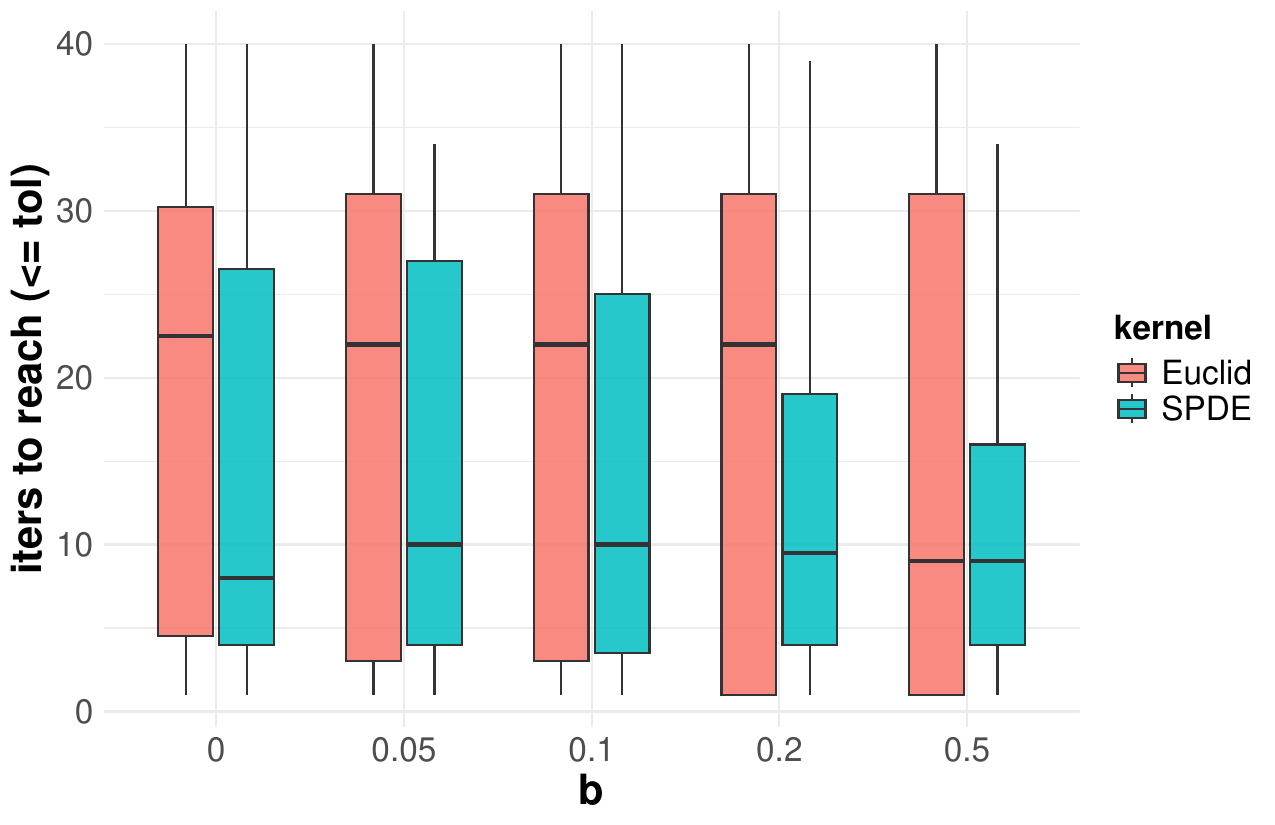}
\end{subfigure}

\vspace{0.6em}

\begin{minipage}{\textwidth}\centering
  \subcaption*{Setting 3: L\'evy}
\end{minipage}

\begin{subfigure}{0.32\textwidth}
  \centering
  \includegraphics[width=\linewidth]{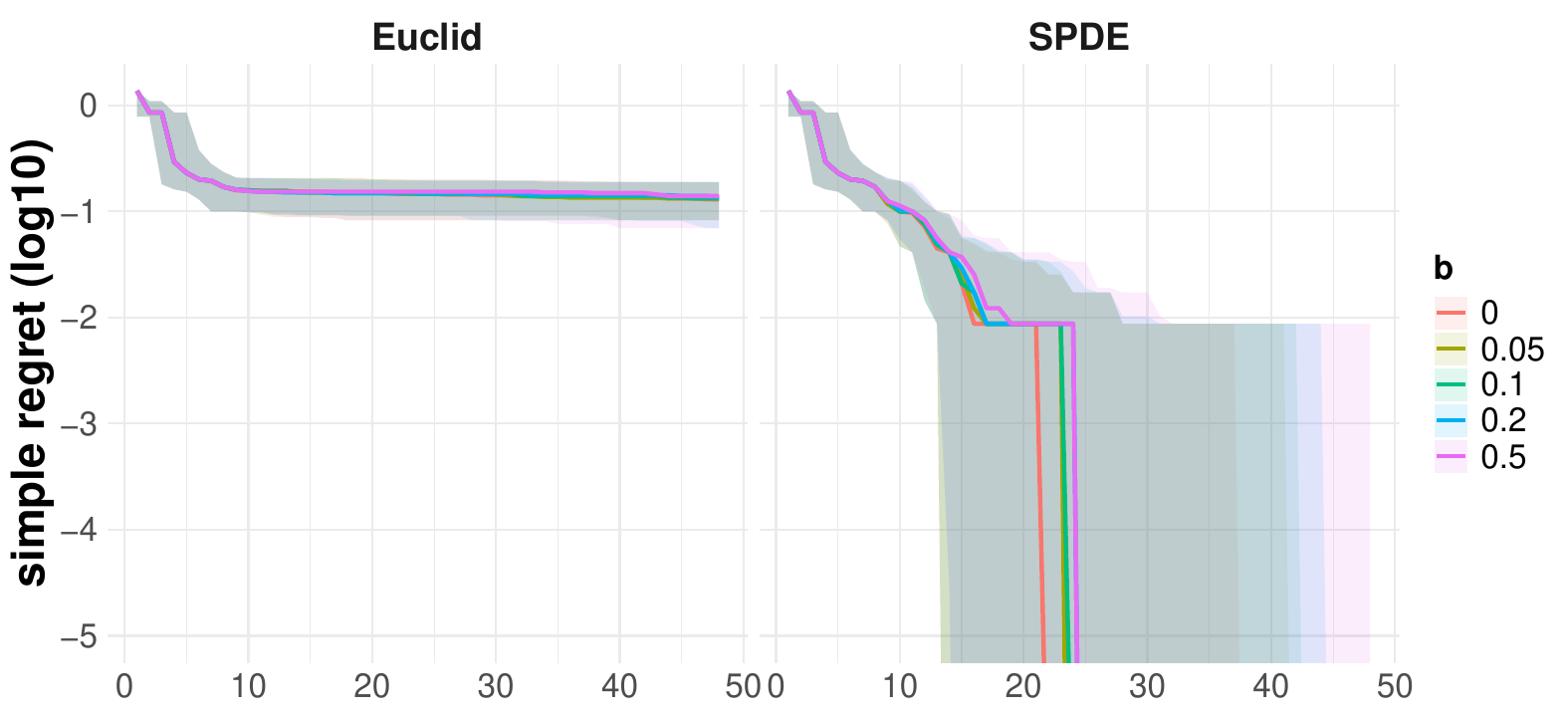}
  \caption{Simple regret}
\end{subfigure}\hfill
\begin{subfigure}{0.32\textwidth}
  \centering
  \includegraphics[width=\linewidth]{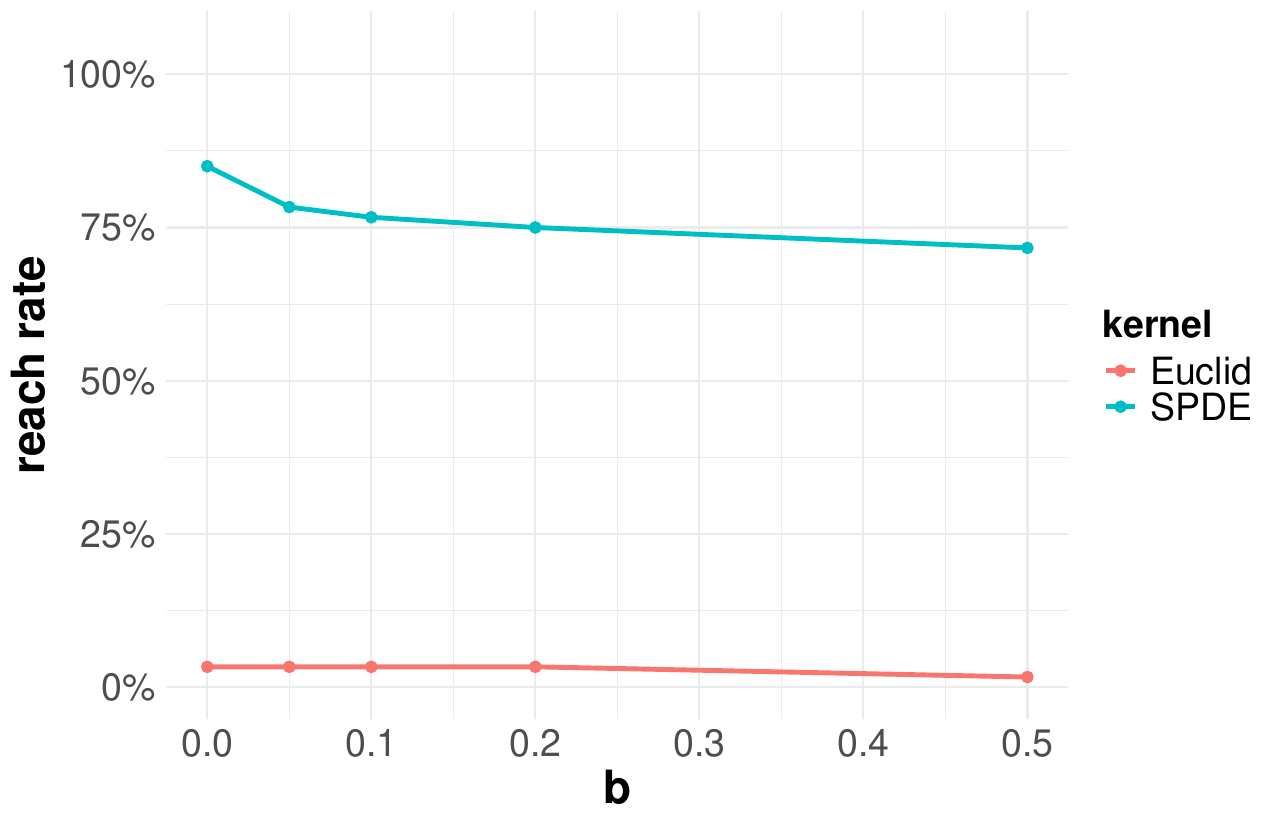}
  \caption{Reach rate}
\end{subfigure}\hfill
\begin{subfigure}{0.32\textwidth}
  \centering
  \includegraphics[width=\linewidth]{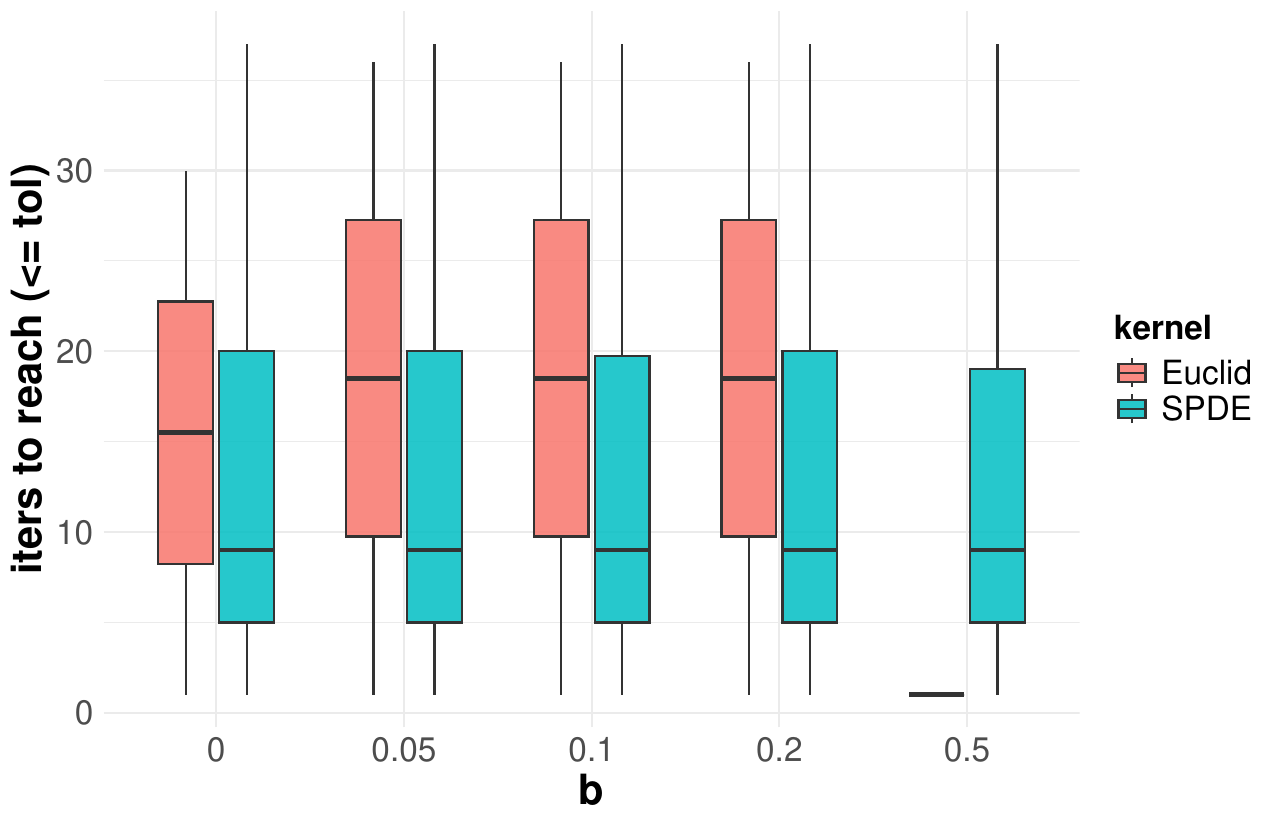}
  \caption{Iterations to $\mathsf{Tol}$}
\end{subfigure}

\caption{Sensitivity to the misspecification-correction parameter $b$ in Algorithm~\ref{alg: benchmarks} (\textbf{GP-TS}). The layout matches Figure~\ref{fig:benchmarks_IGP-UCB_b_sensitivity}.}
\label{fig:benchmarks_GP-TS_b_sensitivity}
\end{figure}

\paragraph{Empirical sensitivity.}
Figures~\ref{fig:benchmarks_IGP-UCB_b_sensitivity} and~\ref{fig:benchmarks_GP-TS_b_sensitivity} summarize the sensitivity of Algorithm~\ref{alg: benchmarks} to the  parameter $b$ over the grid
\(
b\in\{0,0.05,0.1,0.2,0.5\}.
\)
For IGP-UCB, the results are essentially insensitive to $b$ in this setting.
In particular, under the SPDE kernel, the reach rate is consistently $100\%$ across all three benchmarks, and the mean iterations-to-$\mathsf{Tol}$ vary only slightly as $b$ increases.
For the Euclidean kernel, the reach rates remain low and the iterations fluctuate mildly with no systematic dependence on $b$.

For GP-TS, the results are slightly more sensitive to $b$, with benchmark-dependent changes in reach rates and iterations for the SPDE kernel, while the Euclidean kernel remains broadly unchanged across the tested values.
A plausible explanation is that TS draws a single posterior sample path at each iteration, so its exploration is driven more directly by posterior uncertainty and can therefore respond more noticeably to the additional inflation induced by the $b$-term.
Overall, however, the qualitative conclusions are stable across the tested grid: the SPDE kernel consistently achieves substantially higher reach rates and typically fewer iterations than the baseline Euclidean kernel for both IGP-UCB and GP-TS, demonstrating the effectiveness of our approach throughout the considered range of $b$.

\subsection{Sensitivity to the Initialization Size across Discretizations}\label{app:init-sensitivity}

To provide practical guidance for choosing the initialization size $N_{\mathrm{init}}$, we perform a sensitivity study on the \emph{normalized L\'evy} benchmark defined in Subsection~\ref{ssec:benchmarks} on the open-rectangle graph (cf.\ Figure~\ref{fig:true-benchmarks-openrectangle}).
We consider three discretization levels of the same graph: a \emph{coarse} mesh with $h=0.5$ ($N_h\approx 150$), the \emph{baseline} mesh with $h=0.25$ ($N_h\approx 300$), and a \emph{fine} mesh with $h=0.15$ ($N_h\approx 500$).
For each discretization, we vary the maximin initialization size over
$N_{\mathrm{init}}\in\{1,2,4,8,10,12,16\}$ and run $60$ independent Monte Carlo replicates (different random seeds inducing different maximin initialization designs) for both GP-TS and IGP-UCB, using the same post-initialization horizon as in the main experiments (cf. Subsection~\ref{ssec:numerical results}).

We report (i) the \emph{reach rate} and (ii) the \emph{iterations-to-$\mathsf{Tol}$}, as defined in Section~\ref{sec:Numerics}.
Recall that \emph{iterations-to-$\mathsf{Tol}$} counts only acquisition steps \emph{after} initialization and is summarized over successful runs.
Together, these metrics isolate the impact of $N_{\mathrm{init}}$ and illustrate how the amount of initialization should be adjusted as the discretization is refined.

\begin{figure}[t]
\centering

\begin{minipage}{\textwidth}\centering
  \subcaption*{Reach rate}
\end{minipage}

\begin{subfigure}{0.32\textwidth}
  \centering
  \includegraphics[width=\linewidth]{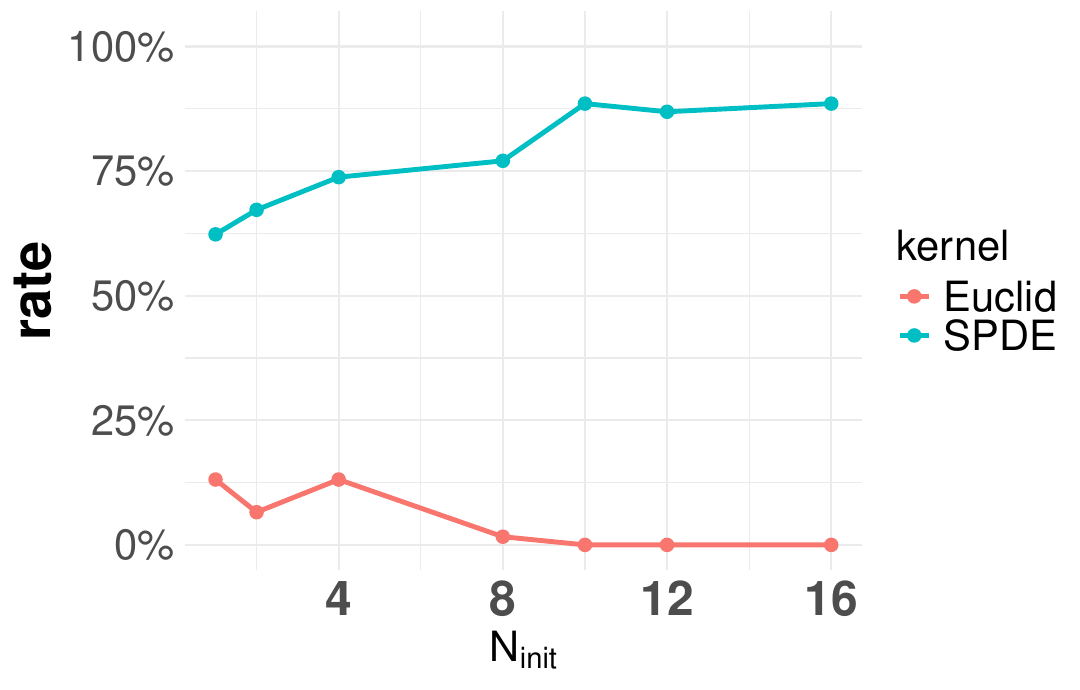}
\end{subfigure}\hfill
\begin{subfigure}{0.32\textwidth}
  \centering
  \includegraphics[width=\linewidth]{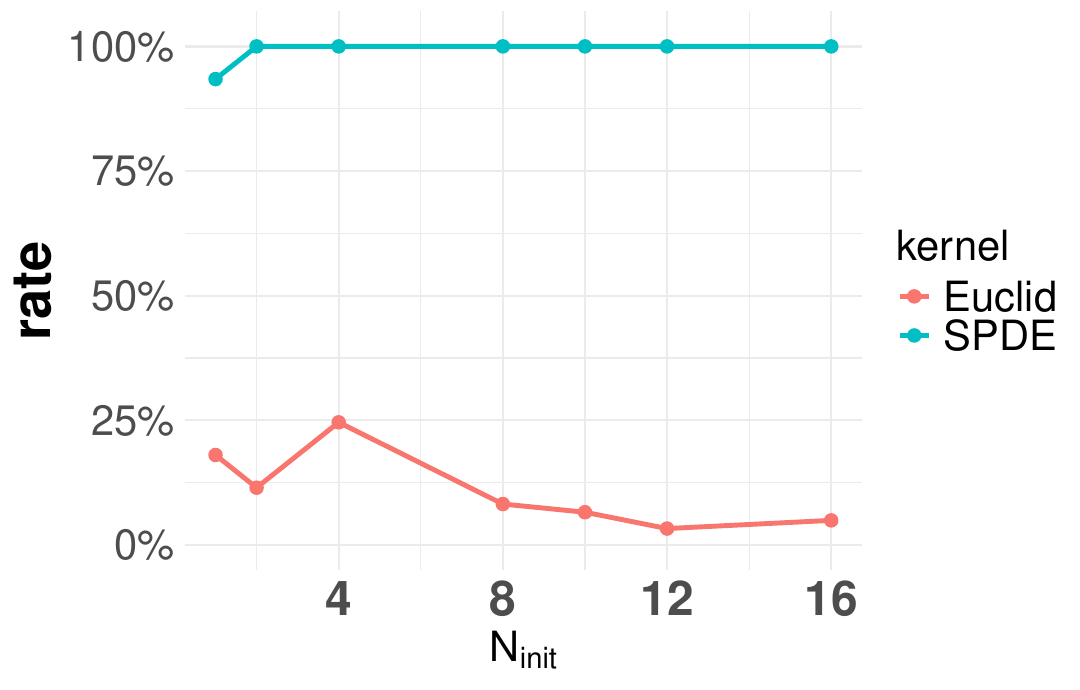}
\end{subfigure}\hfill
\begin{subfigure}{0.32\textwidth}
  \centering
  \includegraphics[width=\linewidth]{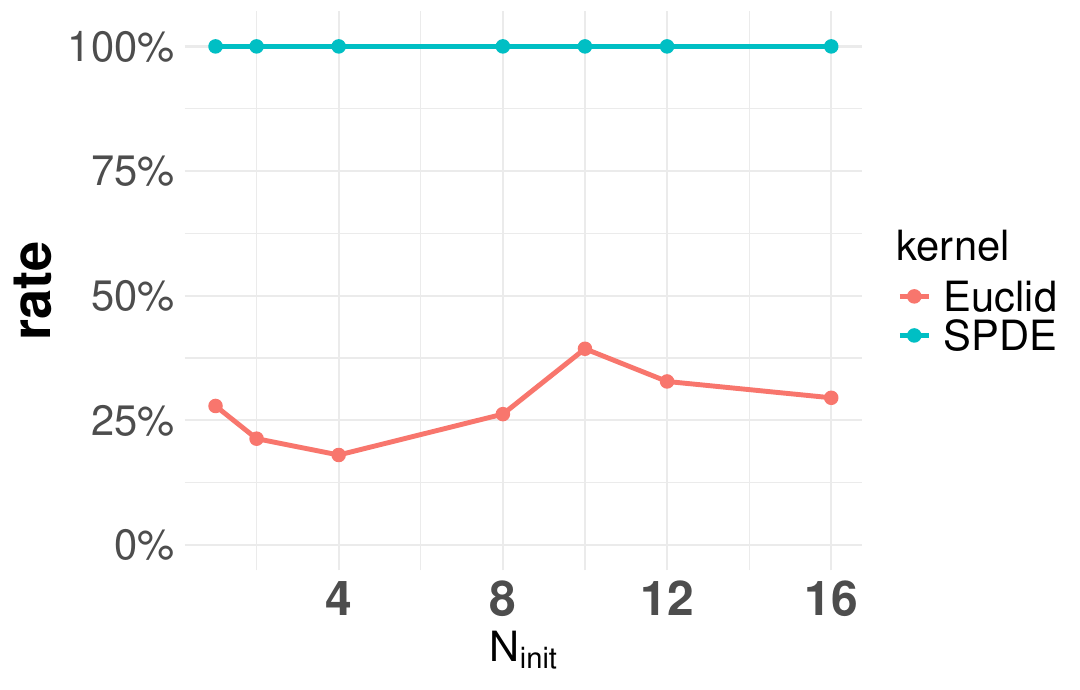}
\end{subfigure}

\begin{minipage}{\textwidth}\centering
  \subcaption*{Iterations to $\mathsf{Tol}$}
\end{minipage}

\begin{subfigure}{0.32\textwidth}
  \centering
  \includegraphics[width=\linewidth]{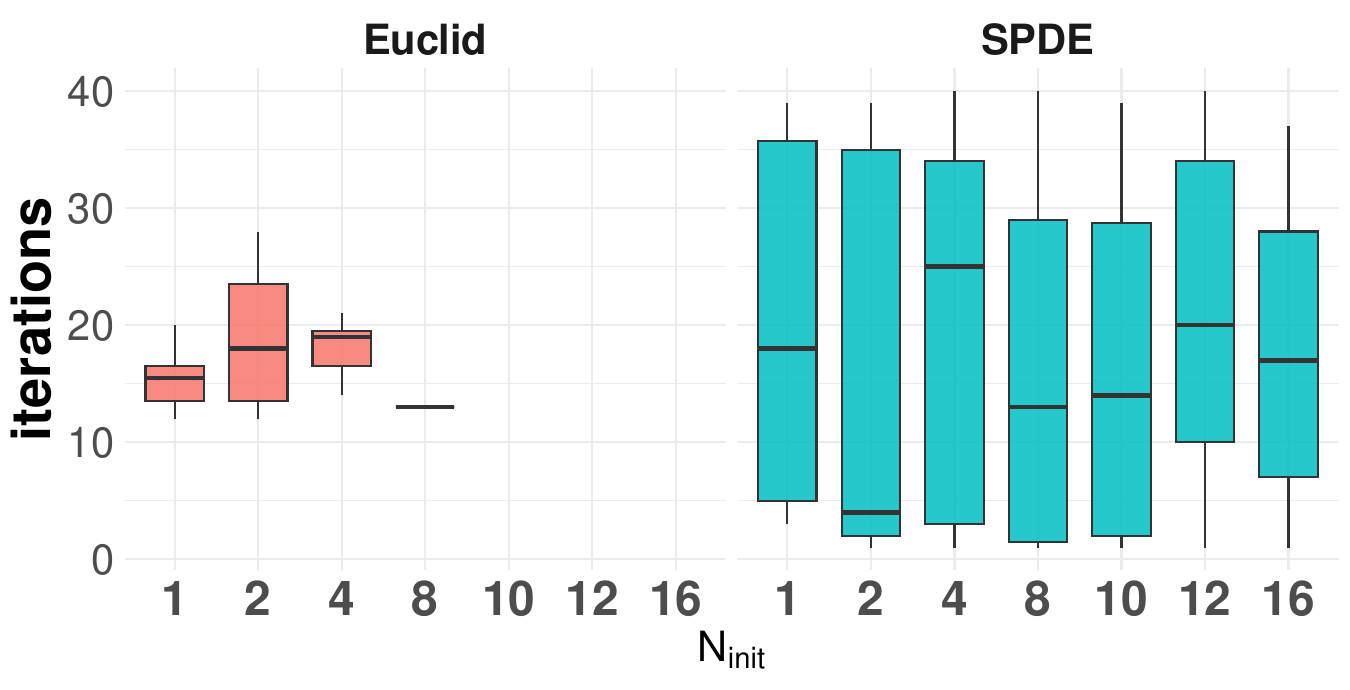}
\end{subfigure}\hfill
\begin{subfigure}{0.32\textwidth}
  \centering
  \includegraphics[width=\linewidth]{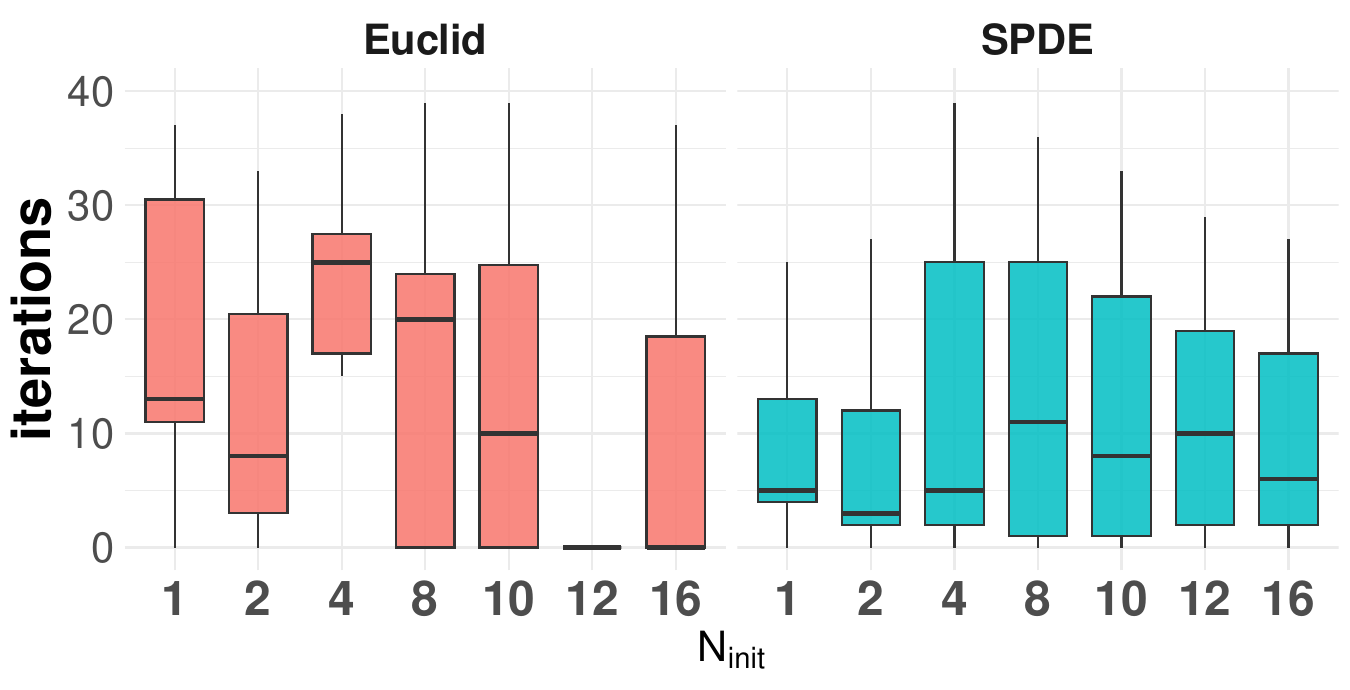}
\end{subfigure}\hfill
\begin{subfigure}{0.32\textwidth}
  \centering
  \includegraphics[width=\linewidth]{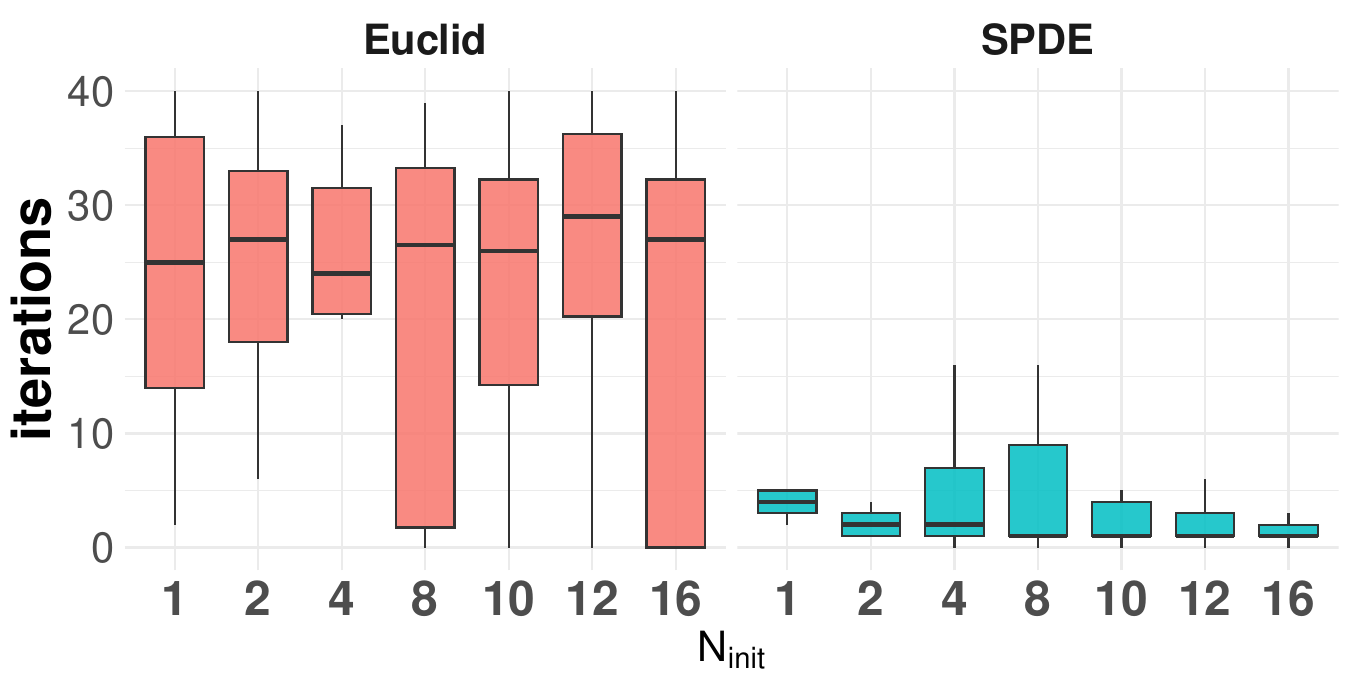}
\end{subfigure}

\vspace{0.6em}
\caption{Sensitivity to the initialization size $N_{\mathrm{init}}$ for Algorithm~\ref{alg: benchmarks} (\textbf{IGP-UCB}) on the normalized L\'evy benchmark.
Columns correspond to: (a) fine discretization ($h=0.15$, $N_h\approx 500$), (b) baseline discretization ($h=0.25$, $N_h\approx 300$), and (c) coarse discretization ($h=0.5$, $N_h\approx 150$).}
\label{fig:benchmarks_IGP-UCB_init_sensitivity}
\end{figure}

\begin{figure}[t]
\centering

\begin{minipage}{\textwidth}\centering
  \subcaption*{Reach rate}
\end{minipage}

\begin{subfigure}{0.32\textwidth}
  \centering
  \includegraphics[width=\linewidth]{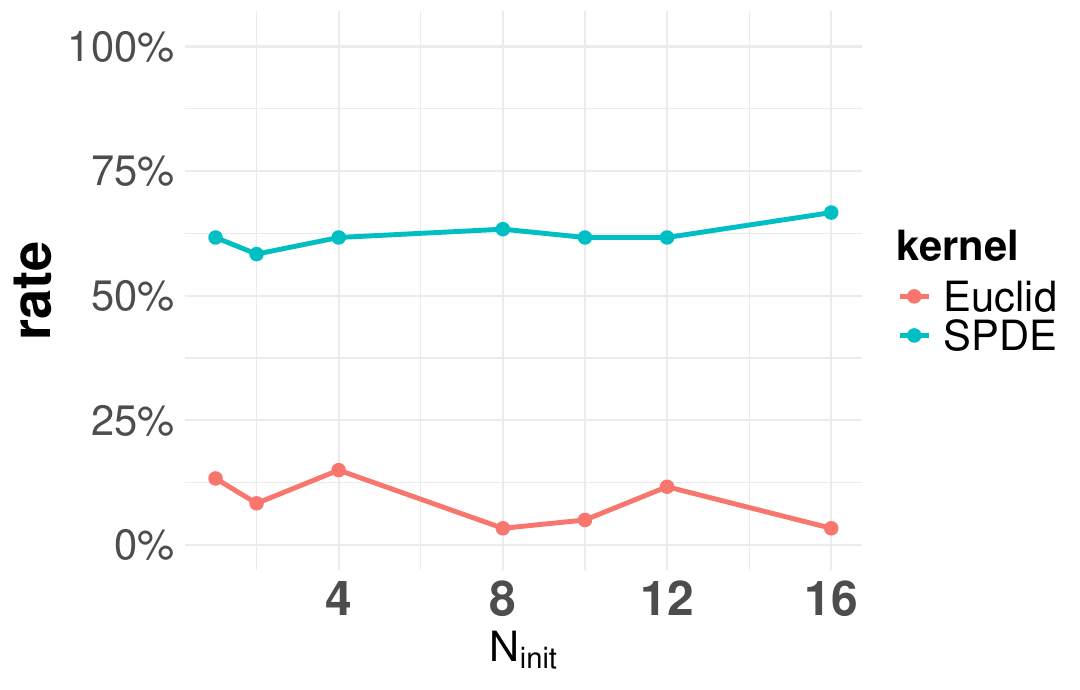}
\end{subfigure}\hfill
\begin{subfigure}{0.32\textwidth}
  \centering
  \includegraphics[width=\linewidth]{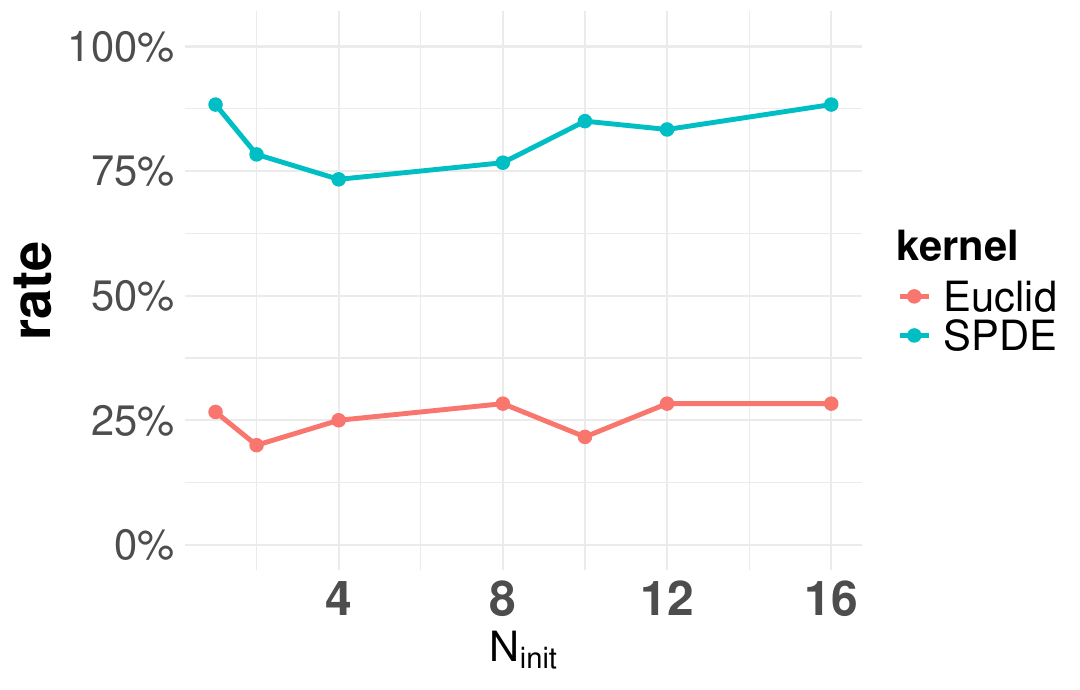}
\end{subfigure}\hfill
\begin{subfigure}{0.32\textwidth}
  \centering
  \includegraphics[width=\linewidth]{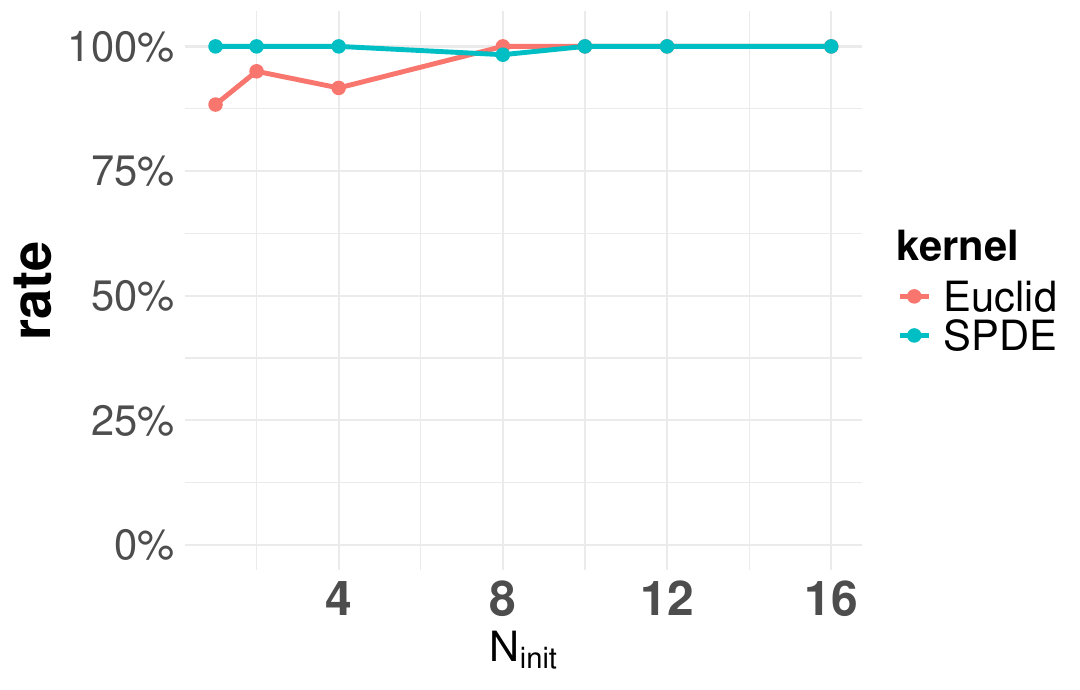}
\end{subfigure}

\begin{minipage}{\textwidth}\centering
  \subcaption*{Iterations to $\mathsf{Tol}$}
\end{minipage}

\begin{subfigure}{0.32\textwidth}
  \centering
  \includegraphics[width=\linewidth]{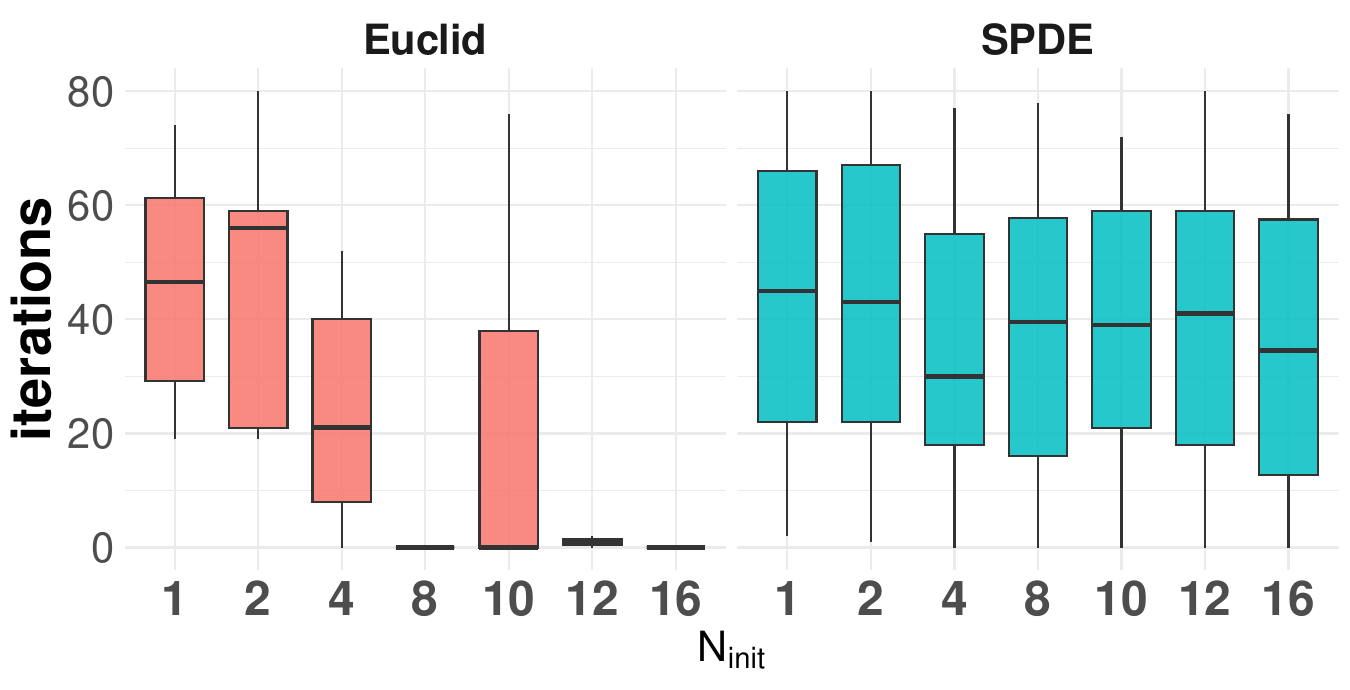}
\end{subfigure}\hfill
\begin{subfigure}{0.32\textwidth}
  \centering
  \includegraphics[width=\linewidth]{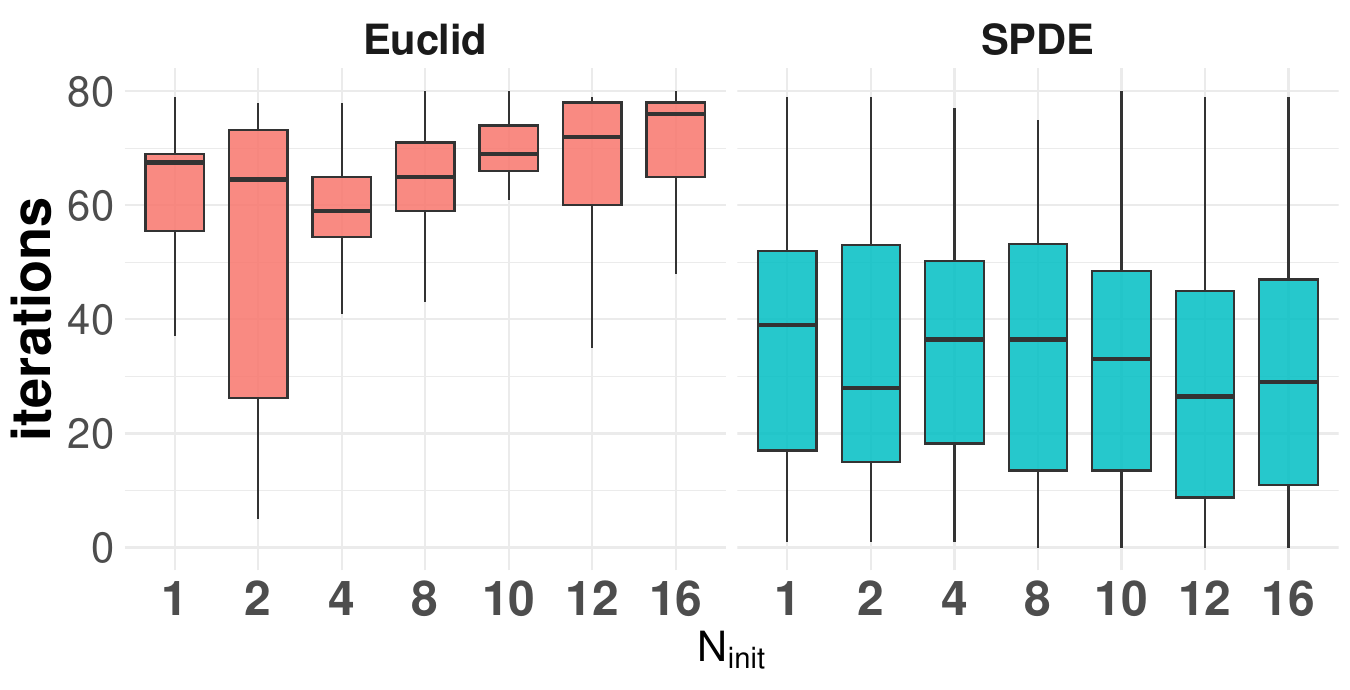}
\end{subfigure}\hfill
\begin{subfigure}{0.32\textwidth}
  \centering
  \includegraphics[width=\linewidth]{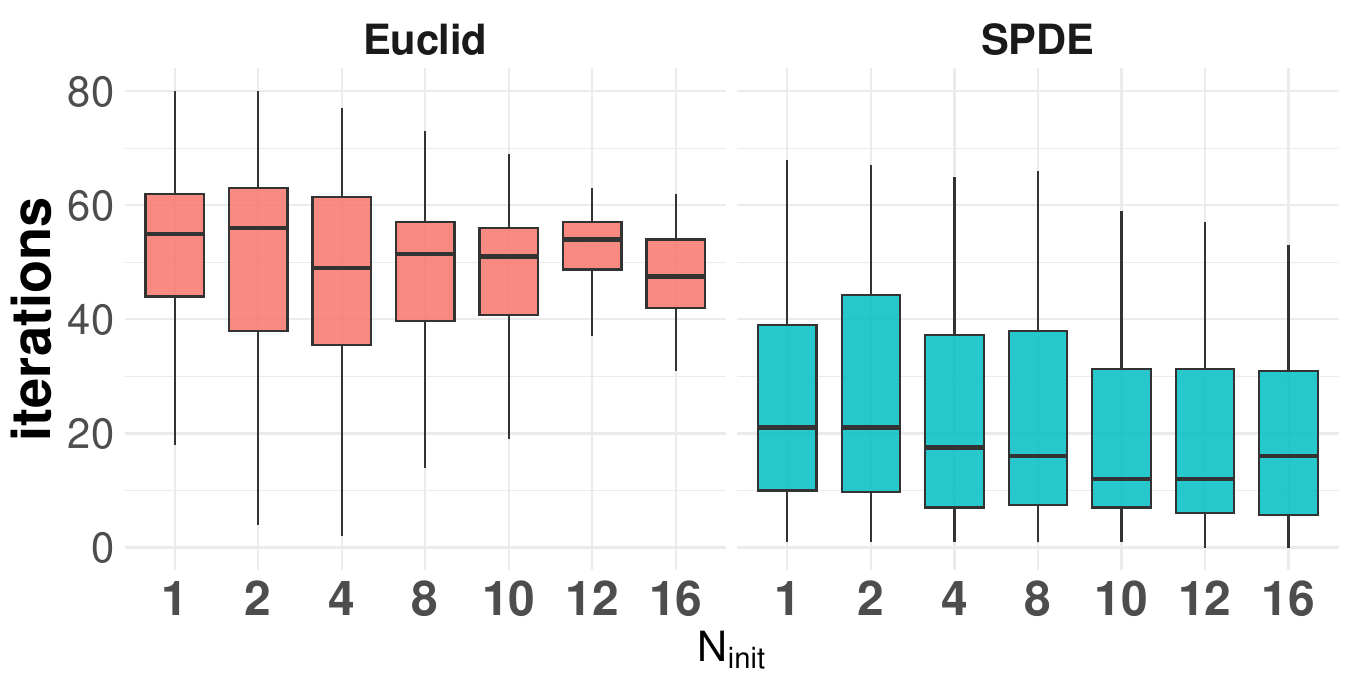}
\end{subfigure}

\vspace{0.6em}
\caption{Sensitivity to the initialization size $N_{\mathrm{init}}$ for Algorithm~\ref{alg: benchmarks} (\textbf{GP-TS}) on the normalized L\'evy benchmark.
Same display as Figure~\ref{fig:benchmarks_IGP-UCB_init_sensitivity}.}
\label{fig:benchmarks_GP-TS_init_sensitivity}
\end{figure}

Figures~\ref{fig:benchmarks_IGP-UCB_init_sensitivity}--\ref{fig:benchmarks_GP-TS_init_sensitivity} exhibit consistent qualitative behavior for both IGP-UCB and GP-TS across all three discretizations.
As $N_{\mathrm{init}}$ increases from very small values, the reach rate improves for both methods, indicating that the maximin initialization is providing broader spatial coverage and avoiding noticeably worse early-stage behavior.
Once $N_{\mathrm{init}}$ enters a moderately space-filling regime, the reach-rate curves flatten and the iterations-to-$\mathsf{Tol}$ distributions become smaller and more concentrated.

The dependence on discretization follows the same pattern, and the plots also suggest a natural ``sweet spot'' for $N_{\mathrm{init}}$ in each case.
For the coarse mesh ($h=0.5$, $N_h\approx 150$), the reach-rate curves, for both IGP-UCB and GP-TS, are already close to saturation once $N_{\mathrm{init}}$ is in the moderate range, and the post-initialization iterations-to-$\mathsf{Tol}$ are correspondingly small and stable.
In this regime, $N_{\mathrm{init}}\approx 4$--$8$ is typically sufficient: it provides basic space-filling coverage and stable early hyperparameter updates, while larger initializations offer little additional gain beyond marginal reductions in variability.

For the baseline mesh ($h=0.25$, $N_h\approx 300$), both algorithms exhibit a clear transition from the small-$N_{\mathrm{init}}$ regime---where reach rates are noticeably lower and the post-initialization iterations-to-$\mathsf{Tol}$ are more dispersed---to a stable regime in which reach rates are high and the iterations-to-$\mathsf{Tol}$ distributions are more concentrated.
The transition occurs around the moderate initializations, and the plots indicate that $N_{\mathrm{init}}=8$ lies within the stable regime for both methods, achieving performance comparable to larger choices while avoiding the poor conditioning and unstable early MLE behavior observed at extremely small initializations.
As a result, $N_{\mathrm{init}}=8$ is a robust default at the baseline resolution.

For the fine mesh ($h=0.15$, $N_h\approx 500$), increasing $N_{\mathrm{init}}$ continues to improve reliability slightly longer than in the coarser settings, reflecting the increased difficulty of the refined discretization.
Nevertheless, the reach-rate curves still level off once $N_{\mathrm{init}}$ reaches a moderate range, and further increases yield diminishing returns.
In this regime, the plots suggest that $N_{\mathrm{init}}\approx 8$--$10$ strikes a good balance: it improves coverage and stabilizes early hyperparameter updates relative to smaller initializations, while pushing beyond this range produces only marginal gains and can be counterproductive under online MLE when noise perturbs early length-scale estimates (an effect that is especially visible for the Euclidean kernel).

The observed plateau also explains why continuing to increase $N_{\mathrm{init}}$ does not markedly improve performance.
Under online hyperparameter learning, the early MLE updates incorporate noisy observations. Once the initialization is already sufficiently space-filling to stabilize the length-scale estimates, additional initial points contribute diminishing geometric information while injecting more noise into the likelihood, which can transiently perturb hyperparameter estimates.
Consequently, improvements in reach rate become marginal, and the distributions of post-initialization iterations-to-$\mathsf{Tol}$ do not continue to shrink in a systematic way.

This effect is most visible for the Euclidean kernel, where performance is substantially more variable and can even be non-monotone in $N_{\mathrm{init}}$, especially at finer discretizations.
A natural explanation is model misspecification: the Euclidean kernel does not fully reflect the intrinsic graph geometry, so the online MLE can be more sensitive to noise and may push the fitted length scale toward extreme values.
In this regime, increasing $N_{\mathrm{init}}$ can lead to overconfident posteriors early on (overly small posterior variances or overly aggressive length-scale fits), which in turn distorts acquisition decisions and produces the observed instability in reach rate and iterations.
By contrast, the SPDE kernel is better aligned with the geometry, yielding more stable hyperparameter updates and more consistent plateau behavior.

These results support a simple, actionable guideline for selecting $N_{\mathrm{init}}$ as a function of the discretization size $N_h$, with Figures~\ref{fig:benchmarks_IGP-UCB_init_sensitivity}--\ref{fig:benchmarks_GP-TS_init_sensitivity} serving as empirical support.
For coarse meshes with $N_h\approx 150$, values around $N_{\mathrm{init}}\approx 4$ are typically sufficient.
For baseline resolutions with $N_h\approx 200$--$400$, $N_{\mathrm{init}}=8$ is a robust default that balances coverage and hyperparameter-update stability while performing comparably to larger choices.
For finer discretizations with $N_h\gtrsim 500$, a modest increase beyond the baseline (e.g., $N_{\mathrm{init}}\approx 8$--$10$) can improve reliability, after which further increases yield only marginal gains.
\endgroup

\end{document}